\documentclass[a4paper,reqno]{amsart}
\usepackage{cmap}
\usepackage[british]{babel}
\usepackage[utf8]{inputenc}
\usepackage[T1]{fontenc}
\usepackage{mathptmx}
\usepackage{amsthm}
\usepackage{mathtools}
\usepackage{amssymb}
\usepackage{comment}
\usepackage{microtype}
\usepackage{hyperref,url}
\usepackage{caption}
  \hypersetup{  % bookmarks
    bookmarksnumbered
  }
  \hypersetup{  % links
    breaklinks=false,
    pdfborderstyle={/S/U/W 0},  % zet op .5 ofzo voor onderlijnen
    citebordercolor=.235 .702 .443,
    urlbordercolor=.255 .412 .882,
    linkbordercolor=.804 .149 .149,
  }
  \hypersetup{ % metadata
    pdfauthor={Jasper De Bock \& Gert de Cooman},
    pdftitle={An efficient algorithm for estimating state
      sequences in imprecise hidden Markov models}
  }
\usepackage{nicefrac}
\usepackage{enumitem}

\usepackage{tikz}
\usetikzlibrary{calc}
\usetikzlibrary{decorations.pathmorphing}
\usetikzlibrary{decorations.markings}
\usetikzlibrary{fit}
\usetikzlibrary{backgrounds}
\usetikzlibrary{trees}
\usetikzlibrary{matrix}
\usetikzlibrary{decorations.shapes}
\usetikzlibrary{shadows}
\usetikzlibrary{shapes,arrows}
  \tikzstyle{commentlink}=[draw opacity=1,draw=gray,thin] 
  \tikzstyle{comment}=[text opacity=1,fill
  opacity=.1,rectangle,rounded corners,fill=gray,text=gray] 

\xdefinecolor{rood}{RGB}{241, 90, 34}
\xdefinecolor{geel}{RGB}{255, 197, 11}
\xdefinecolor{roze}{RGB}{236, 0, 140}
\xdefinecolor{groen}{RGB}{179, 211, 53}
\xdefinecolor{grijs}{RGB}{72, 119, 116}

\colorlet{bijschriftkleur}{rood!80!black}

\newcommand{\afstandknopen}{1.9 cm}
\newcommand{\tssknoop}{3 cm}

\newcommand{\afstandvollelijn}{1 cm}

\tikzstyle{pijl1}=[draw=geel!0!black,line width=.6pt,
postaction={decorate}, decoration={markings,mark=at position 1.0 with
  {\arrow[draw=geel!0!black,line width=1.5pt]{stealth}}}] 
\tikzstyle{pijl2}=[draw=groen!50!black,line width=.6pt, postaction={decorate}, decoration={markings,mark=at position 1.0 with {\arrow[draw=groen!50!black,line width=1pt]{>}}}]
\tikzstyle{pijl3}=[draw=geel!0!black,dashed,line width=.6pt]
\tikzstyle{pijl4}=[draw=geel!0!black,line width=.6pt]
\tikzstyle{knoopobservatie}=[circle,draw=black!80,fill=geel!0,minimum size=.65cm,line width=.6pt]
\tikzstyle{knoopstate}=[circle,draw=black!80,fill=grijs!20,minimum size=.65cm,line width=.6pt]
\tikzstyle{knoop1}=[circle,draw=black!80,fill=grijs!20,minimum size=.65cm,line width=.6pt]
\tikzstyle{knoop2}=[circle,draw=black!80,fill=geel!0,minimum size=.65cm,line width=.6pt]
\tikzstyle{omcirkel}=[black!20,rounded corners=.8cm, densely dotted,
line width=.6pt] 

\tikzstyle{niets}=[circle,draw=black,line width=1pt,fill=white,minimum size=6mm,inner sep=2pt]
\tikzstyle{ja}=[niets,fill=groen!60]
\tikzstyle{nee}=[niets,fill=red!30]
\tikzstyle{boomlijn}=[draw=black,line width=1pt]
\tikzstyle{boomstippellijn}=[boomlijn,dashed]
\tikzstyle{boomlijngroen}=[boomlijn]

\tikzstyle{kplus}=[fill opacity=.2,fill=grijs!20,text
opacity=.9,text=black,rounded corners=6mm, draw=black!20, line
width=1pt]

\tikzstyle{roodkader}=[fill opacity=.2,fill=red!40,text
opacity=.9,text=red,rounded corners=4mm, draw=red, line
width=1pt]

\usepackage{wrapfig}
\newcommand{\reals}{\mathbb{R}}

\newcommand{\signedprod}[2]{\overline{\underline{#1}}\odot#2}

\newcommand{\asa}{\Leftrightarrow}
\newcommand{\dan}{\Rightarrow}

\newcommand{\from}[1]{#1:n}
\newcommand{\fromto}[2]{#1:#2}
\newcommand{\sing}[1]{\{#1\}}
\newcommand{\ind}[1]{\mathbb{I}_{#1}}
\newcommand{\indsing}[1]{\ind{\sing{#1}}}
\newcommand{\set}[2]{\left\{#1\colon#2\right\}}
\newcommand{\biggset}[2]{\bigg\{#1\colon#2\bigg\}}

\newcommand{\statevar}[1]{X_{#1}}
\newcommand{\fromstatevar}[1]{X_{#1:n}}
\newcommand{\fromtostatevar}[2]{X_{#1:#2}}
\newcommand{\outvar}[1]{O_{#1}}
\newcommand{\fromoutvar}[1]{O_{#1:n}}
\newcommand{\fromtooutvar}[2]{O_{#1:#2}}
\newcommand{\states}[1]{\mathcal{X}_{#1}}
\newcommand{\outs}[1]{\mathcal{O}_{#1}}
\newcommand{\othervar}{Y}
\newcommand{\others}{\mathcal{Y}}

\newcommand{\fromstates}[1]{\states{\from{#1}}}
\newcommand{\fromouts}[1]{\outs{\from{#1}}}

\newcommand{\fromtostates}[2]{\states{\fromto{#1}{#2}}}
\newcommand{\fromtoouts}[2]{\outs{\fromto{#1}{#2}}}

\newcommand{\statelpr}[1]{\underline{Q}_{#1}}
\newcommand{\statepr}[1]{Q_{#1}}
\newcommand{\statelprone}[1][\cdot]{\statelpr{1}\left(#1\right)}
\newcommand{\stateprone}[1][\cdot]{\statepr{1}\left(#1\right)}
\newcommand{\stateclpr}[3][\cdot]{\underline{Q}_{#2}(#1\vert\statevar{#3})}
\newcommand{\zinstateclpr}[3][\cdot]{\underline{Q}_{#2}(#1\vert\zstate{#3})}
\newcommand{\biggzinstateclpr}[3][\cdot]{\underline{Q}_{#2}\bigg(#1\Big\vert\zstate{#3}\bigg)}
\newcommand{\xinstateclpr}[3][\cdot]{\underline{Q}_{#2}(#1\vert\xstate{#3})}

\newcommand{\stateupr}[1]{\overline{Q}_{#1}}
\newcommand{\stateuprone}[1][\cdot]{\stateupr{1}\left(#1\right)}

\newcommand{\zinstatecupr}[3][\cdot]{\overline{Q}_{#2}(#1\vert\zstate{#3})}
\newcommand{\xinstatecupr}[3][\cdot]{\overline{Q}_{#2}(#1\vert\xstate{#3})}

\newcommand{\zinstateclupr}[3][\cdot]{\overline{\underline{Q}}_{#2}(#1\vert\zstate{#3})}
\newcommand{\outclpr}[2][\cdot]{\underline{S}_{#2}(#1\vert\statevar{#2})}
\newcommand{\zinoutclpr}[2][\cdot]{\underline{S}_{#2}(#1\vert\zstate{#2})}
\newcommand{\xinoutclpr}[2][\cdot]{\underline{S}_{#2}(#1\vert\xstate{#2})}
\newcommand{\xhatinoutclpr}[2][\cdot]{\underline{S}_{#2}(#1\vert\xhatstate{#2})}

\newcommand{\zinoutcupr}[2][\cdot]{\overline{S}_{#2}(#1\vert\zstate{#2})}
\newcommand{\xinoutcupr}[2][\cdot]{\overline{S}_{#2}(#1\vert\xstate{#2})}
\newcommand{\xhatinoutcupr}[2][\cdot]{\overline{S}_{#2}(#1\vert\xhatstate{#2})}

\newcommand{\zinoutclupr}[2][\cdot]{\overline{\underline{S}}_{#2}(#1\vert\zstate{#2})}
\newcommand{\xinoutclupr}[2][\cdot]{\overline{\underline{S}}_{#2}(#1\vert\xstate{#2})}
\newcommand{\xhatinoutclupr}[2][\cdot]{\overline{\underline{S}}_{#2}(#1\vert\xhatstate{#2})}
\newcommand{\jointlpr}[1]{\underline{P}_{#1}}
\newcommand{\jointupr}[1]{\overline{P}_{#1}}
\newcommand{\jointlupr}[1]{\overline{\underline{P}}_{#1}}
\newcommand{\jointclpr}[3][\cdot]{\underline{P}_{#2}(#1\vert\statevar{#3})}
\newcommand{\zinjointclpr}[3][\cdot]{\underline{P}_{#2}(#1\vert\zstate{#3})}
\newcommand{\xinjointclpr}[3][\cdot]{\underline{P}_{#2}(#1\vert\xstate{#3})}
\newcommand{\xhatinjointclpr}[3][\cdot]{\underline{P}_{#2}(#1\vert\xhatstate{#3})}
\newcommand{\zinjointcupr}[3][\cdot]{\overline{P}_{#2}(#1\vert\zstate{#3})}
\newcommand{\xinjointcupr}[3][\cdot]{\overline{P}_{#2}(#1\vert\xstate{#3})}
\newcommand{\biggxinjointcupr}[3][\cdot]{\overline{P}_{#2}\bigg(#1\Big\vert\xstate{#3}\bigg)}
\newcommand{\xhatinjointcupr}[3][\cdot]{\overline{P}_{#2}(#1\vert\xhatstate{#3})}
\newcommand{\indclpr}[2][\cdot]{\underline{E}_{#2}(#1\vert\statevar{#2})}
\newcommand{\zinindclpr}[2][\cdot]{\underline{E}_{#2}(#1\vert\zstate{#2})}
\newcommand{\xinindclpr}[2][\cdot]{\underline{E}_{#2}(#1\vert\xstate{#2})}
\newcommand{\xhatinindclpr}[2][\cdot]{\underline{E}_{#2}(#1\vert\xhatstate{#2})}
\newcommand{\zinindcupr}[2][\cdot]{\overline{E}_{#2}(#1\vert\zstate{#2})}
\newcommand{\xinindcupr}[2][\cdot]{\overline{E}_{#2}(#1\vert\xstate{#2})}
\newcommand{\biggxinindcupr}[2][\cdot]{\overline{E}_{#2}\bigg(#1\Big\vert\xstate{#2}\bigg)}
\newcommand{\xhatinindcupr}[2][\cdot]{\overline{E}_{#2}(#1\vert\xhatstate{#2})}

\newcommand{\stategambles}[1]{\mathcal{G}(\states{#1})}
\newcommand{\fromstategambles}[1]{\stategambles{\from{#1}}}

\newcommand{\outgambles}[1]{\mathcal{G}(\outs{#1})}
\newcommand{\stateoutgambles}[2]{\mathcal{G}(\states{#1}\times\outs{#2})}
\newcommand{\fromstateoutgambles}[1]{\stateoutgambles{\from{#1}}{\from{#1}}}
\newcommand{\fromtostateoutgambles}[3]{\stateoutgambles{\fromto{#1}{#3}}{\fromto{#2}{#3}}}
\newcommand{\othergambles}{\mathcal{G}(\others)}

\newcommand{\xstate}[1]{x_{#1}}
\newcommand{\fromxstate}[1]{\xstate{\from{#1}}}
\newcommand{\fromtoxstate}[2]{\xstate{\fromto{#1}{#2}}}
\newcommand{\xhatstate}[1]{\hat{x}_{#1}}
\newcommand{\fromxhatstate}[1]{\xhatstate{\from{#1}}}
\newcommand{\fromtoxhatstate}[2]{\xhatstate{\fromto{#1}{#2}}}
\newcommand{\xbarstate}[1]{\overline{x}_{#1}}
\newcommand{\fromxbarstate}[1]{\xbarstate{\from{#1}}}
\newcommand{\zstate}[1]{z_{#1}}
\newcommand{\fromzstate}[1]{\zstate{\from{#1}}}
\newcommand{\fromtozstate}[2]{\xstate{\fromto{#1}{#2}}}

\newcommand{\out}[1]{o_{#1}}
\newcommand{\fromout}[1]{\out{\from{#1}}}
\newcommand{\fromtoout}[2]{\out{\fromto{#1}{#2}}}

\newcommand{\singxstate}[1]{\sing{\xstate{#1}}}
\newcommand{\singfromxstate}[1]{\sing{\fromxstate{#1}}}
\newcommand{\singxhatstate}[1]{\sing{\xhatstate{#1}}}

\newcommand{\singzstate}[1]{\sing{\zstate{#1}}}
\newcommand{\singfromzstate}[1]{\sing{\fromzstate{#1}}}

\newcommand{\singout}[1]{\sing{\out{#1}}}
\newcommand{\singfromout}[1]{\sing{\fromout{#1}}}

\newcommand{\indxstate}[1]{\indsing{\xstate{#1}}}
\newcommand{\indfromxstate}[1]{\indsing{\fromxstate{#1}}}
\newcommand{\indxhatstate}[1]{\indsing{\xhatstate{#1}}}
\newcommand{\indfromxhatstate}[1]{\indsing{\fromxhatstate{#1}}}
\newcommand{\indzstate}[1]{\indsing{\zstate{#1}}}
\newcommand{\indfromzstate}[1]{\indsing{\fromzstate{#1}}}

\newcommand{\indout}[1]{\indsing{\out{#1}}}
\newcommand{\indfromout}[1]{\indsing{\fromout{#1}}}

\newcommand{\globfromopt}[1][\fromstates{1}]{\optim\left(#1\vert\fromout{1}\right)}
\newcommand{\explfromopt}[4][\cdot]{\optim\left(#1\vert{#2}_{#3},\fromout{#4}\right)}
\newcommand{\fromopt}[3][\cdot]{\optim\left(#1\vert{#2}_{#3-1},\fromout{#3}\right)}
\newcommand{\opt}[3][\cdot]{\optim\left(#1\vert{#2}_{#3-1},\out{#3}\right)}
\newcommand{\explopt}[3][\cdot]{\optim\left(#1\vert{#2},\out{#3}\right)}

\newcommand{\frommog}[3][\cdot]{\mog\left(#1\vert{#2}_{#3-1},\fromout{#3}\right)}
\newcommand{\explfrommog}[4][\cdot]{\mog\left(#1\vert{#2}_{#3},\fromout{#4}\right)}
\newcommand{\frommogDoor}[4][\cdot]{\mog_{#4}\left(#1\vert{#2}_{#3-1},\fromout{#3}\right)}
\newcommand{\explfrommogDoor}[4][\cdot]{\mog_{#4}\left(#1\vert{#2},\fromout{#3}\right)}

\newcommand{\target}[1]{\Delta[\xstate{#1},\xhatstate{#1}]}
\newcommand{\fromtarget}[1]{\target{\from{#1}}}

\newcommand{\lmem}{\beta}
\newcommand{\umem}{\alpha}
\newcommand{\xlmem}[1]{\lmem(\xstate{#1})}
\newcommand{\xumem}[1]{\umem(\xstate{#1})}
\newcommand{\xlmemmax}[1]{\lmem_{#1}^{\mathrm{max}}(\xstate{#1})}
\newcommand{\xumemmax}[1]{\umem_{#1}^{\mathrm{max}}(\xstate{#1})}
\newcommand{\xumemmaxstar}[1]{\umem_{#1}^{\mathrm{max}}(\xstate{#1}^*)}
\newcommand{\zlmemmax}[1]{\lmem_{#1}^{\mathrm{max}}(\zstate{#1})}
\newcommand{\zumemmax}[1]{\umem_{#1}^{\mathrm{max}}(\zstate{#1})}
\newcommand{\explumemmax}[2]{\umem_{#1}^{\mathrm{max}}(#2)}

\newcommand{\xhatumem}[1]{\umem(\xhatstate{#1})}

\newcommand{\xhatumemmax}[1]{\umem_{#1}^{\mathrm{max}}(\xhatstate{#1})}

\newcommand{\zlmem}[1]{\lmem(\zstate{#1})}
\newcommand{\zumem}[1]{\umem(\zstate{#1})}
\newcommand{\specialumem}[2]{\umem_{#1}(#2)}
\newcommand{\fromxlmem}[1]{\lmem(\fromxstate{#1})}

\newcommand{\fromxumemstar}[1]{\umem_{#1}(\fromxstate{#1}^*)}

\newcommand{\fromxhatumem}[1]{\umem(\fromxhatstate{#1})}

\newcommand{\fromzlmem}[1]{\lmem(\fromzstate{#1})}
\newcommand{\fromzumem}[1]{\umem(\fromzstate{#1})}

\newcommand{\zinthreshold}[1]{\theta_{#1}(\xhatstate{#1},\xstate{#1}\vert\zstate{#1-1})}
\newcommand{\umemopt}[1]{\umem^{\mathrm{opt}}_{#1}(\xhatstate{#1}\vert\zstate{#1-1})}
\newcommand{\zumemopt}[1]{\umem^{\mathrm{opt}}_{#1}(\zstate{#1}\vert\zstate{#1-1})}
\newcommand{\xumemopt}[1]{\umem^{\mathrm{opt}}_{#1}(\xstate{#1}\vert\zstate{#1-1})}
\newcommand{\explxumemopt}[2]{\umem^{\mathrm{opt}}_{#1}(\xstate{#1}\vert#2)}
\newcommand{\xhatumemopt}[1]{\umem^{\mathrm{opt}}_{#1}(\xhatstate{#1}\vert\zstate{#1-1})}
\newcommand{\xhatinxhatumemopt}[1]{\umem^{\mathrm{opt}}_{#1}(\xhatstate{#1}\vert\xhatstate{#1-1})}
\newcommand{\specialumemopt}[2]{\umem^{\mathrm{opt}}_{#1}(#2\vert\zstate{#1-1})}
\newcommand{\explspecialumemopt}[3]{\umem^{\mathrm{opt}}_{#1}(#2\vert#3)}

\newcommand{\zumemoptfromto}[2]{\umem^{\mathrm{opt}}_{#1}(\zstate{\fromto{#1}{#2}}\vert\zstate{#1-1})}

\newcommand{\credalset}{{\mathcal{M}}}

\newcommand{\jointpr}[1]{{P}_{#1}}

\newtheorem{theorem}{Theorem}
\newtheorem{proposition}[theorem]{Proposition}

\newtheorem{lemma}[theorem]{Lemma}
\theoremstyle{definition}

\newtheorem{example}{Example}

\DeclareMathOperator{\optim}{opt}
\DeclareMathOperator*{\argmax}{argmax}
\DeclareMathOperator{\mog}{cand}
\DeclareMathOperator{\pos}{Pos}

\begin{document}
\title[Estimating state sequences in imprecise hidden Markov
models]{An efficient algorithm\\ for estimating state sequences in\\ imprecise hidden Markov models}
\author{Jasper De Bock}
\author{Gert de Cooman}
\begin{abstract}
  We present an efficient exact algorithm for estimating state sequences from outputs (or observations) in imprecise hidden Markov models (iHMM), where both the uncertainty linking one state to the next, and that linking a state to its output, are represented using coherent lower previsions.
  The notion of independence we associate with the credal network representing the iHMM is that of epistemic irrelevance. 
  We consider as best estimates for state sequences the (Walley--Sen) maximal sequences for the posterior joint state model conditioned on the observed output sequence, associated with a gain function that is the indicator of the state sequence. 
This corresponds to (and generalises) finding the state sequence with the highest posterior probability in HMMs with precise transition and output probabilities (pHMMs).
  We argue that the computational complexity is at worst quadratic in the length of the Markov chain, cubic in the number of states, and essentially linear in the number of maximal state sequences.
  For binary iHMMs, we investigate experimentally how the number of maximal state sequences depends on the model parameters. 
  We also present a simple toy application in optical character recognition, demonstrating that our algorithm can be used to robustify the inferences made by precise probability models.
\end{abstract}
\maketitle

\section{Introduction}

\label{sec:introduction}
In Artificial Intelligence, probabilistic graphical models are becoming an increasingly powerful tool. 
Amongst these, hidden Markov models (HMMs) are definitely amongst the simplest, and perhaps also amongst the more popular ones.

An important application for HMMs involves finding the \emph{sequence} of (hidden) states with the highest posterior probability after observing a sequence of outputs \cite{rabiner1989}. 
For HMMs with precise local transition and emission probabilities, there is a quite efficient dynamic programming algorithm, due to Viterbi \cite{rabiner1989,viterbi1967}, for performing this task. 
For imprecise-probabilistic local models, such as coherent lower previsions, we know of no algorithm in the literature for which the computational complexity comes even close to that of Viterbi's.

In this paper, we take the first steps towards remedying this situation. 
We describe imprecise hidden Markov models as special cases of credal trees (a special case of credal networks) under epistemic irrelevance in Section~\ref{sec:basics}. 
We show in particular how we can use the ideas underlying the MePiCTIr\footnote{MePiCTIr: \underline{Me}ssage \underline{P}assing \underline{i}n \underline{C}redal \underline{T}rees under \underline{Ir}relevance.} algorithm \cite{cooman2009}, involving independent natural extension and marginal extension, to construct a most conservative joint model from imprecise local transition and emission models.
We also derive a number of interesting and useful formulas from that construction.

The results in Section~\ref{sec:basics} assume basic knowledge of the theory of coherent lower previsions, a generalisation of classical probability that allows for incomplete specification of probabilities. 
We include a short introduction to this theory in Section~\ref{sec:lpr}.

In Section~\ref{sec:optimal-sequences} we explain how a sequence of observations leads to (a collection of) so-called maximal state sequences.
Finding all of them seems a daunting task at first: it has a search space that grows exponentially in the length of the Markov chain.
However, in Section~\ref{sec:principle-of-optimality} we use the basic formulas found in Section~\ref{sec:basics} to derive an appropriate version of Bellman's \cite{bellman1957} Principle of Optimality, which allows for an exponential reduction of the search space. 
By using a number of additional tricks, we are able in Section~\ref{sec:algorithm} to devise the EstiHMM\footnote{\underline{Est}imation in \underline{i}mprecise \underline{H}idden \underline{M}arkov \underline{M}odels} algorithm, which efficiently constructs all maximal state sequences. 
We prove in Section~\ref{sec:complexity} that this algorithm is essentially linear in the number of maximal sequences, quadratic in the length of the chain, and cubic in the number of states.
We perceive this complexity to be comparable to that of the Viterbi algorithm, especially after realising that the latter makes the simplifying step of resolving ties more or less arbitrarily in order to produce only a single optimal state sequence.
This is something we will not allow our algorithm to do, for reasons that should become clear further on.  

In Section~\ref{sec:experiments}, we consider the special case of binary iHMMs, and investigate experimentally how the number of maximal state sequences depends on the model parameters.
We comment on the very interesting structures that emerge, and give them an heuristic explanation. 

We show off the algorithm's efficiency in Section~\ref{sec:example} by calculating the maximal sequences for a specific iHMM of length $100$. 

We conclude in Section~\ref{sec:app} with a simple toy application in optical character recognition. 
It demonstrates the advantages of our algorithm and gives a clear indication that the EstiHMM algorithm is able to robustify the existing Viterbi algorithm in an intelligent manner.

In order to make our main argumentation as readable as possible, we have relegated all technical proofs to an appendix.

\section{Freshening up on coherent lower previsions}\label{sec:lpr}
We begin with some basic theory of coherent lower previsions. See Ref.~\cite{walley1991} for an in-depth study, and Ref.~\cite{miranda2008a} for a recent survey.

Coherent lower previsions are a special type of imprecise probability model. 
Roughly speaking, whereas classical probability theory assumes that a subject's uncertainty can be represented by a single probability mass function, the theory of imprecise probabilities effectively works with sets of possible probability mass functions, and thereby allows for imprecision as well as indecision to be modelled and represented. 
For people who are unfamiliar with the theory, looking at it as a way of robustifying the classical theory is perhaps the easiest way to understand and interpret it, and we will use this approach here.

Consider a set $\credalset$ of probability mass functions, defined on a discrete set $\states{}$ of possible states. 
With each mass function $p\in\credalset$, we can associate a \emph{linear prevision} (or expectation operator) $\jointpr{p}$, defined on the set $\stategambles{}$ of all real-valued maps on $\states{}$.
Any $f\in\stategambles{}$ is also called a \emph{gamble} on $\states{}$, and $\jointpr{p}(f)\coloneqq\sum_{\xstate{}\in\states{}}p(\xstate{})f(\xstate{})$ is the expected value of $f$, associated with the probability mass function $p$.
We can now define the \emph{lower prevision} $\jointlpr{\credalset}$ that corresponds with the set $\credalset$ as the following \emph{lower envelope} of linear previsions:
\begin{equation}\label{def:onderprevisie}
  \jointlpr{\credalset}(f)\coloneqq\inf\set{\jointpr{p}(f)}{p\in\credalset}
  \text{ for all gambles $f$ in $\states{}$}.
\end{equation}
Similarly, we define the \emph{upper prevision} $\jointupr{\credalset}$ as
\begin{align}
  \jointupr{\credalset}(f)
  &\coloneqq\sup\set{\jointpr{p}(f)}{p\in\credalset}\notag\\
  &=-\inf\set{-\jointpr{p}(f)}{p\in\credalset}
  =-\inf\set{\jointpr{p}(-f)}{p\in\credalset}
  =-\jointlpr{\credalset}(-f) 
  \label{eq:toegevoegdheid}
\end{align}
for all gambles $f$ on $\states{}$. 
We will mostly talk about lower previsions, since it follows from the \emph{conjugacy relation}~\eqref{eq:toegevoegdheid} that the two models are mathematically equivalent.

An \emph{event} $A$ is a subset of the set of possible values $\states{}$: $A\subseteq\states{}$.
With such an event, we can associate an \emph{indicator} $\ind{A}$, which is the gamble on $\states{}$ that assumes the value $1$ on $A$, and $0$ outside $A$.
We call
\begin{equation*}
  \jointlpr{\credalset}(A)\coloneqq\jointlpr{\credalset}(\ind{A})
  =\inf\biggset{\sum_{\xstate{}\in A}p(x)}{p\in\credalset}
\end{equation*}
the \emph{lower probability} of the event $A$, and similarly $\jointupr{\credalset}(A)\coloneqq\jointupr{\credalset}(\ind{A})$ its \emph{upper probability}.

It can be shown \cite{walley1991} that the functional $\jointlpr{\credalset}$ satisfies the following set of interesting mathematical properties, which define a \emph{coherent lower prevision}:
\begin{enumerate}[label=\upshape C\arabic*.,ref=\upshape C\arabic*,leftmargin=*]
\item\label{C1} $\jointlpr{\credalset}(f)\geq\min{f}$ for all $f\in\stategambles{}$, 
\item\label{C2} $\jointlpr{\credalset}(\lambda{f})=\lambda\jointlpr{\credalset}(f)$ for all $f\in\stategambles{}$ and all real $\lambda\geq0$,\hfill[non-negative homogeneity]
\item\label{C3} $\jointlpr{\credalset}(f+g)\geq\jointlpr{\credalset}(f)+\jointlpr{\credalset}(g)$ for all $f,g\in\stategambles{}$.\hfill[superadditivity]
\end{enumerate}
Every set of mass functions $\credalset$ uniquely defines a coherent lower prevision $\jointlpr{\credalset}$, but in general the converse does not hold. 
However, if we limit ourselves to sets of mass functions $\credalset$ that are closed and convex---which makes them \emph{credal sets}---they are in a one-to-one correspondence with coherent lower previsions \cite{walley1991}. 
This implies that we can use the theory of coherent lower previsions as a tool for reasoning with closed convex sets of probability mass functions.
From now on, we will no longer explicitly refer to credal sets $\credalset$, but we will simply talk about coherent lower previsions $\jointlpr{}$.
It is useful to keep in mind that there always is a unique credal set that corresponds with such a coherent lower prevision: $\jointlpr{}=\jointlpr{\credalset}$ for some unique credal set $\credalset$, given by $\credalset=\set{p}{(\forall f\in\stategambles{})\jointpr{p}(f)\geq\jointlpr{}(f)}$.

A special kind of imprecise model on $\states{}$ is the \emph{vacuous} lower prevision. 
It is a model that represents complete ignorance and therefore has the set of all possible mass functions on $\states{}$ as its credal set $\credalset$. 
It can be shown easily that for every $f\in\stategambles{}$, the corresponding lower prevision is given by $\jointlpr{}(f)=\min f$.

Conditional lower and upper previsions, which are extensions of the classical conditional expectation functionals, can be defined in a similar, intuitively obvious way as lower envelopes associated with sets of conditional mass functions.

Consider a variable $\statevar{}$ in $\states{}$ and a variable $\othervar$ in $\others$. 
A \emph{conditional lower prevision} $\jointlpr{}(\cdot\vert Y)$ on the set $\stategambles{}$ of all gambles on $\states{}$ is a two-place real-valued function. 
For any gamble $f$ on $\states{}$, $\jointlpr{}(f\vert Y)$ is a gamble on $\others$, whose value $\jointlpr{}(g\vert y)$ in $y\in\others$ is the lower prevision of $g$, \emph{conditional on the event $Y=y$}. 
If for any $y\in\others$, the lower prevision $\jointlpr{}(\cdot\vert y)$ is coherent---satisfies conditions~\ref{C1}--\ref{C3}---then we call the conditional lower prevision $\jointlpr{}(\cdot\vert Y)$ \emph{separately coherent}. 
It will sometimes be useful to extend the domain of the conditional lower prevision $\jointlpr{}(\cdot\vert y)$ from $\stategambles{}$ to $\mathcal{G}(\states{}\times\others)$ by letting $\jointlpr{}(f\vert y)\coloneqq\jointlpr{}(f(\cdot,y)\vert y)$ for all gambles $f$ on $\states{}\times\others$.

If we have a number of conditional lower previsions involving a number of variables, then each of them must be separately coherent, but we also have to make sure that they satisfy a more stringent \emph{joint coherence} requirement. 
Explaining this in detail would take us too far, but we refer to Ref.~\cite{walley1991} for a detailed discussion, with motivation. 
For our present purposes, it suffices to say that joint coherence is very closely related to making sure that these conditional lower previsions are lower envelopes associated with conditional mass functions that satisfy Bayes's Rule.

For a given lower prevision $\jointlpr{}$ on $\mathcal{G}(\states{}\times\others)$, a corresponding conditional lower prevision $\jointlpr{}(\cdot\vert Y)$ that is jointly coherent with $\jointlpr{}$ is not uniquely defined. 
It is however shown in Ref.~\cite{miranda2009a} that it always lies between the so-called natural and regular extensions. 

Using \emph{natural extension}, the conditional coherent lower prevision $\jointlpr{}(\cdot\vert Y)$ is defined by $\jointlpr{}(f\vert y)\coloneqq\max\set{\mu\in\reals}{\jointlpr{}(\indsing{y}[f-\mu])\geq0}$ if $\jointlpr{}(\{y\})>0$, and it is vacuous and thus given by $\jointlpr{}(f\vert y)\coloneqq\min f$ if $\jointlpr{}(\{y\})=0$. 
This is the smallest (most conservative) way of conditioning a lower prevision. 
If $\jointlpr{}(\{y\})>0$, it corresponds to conditioning every probability mass function in the credal set of $\jointlpr{}$ on the observation that $Y = y$ and taking the lower envelope of all these conditioned mass functions.

Using \emph{regular extension}, the conditional coherent lower prevision $\jointlpr{}(\cdot\vert Y)$ is defined by $\jointlpr{}(f\vert y)\coloneqq\max\set{\mu\in\reals}{\jointlpr{}(\ind{y}[f-\mu])\geq0}$ if $\jointupr{}(\{y\})>0$, and it is vacuous if $\jointupr{}(\{y\})=0$. 
This gives us the greatest (most informative) conditional lower prevision that is jointly coherent with the original unconditional lower prevision. 
It corresponds to taking all mass functions $p$ in the credal set of $\jointlpr{}$ for which $p(y)\neq0$, conditioning them on the observation that $Y=y$ and taking their lower envelope.

Natural and regular extension coincide if $\jointlpr{}(\{y\})>0$ or $\jointupr{}(\{y\})=0$ but are different if $\jointupr{}(\{y\})>\jointlpr{}(\{y\})=0$. 
In the latter case, natural extension is vacuous, but regular extension usually remains more informative.

In this introduction, coherent lower previsions were interpreted as an alternative representation for closed and convex sets of probability mass functions. 
This approach is often adopted by sensitivity analysts and is rather intuitive for people who are used to working in classical probability theory. 
For the sake of completeness, we mention here that coherent lower previsions can also be given a behavioural interpretation, without using the notion of a probability mass function. 
The lower prevision $\jointlpr{}(f)$ of a gamble $f\in\stategambles{}$ can be interpreted as the supremum acceptable buying price that a subject is willing to pay in order to gain the (possibly negative) reward $f(x)$ after the outcome $x\in\states{}$ of the experiment has been determined. 
See Ref.~\cite{walley1991} for more information regarding this interpretation.

\section{Basic notions}
\label{sec:basics}
An imprecise hidden Markov model can be depicted using the following probabilistic graphical model:
\par
\begin{figure}[h]
  \centering
  \begin{tikzpicture}
    \tikzset{node distance=\afstandknopen, auto}
    % knopen tekenen:
    \node[knoopstate] (1)  at (-1,1) {$\statevar{1}$};
    \node[knoopstate] (2) [right of = 1] {$\statevar{2}$};
    \node[knoopstate] (3)  at ($(2)+(1.5*\afstandknopen , 0)$) {$\statevar{k}$};
    \node[knoopstate] (4) at ($(3)+(1.5*\afstandknopen, 0)$)  {$\statevar{n}$};
    \node[knoopobservatie] (5) [below of = 1]  {$\outvar{1}$};
    \node[knoopobservatie] (6) [below of = 2] {$\outvar{2}$};
    \node[knoopobservatie] (7) [below of = 3]  {$\outvar{k}$};
    \node[knoopobservatie] (8) [below of = 4] {$\outvar{n}$};
    % pijlen tekenen:
    % pijlen van toestand naar toestand:
    \draw[pijl1] (1) -- (2) ;
    \draw[pijl4] (2) to node[] {} ($(2)+(\afstandvollelijn, 0)$);
    \draw[pijl1] ($(3)+(-\afstandvollelijn,0)$) to node[] {} (3);
    \draw[pijl3] ($(2)+(\afstandvollelijn, 0)$) to node[] {} ($(3)+(-\afstandvollelijn,0)$) ;
    \draw[pijl4] (3) to node[] {} ($(3)+(\afstandvollelijn, 0)$);
    \draw[pijl1] ($(4)+(-\afstandvollelijn,0)$) to node[] {} (4);
    \draw[pijl3] ($(3)+(\afstandvollelijn, 0)$) to node[] {} ($(4)+(-\afstandvollelijn,0)$) ;
    \draw[pijl1] (1) -- (5) ;
    \draw[pijl1] (2) -- (6) ;
    \draw[pijl1] (3) -- (7) ;
    \draw[pijl1] (4) -- (8) ;
    % modellen erbij zetten
    \draw (1) node [above=15pt]{$\statelprone{}$};
    \draw (2) node [above=15pt]{$\stateclpr{2}{1}$};
    \draw (3) node [above=15pt]{$\stateclpr{k}{k-1}$};
    \draw (4) node [above=15pt]{$\stateclpr{n}{n-1}$};
    \draw (5) node [below=17pt]{$\outclpr{1}$};
    \draw (6) node [below=17pt]{$\outclpr{2}$};
    \draw (7) node [below=17pt]{$\outclpr{k}$};
    \draw (8) node [below=17pt]{$\outclpr{n}$};
    
    \draw (1) node [left=20pt]{State sequence:};
    \draw (5) node [left=20pt]{Output sequence:};
    
  \end{tikzpicture}
  \caption{Tree representation of a hidden Markov model}
  \label{fig:hmmtree}
\end{figure}
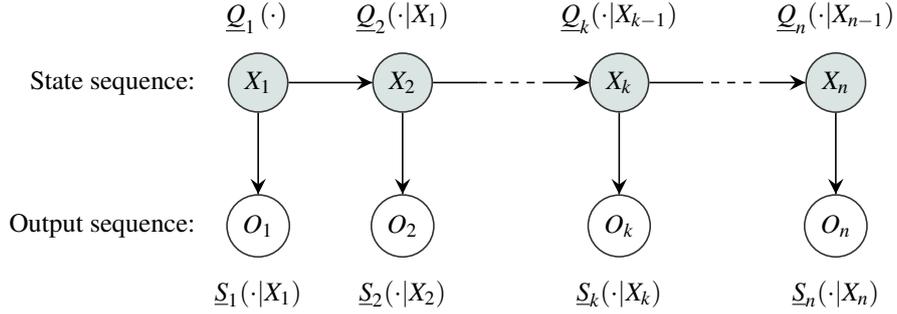
\noindent
Here $n$ is some natural number.
The \emph{state variables} $\statevar{1}$, \dots, $\statevar{n}$ assume values in the respective finite sets $\states{1}$, \dots, $\states{n}$, and the \emph{output variables}  $\outvar{1}$, \dots, $\outvar{n}$ assume values in the respective finite sets $\outs{1}$, \dots, $\outs{n}$.
We denote generic values of $\statevar{k}$ by $\xstate{k}$, $\xhatstate{k}$ or $\zstate{k}$, and generic values of $\outvar{k}$ by $\out{k}$.

\subsection{Local uncertainty models}\label{sec:local}
We assume that we have the following local uncertainty models for these variables.
For $\statevar{1}$, we have a \emph{marginal} lower prevision $\statelpr{1}$, defined on the set $\stategambles{1}$ of all real-valued maps (or \emph{gambles}) on $\states{1}$.
For the subsequent states $\statevar{k}$, with $k\in\{2,\dots,n\}$, we have a conditional lower prevision $\stateclpr{k}{k-1}$ defined on $\stategambles{k}$, called a \emph{transition model}.
In order to maintain uniformity of notation, we will also denote the marginal lower prevision $\statelpr{1}$ as a conditional lower prevision $\stateclpr{1}{0}$, where $\statevar{0}$ denotes a variable that may only assume a single value, and whose value is therefore certain. 
For any gamble $f_k$ in $\stategambles{k}$, $\stateclpr[f_k]{k}{k-1}$ is interpreted as a gamble on $\states{k-1}$, whose value $\zinstateclpr[f_k]{k}{k-1}$ in any $\zstate{k-1}\in\states{k-1}$ is the lower prevision of the gamble $f_k(\statevar{k})$, conditional on $\statevar{k-1}=\zstate{k-1}$.
\par
In addition, for each output $\outvar{k}$, with $k\in\{1,\dots,n\}$, we have a conditional lower prevision $\outclpr{k}$ defined on $\outgambles{k}$, called an \emph{emission model}.  
For any gamble $g_k$ in $\outgambles{k}$, $\outclpr[g_k]{k}$ is interpreted as a gamble on $\states{k}$, whose value $\zinoutclpr[g_k]{k}$ in any $\zstate{k}\in\states{k}$ is the lower prevision of the gamble $g_k(\outvar{k})$, conditional on $\statevar{k}=\zstate{k}$.
\par
We take all these local (marginal, transition and emission) uncertainty models to be \emph{separately coherent}. 
Recall that this simply means that for any $k\in\{1,\dots,n\}$, the lower prevision $\zinstateclpr{k}{k-1}$ should be coherent (as an unconditional lower prevision) for every $\zstate{k-1}\in\states{k-1}$ and $\zinoutclpr{k}$ should be coherent for every $\zstate{k}\in\states{k}$.

\subsection{Interpretation of the graphical structure}\label{sec:interpretation}
We will assume that the graphical representation in Figure~\ref{fig:hmmtree} represents the following irrelevance assessments: \emph{conditional on its mother variable, the non-parent non-descendants of any variable in the tree are epistemically irrelevant to this variable and its descendants.}
We say that a variable $\statevar{}$ is \emph{epistemically irrelevant} to a variable $\othervar$ if observing $\statevar{}$ does not affect our beliefs about $\othervar$. 
Mathematically stated in terms of lower previsions: $\jointlpr{}(f(\othervar))=\xinjointclpr[f(\othervar)]{}{}$ for all $f\in\othergambles$ and all $x\in\states{}$.

Before we go on, it will be useful to introduce some mathematical short-hand notation for describing joint variables in the tree of Figure~\ref{fig:hmmtree}.
For any $1\leq k\leq\ell\leq n$, we denote the tuple $(\statevar{k},\statevar{k+1},\dots,\statevar{\ell})$ by $\fromtostatevar{k}{\ell}$, and 
the tuple $(\outvar{k},\outvar{k+1},\dots,\outvar{\ell})$ by $\fromtooutvar{k}{\ell}$.
$\fromtostatevar{k}{\ell}$ is a (joint) variable that can assume all values in the set $\fromtostates{k}{\ell}\coloneqq\times_{r=k}^\ell\states{r}$, and $\fromtooutvar{k}{\ell}$ is a (joint) variable that can assume all values in the set $\fromtoouts{k}{\ell}\coloneqq\times_{r=k}^\ell\outs{r}$. 
Generic values of $\fromtostatevar{k}{\ell}$ are denoted by $\fromtoxstate{k}{\ell}$ or $\fromtozstate{k}{\ell}$, and generic values of $\fromtooutvar{k}{\ell}$ by $\fromtoout{k}{\ell}$.

\begin{example}
  Consider the variable $\statevar{k}$ with mother variable $\statevar{k-1}$ in Figure~\ref{fig:hmmtree}. 
  The variables $\fromtostatevar{1}{k-2}$ and $\fromtooutvar{1}{k-1}$ are its non-parent non-descendants, and the variables $\fromtostatevar{k+1}{n}$ and $\fromtooutvar{k}{n}$ its descendants. 
  Our interpretation of the graphical structure of Figure~\ref{fig:hmmtree} implies that once we know (conditional on) the value $\xstate{k_1}$ of $\statevar{k-1}$, additionally learning the values of any of the variables $\statevar{1}$, \dots, $\statevar{k-2}$ and $\outvar{1}$, \dots, $\outvar{k-1}$ will not change our beliefs about $\fromtostatevar{k}{n}$ and $\fromtooutvar{k}{n}$.  
  \hfill $\blacklozenge$
\end{example}

Epistemic irrelevance is weaker than the so-called \emph{strong independence} condition that is usually associated with \emph{credal networks} \cite{cozman2000}, which is the name usually given to probabilistic graphical models with coherent lower previsions as local uncertainty models.
Recent work \cite{cooman2009} has shown that using this weaker condition guarantees that an efficient algorithm exists for updating beliefs about a single target node of a credal \emph{tree}, that is essentially linear in the number of nodes in the tree.

\subsection{A joint uncertainty model}\label{sec:jointmodel}
Using the local uncertainty models, we now want to construct a global model: a joint lower prevision $\jointlpr{}$ on $\fromstateoutgambles{1}$ for all the variables $(\fromtostatevar{1}{n},\fromtooutvar{1}{n})$ in the tree. 
This joint lower prevision should (i) be jointly coherent with all the local models; (ii) encode all epistemic irrelevance assessments encoded in the tree; and (iii) be as small, or conservative,\footnote{Recall that point-wise smaller lower previsions correspond to larger credal sets.} as possible.
This is a special case of a more general problem for credal trees, discussed and solved in great detail in Ref.~\cite{cooman2009}.
In this section, we summarise the solution for iHMMs and give an heuristic justification for it, but we refer to Ref.~\cite{cooman2009} for a proof that the joint model we present below is indeed the most conservative lower prevision that is coherent with all the local models and captures all epistemic irrelevance assessments encoded in the tree.

We proceed in a recursive manner, and consider any $k\in\{1,\dots,n\}$. 
For any $\xstate{k-1}\in\states{k-1}$, we consider the smallest coherent joint lower prevision $\xinjointclpr{k}{k-1}$ on $\fromstateoutgambles{k}$ for the variables $(\fromstatevar{k},\fromoutvar{k})$ on the iHMM depicted in Figure~\ref{fig:hmmsubtree}, representing a subtree of the tree represented in Figure~\ref{fig:hmmtree}, with the lower prevision $\xinstateclpr{k}{k-1}$ acting as the marginal model for the `first' state variable $\statevar{k}$.
Note that the global model $\jointlpr{}$ we are looking for can be identified with the conditional lower prevision $\jointclpr{1}{0}$, for the reasons given in Section~\ref{sec:local}. 

\begin{figure}[h]
  \centering
  \begin{tikzpicture}[scale=0.9]
    \tikzset{node distance=\tssknoop, auto}
    % knopen tekenen:
    \node[knoop1,minimum size=1cm] (1)  at (-1,1) {$\statevar{k}$};
    \node[knoop1,minimum size=1cm] (2)  at ($(1)+(1.5*\tssknoop,0)$) {$\statevar{k+1}$};
    \node[knoop2,minimum size=1cm] (3) [below of = 1]  {$\outvar{k}$};
    \node[knoop2,minimum size=1cm] (4) [below of = 2] {$\outvar{k+1}$};
    % pijlen tekenen:
    \draw[pijl1] (1) -- (2) ;
    \draw[pijl1] (1) -- (3) ;
    \draw[pijl1] (2) -- (4) ;
    % pijl met stippellijn:
    \draw[pijl1] ($(1)+(-0.5*\tssknoop,0)$) to node[] {} (1) ;
    \draw[pijl3] ($(1)+(-\tssknoop,0)$) to node[] {} ($(1)+(-0.5*\tssknoop, 0)$);
    \draw[pijl4] (2) to node[] {} ($(2)+(0.5*\tssknoop,0)$) ;
    \draw[pijl3] ($(2)+(0.5*\tssknoop,0)$) to node[] {} ($(2)+(\tssknoop,0)$);   
    % modellen bijzetten:
    \draw (1) node [above=18pt]{$\xinstateclpr{k}{k-1}$};
    \draw (2) node [above=18pt]{$\stateclpr{k+1}{k}$};
    \draw (3) node [below=18pt]{$\outclpr{k}$};
    \draw (4) node [below=18pt]{$\outclpr{k+1}$};
    % omcirkeling:
    \draw[black,rounded corners=.8cm,loosely dashed, line width=2pt] ($(2)+(\tssknoop,0.6*\tssknoop)$) -- ($(2)+(-0.6*\tssknoop,0.6*\tssknoop)$) -- ($(4)+(-0.6*\tssknoop,-0.6*\tssknoop)$) --  ($(4)+(\tssknoop,-0.6*\tssknoop)$) ;
    % omcirkeling:
    \draw[black,rounded corners=.8cm, loosely dotted, line width=1.5pt]
    ($(2)+(\tssknoop,0.7*\tssknoop)$) --
    ($(2)+(-0.7*\tssknoop,0.7*\tssknoop)$) --
    ($(4)+(-0.7*\tssknoop,0.5*\tssknoop)$) --
    ($(3)+(-0.6*\tssknoop,0.5*\tssknoop)$) --
    ($(3)+(-0.6*\tssknoop,-0.7*\tssknoop)$) --
    ($(4)+(\tssknoop,-0.7*\tssknoop)$) ;
    \draw[black,rounded corners=.8cm, dotted, line width=.6pt]
    ($(2)+(\tssknoop,0.8*\tssknoop)$) --
    ($(1)+(-0.7*\tssknoop,0.8*\tssknoop)$) --
    ($(3)+(-0.7*\tssknoop,-0.8*\tssknoop)$) --
    ($(4)+(\tssknoop,-0.8*\tssknoop)$) ;
    % legende maken:
    \node[draw=black,loosely dashed, line width=2pt, minimum width=2.57cm,minimum height=1.2cm,
    rounded corners=.5cm] (5) at (8.5,-2) {$\jointclpr{k+1}{k}$};
    \node[draw=black,loosely dotted, line width=1.5pt, minimum width=2.57cm,minimum height=1.2cm,
    rounded corners=.5cm] (6) at ($(5)+(0,1.5cm)$) {$\indclpr{k}$};
    \node[draw=black,dotted, line width=.6pt, minimum width=2.57cm,minimum height=1.2cm,
    rounded corners=.5cm] (7) at ($(6)+(0,1.5cm)$) {$\jointclpr{k}{k-1}$};
  \end{tikzpicture}
  \caption{Subtree of the iHMM involving the variables $(\fromstatevar{k},\fromoutvar{k})$}
  \label{fig:hmmsubtree}
\end{figure}

Our aim is to develop recursive expressions that enable us to construct $\jointclpr{k}{k-1}$ out of $\jointclpr{k+1}{k}$. 
Using these expressions over and over again will eventually yield the global model $\jointlpr{}=\jointclpr{1}{0}$.
 
In a first step, we combine the joint model $\jointclpr{k+1}{k}$ for the variables $(\fromstatevar{k+1},\fromoutvar{k+1})$, defined on $\fromtostateoutgambles{k+1}{k+1}{n}$---see the thick dotted lines in Figure~\ref{fig:hmmsubtree}---,with the local model $\outclpr{k}$ for the variable $\outvar{k}$, defined on $\outgambles{k}$.
This will lead to a joint model $\indclpr{k}$ for the variables $(\fromstatevar{k+1},\fromoutvar{k})$, defined on $\fromtostateoutgambles{k+1}{k}{n}$---see the semi-thick dotted lines in Figure~\ref{fig:hmmsubtree}.
This is trivial for $k=n$, since we must have that $\indclpr{n}=\outclpr{n}$. 

For $k\neq n$, the solution is less obvious.  
A joint model can be constructed in many different ways, so we will have to impose some conditions.  
A first condition is that $\indclpr{k}$ should be a separately coherent conditional lower prevision that is jointly coherent with the `marginal' models $\jointclpr{k+1}{k}$ and $\outclpr{k}$.
A second, rather obvious, condition is that $\indclpr{k}$ should coincide with $\jointclpr{k+1}{k}$ and $\outclpr{k}$ on their respective domains. 
A third condition is that the model should capture the epistemic irrelevance assessments encoded in the tree. 
In particular these state that, conditional on $\statevar{k}$, the two variables $(\fromstatevar{k+1},\fromoutvar{k+1})$ and $\outvar{k}$ should be \emph{epistemically independent}, or in other words, epistemically irrelevant to one another.

Any model that meets all these conditions is called a (conditionally) \emph{independent product} \cite{cooman2011a} of $\jointclpr{k+1}{k}$ and $\outclpr{k}$. 
Generally speaking, such a (conditionally) independent product is not unique. 
We call the point-wise smallest, most conservative, of all possible (conditionally) independent products, which always exists, the (conditionally) \emph{independent natural extension} \cite{walley1991,cooman2011a} of $\jointclpr{k+1}{k}$ and $\outclpr{k}$, and we denote it as $\jointclpr{k+1}{k}\otimes\outclpr{k}$.

Summarising, $\indclpr{k}$ is given by
\begin{equation}\label{eq:indrecurse}
  \indclpr{k}
  \coloneqq
  \begin{cases}
  	\outclpr{n}
    &k=n\\
    \outclpr{k}\otimes\jointclpr{k+1}{k}
    &k=n-1,\dots,1
  \end{cases}
\end{equation}

The (conditionally) independent natural extension and its properties were studied in great detail in Ref.~\cite{cooman2011a}.
For the purposes of this paper, it will suffice to recall from that study that---very much like independent products of precise probability models---such independent natural extensions are \emph{factorising}, which implies in particular that
\begin{align}
  \zinindclpr[fg]{k}
  &=\zinindclpr[g\zinindclpr[f]{k}]{k}
  \notag\\
  &=\zinoutclpr[g{\zinjointclpr[f]{k+1}{k}}]{k}
  \notag\\
  &=
  \begin{cases}
    \zinoutclpr[g]{k}\zinjointclpr[f]{k+1}{k}
    &\text{ if $\zinjointclpr[f]{k+1}{k}\geq0$}\\
    \zinoutcupr[g]{k}\zinjointclpr[f]{k+1}{k}
    &\text{ if $\zinjointclpr[f]{k+1}{k}\leq0$}
  \end{cases}
  \notag\\
  &=\zinoutclupr[g]{k}\odot\zinjointclpr[f]{k+1}{k},
  \label{eq:factorisation}
\end{align}
for all $\zstate{k}\in\states{k}$, all $f\in\fromstateoutgambles{k+1}$ and all \emph{non-negative} $g\in\outgambles{k}$---we call a gamble non-negative if all its values are.
In this expression, the first equality is the actual factorisation property.
The second equality holds because $\indclpr{k}$ coincides with $\jointclpr{k+1}{k}$ and $\outclpr{k}$ on their respective domains. 
The third equality follows from the conjugacy relation~\eqref{eq:toegevoegdheid} and coherence condition \ref{C2}, and for the fourth we have used the shorthand notation $\signedprod{m}{x}\coloneqq\underline{m}\max\{0,x\}+\overline{m}\min\{0,x\}$. 
Further on, we will also use the analogous notation $\overline{\underline{m}}\,\signedprod{n}{x}\coloneqq\underline{m}\,\underline{n}\max\{0,x\}+\overline{m}\,\overline{n}\min\{0,x\}$.

In a second and final step, we combine the joint model $\indclpr{k}$ for the variables $(\fromstatevar{k+1},\fromoutvar{k})$, defined on $\fromtostateoutgambles{k+1}{k}{n}$, with the local model $\xinstateclpr{k}{k-1}$ for the variable $\statevar{k}$, defined on $\stategambles{k}$, into the joint model $\jointclpr{k}{k-1}$ for the variables $(\fromstatevar{k},\fromoutvar{k})$, defined on $\fromtostateoutgambles{k}{k}{n}$.
It has been shown elsewhere \cite{walley1991,miranda2006b} that the most conservative coherent way of doing this, is by means of \emph{marginal extension}, also known as the law ot iterated (lower) expectations.
This leads to $\xinjointclpr{k}{k-1}\coloneqq\xinstateclpr[\indclpr{k}]{k}{k-1}$, or, if we now allow $\xstate{k-1}$ to range over $\states{k-1}$:
\begin{equation}\label{eq:jointrecurse}
  \jointclpr{k}{k-1}\coloneqq\stateclpr[\indclpr{k}]{k}{k-1}.
\end{equation}
For practical purposes, it is useful to see that this is equivalent with
\begin{equation*}
  \jointclpr[f]{k}{k-1}
  =\statelpr{k}\bigg(
  \sum_{\zstate{k}\in\states{k}}\indsing{\zstate{k}}\zinindclpr[{f(\zstate{k},\fromstatevar{k+1},\outvar{k:n})}]{k}
  \Big\vert\statevar{k-1}
  \bigg)
\end{equation*}
for all $f\in\fromstateoutgambles{k}$.
Recall that in this expression, the \emph{indicator} $\indsing{\zstate{k}}$ is a gamble on $\states{k}$ that assumes the value $1$ if $\statevar{k}=\zstate{k}$ and $0$ if $\statevar{k}\neq\zstate{k}$. 

\subsection{Interesting lower and upper probabilities}\label{sec:assumption}
Without too much trouble,\footnote{As an example, we derive Equations~\eqref{eq:lower:state:out:mass} and~\eqref{eq:upper:state:out:mass} in Appendix~\ref{appendix}.}, we can use Equations~\eqref{eq:indrecurse}--\eqref{eq:jointrecurse} to derive the following expressions for a number of interesting lower and upper probabilities:
% \begin{align*}
%   % \label{eq:lower:state:mass}
%   \zinjointclpr[\singfromzstate{k}]{k}{k-1}
%   &=\prod_{i=k}^n\zinstateclpr[\singzstate{i}]{i}{i-1}\\
%   % \label{eq:upper:state:mass}
%   \zinjointcupr[\singfromzstate{k}]{k}{k-1}
%   &=\prod_{i=k}^n\zinstatecupr[\singzstate{i}]{i}{i-1},
% \end{align*}
\begin{align}
  \zinjointclpr[{\singfromout{k}\times\singfromzstate{k}}]{k}{k-1}
  % =\zinjointclpr[{\indfromout{k}\indfromzstate{k}}]{k}{k-1}
  &=\prod_{i=k}^n\zinoutclpr[\singout{i}]{i}\zinstateclpr[\singzstate{i}]{i}{i-1}
  \label{eq:lower:state:out:mass}\\
  \zinjointcupr[{\singfromout{k}\times\singfromzstate{k}}]{k}{k-1}
  % =\zinjointcupr[{\indfromout{k}\indfromzstate{k}}]{k}{k-1}
  &=\prod_{i=k}^n\zinoutcupr[\singout{i}]{i}\zinstatecupr[\singzstate{i}]{i}{i-1}
  \label{eq:upper:state:out:mass}
\end{align}
for all $\zstate{k-1}\in\states{k-1}$, $\fromzstate{k}\in\fromstates{k}$, $\fromout{k}\in\fromouts{k}$ and $k\in\{1,\dots,n\}$, and
\begin{align}
  \zinindclpr[{\singfromout{k}\times\singfromzstate{k+1}}]{k}
  % =\zinindclpr[{\indfromout{k}\indfromzstate{k+1}}]{k}
  &=\zinoutclpr[\singout{k}]{k}\prod_{i=k+1}^n\zinoutclpr[\singout{i}]{i}\zinstateclpr[\singzstate{i}]{i}{i-1}
  \label{eq:Eonder}\\
  \zinindcupr[{\singfromout{k}\times\singfromzstate{k+1}}]{k}
  % =\zinindcupr[{\indfromout{k}\indfromzstate{k+1}}]{k}
  &=\zinoutcupr[\singout{k}]{k}\prod_{i=k+1}^n\zinoutcupr[\singout{i}]{i}\zinstatecupr[\singzstate{i}]{i}{i-1}.
  \label{eq:Eboven}
\end{align}
for all $\zstate{k}\in\states{k}$, $\fromzstate{k+1}\in\fromstates{k+1}$, $\fromout{k}\in\fromouts{k}$ and $k\in\{1,\dots,n\}$.
We will assume throughout that 
\begin{equation*}
  \jointupr{}(\singfromzstate{1}\times\singfromout{1})>0
  \text{ for all $\fromzstate{1}\in\fromstates{1}$ and $\fromout{1}\in\fromouts{1}$}
\end{equation*}
or equivalently, that all \emph{local upper previsions are positive}, in the sense that  \cite{cooman2009}:
\begin{multline}\label{eq:assumption}
  \zinstatecupr[\singzstate{k}]{k}{k-1}>0
  \text{ and }
  \zinoutcupr[\singout{k}]{k}>0\\
  \text{ for all $\zstate{k-1}\in\states{k-1}$, $\zstate{k}\in\states{k}$, $\out{k}\in\outs{k}$ and $k\in\{1,\dots,n\}$.}
\end{multline}
This assumption is very weak and not at all restrictive for practical purposes. 
The imprecise-probabilistic local models are usually constructed by adding some margin of error around a precise model, thereby making all upper transition probabilities positive by construction. 
We will however allow lower transition probabilities to be zero, which is something that does happen often in practical problems.

\begin{proposition}\label{prop:PenEpos}
  The assumption~\eqref{eq:assumption} that all local upper previsions are positive implies that $\zinjointcupr[\singfromout{k}]{k}{k-1}>0$ and $\zinindcupr[\singfromout{k}]{k}>0$ for all $k\in\{1,\dots,n\}$, $\zstate{k}\in\states{k}$, $\zstate{k-1}\in\states{k-1}$ and $\fromout{k}\in\fromouts{k}$.  
\end{proposition}

\section{Estimating states from outputs}
\label{sec:optimal-sequences}
In a hidden Markov model, the states are not directly observable, but the outputs are, and the general aim is to use the outputs to estimate the states.
We concentrate on the following problem: \emph{Suppose we have observed the output sequence~$\fromout{1}$, estimate the state sequence~$\fromxstate{1}$.}
We will use an essentially Bayesian approach to do so, but need to allow for the fact that we are working with imprecise rather than precise probability models. 

\subsection{Updating the iHMM}\label{sec:updatingtheihmm}
The first step in our approach consists in updating (or conditioning) the joint model $\jointlpr{}\coloneqq\jointclpr{1}{0}$ on the observed outputs $\outvar{1:n}=\fromout{1}$.  
As mentioned in Section~\ref{sec:lpr}, there is no unique coherent way to perform this updating. 
However, for the particular problem we are solving in this paper, it so happens that it makes no difference which updating method is used, as long as it is coherent. 
For the time being, we choose to use the least conservative\footnote{The most conservative coherent way yields a vacuous model.} (most informative) coherent updating method, which is \emph{regular extension}. 
Later on in Section~\ref{sec:maximalstatesequences}, we will show that any other coherent updating method yields the same results.
% We choose to use the so-called \emph{regular extension}, which is the least conservative coherent way to perform this updating \cite{walley1991} and reduces to Bayes's Rule when all local models are precise. *** More on this in the Introduction, and then adapt what is being said here to reflect that ***
%\footnote{The most conservative coherent way yields a vacuous model.}
% \footnote{Regular extension is well-defined here since
% $\jointupr{}(\singfromout{1})>0$ by assumption~\eqref{eq:assumption}
% and Proposition~\ref{prop:PenEpos}.}

Since it follows from the positivity assumption~\eqref{eq:assumption} and Proposition~\ref{prop:PenEpos} that $\jointupr{}(\singfromout{1})>0$,
regular extension leads us to consider the updated lower prevision $\jointlpr{}(\cdot\vert\fromout{1})$ on $\fromstategambles{1}$, given by:
\begin{equation}\label{eq:GBR}
  \jointlpr{}(f\vert\fromout{1})
  \coloneqq\max\set{\mu\in\reals}{\jointlpr{}(\indfromout{1}[f-\mu])\geq0}
  \text{ for all gambles $f$ on $\fromstates{1}$.}
\end{equation}
Using the coherence of the joint lower prevision $\jointlpr{}$, it is not hard to prove that when $\jointlpr{}(\singfromout{1})>0$, $\jointlpr{}(\indfromout{1}[f-\mu])$ is a strictly decreasing and continuous function of $\mu$, which therefore has a unique zero (see Lemma~\ref{lemma:rho}\ref{eig:nscc}\&\ref{eig:dalend} in Appendix~\ref{appendix}). 
As a consequence, we have for any $f\in\fromstategambles{1}$ that
\begin{equation}\label{eq:crucial}
  \jointlpr{}(f\vert\fromout{1})\leq0
  \asa
  (\forall\mu>0) \jointlpr{}(\indfromout{1}[f-\mu])<0
  \asa
  \jointlpr{}(\indfromout{1}f)\leq0.
\end{equation}
In fact, it is not hard to infer from the strictly decreasing and continuous character of  $\jointlpr{}(\indfromout{1}[f-\mu])$ that $\jointlpr{}(f\vert\fromout{1})$ and $\jointlpr{}(\indfromout{1}f)$ have the same sign.
They are either both negative, both positive or both equal to zero;
see also the illustration below.
\begin{center}
  \begin{tikzpicture}[domain=-1:1.5]
    \coordinate (a) at (1,0);
    \coordinate (b) at (0,1);
    \draw[->] (-2,0) -- (4,0) node[below right] {$\mu$}; 
    \draw[->] (0,-1) -- (0,2); 
    \draw[semithick] plot (\x,{(9-(\x+2)^2)/5}) node[right] {$\jointlpr{}(\indfromout{1}[f-\mu])$};
    \node[below of=a,node distance=40] (c) {$\jointlpr{}(f\vert\fromout{1})$};
    \path[commentlink] (a) -- (c);
    \node[circle,inner sep=1.5pt,fill=gray] at (a) {};
    \node[left of=b,node distance=50] (d) {$\jointlpr{}(\indfromout{1}f)$};
    \path[commentlink] (b) -- (d);
    \node[circle,inner sep=1.5pt,fill=gray] at (b) {};
  \end{tikzpicture}
\end{center}
Equation~\eqref{eq:crucial} will be of crucial importance further on. 
However, in general, we want to allow $\jointlpr{}(\singfromout{1})$ to be zero (because this may happen if you allow lower transition
probabilities to be zero), while requiring that $\jointupr{}(\singfromout{1})>0$ (because this follows from the positivity assumption~\eqref{eq:assumption} and Proposition~\ref{prop:PenEpos}).
This will, generally speaking, invalidate the second equivalence in Equation~\eqref{eq:crucial}: it turns into an implication only.
But, if we limit ourselves to the specific type of gambles on $\fromstates{1}$ of the form $f=\indfromxhatstate{1}-\indfromxstate{1}$, we can still prove the following important theorem.

\begin{theorem}\label{theorem:crucial} 
  If all local upper previsions are positive, then $\jointlpr{}(\indfromout{1}[\indfromxstate{1}-\indfromxhatstate{1}])$ and $\jointlpr{}(\indfromxstate{1}-\indfromxhatstate{1}\vert\fromout{1})$ have the same sign for all fixed values of $\fromxstate{1},\fromxhatstate{1}\in\fromstates{1}$ and $\fromout{1}\in\fromouts{1}$.
  They are both positive, both negative or both zero.  
\end{theorem}

\subsection{Maximal state sequences}\label{sec:maximalstatesequences}
The next step now consists in using the posterior model $\jointlpr{}(\cdot\vert\fromout{1})$ to find best estimates for the state sequence $\fromxstate{1}$.
On the Bayesian approach, this is usually done by solving a decision-making, or optimisation problem: we associate a gain function $\indfromxstate{1}$ with every candidate state sequence $\fromxstate{1}$, and select as best estimates those state sequences $\fromxhatstate{1}$ that maximise the posterior expected gain, resulting in state sequences with maximal posterior probability.
\par
Here we generalise this decision-making approach towards working with imprecise probability models.
The criterion we use to decide which estimates are optimal for the given gain functions is that of (Walley--Sen) \emph{maximality} \cite{troffaes2007,walley1991}. 
Maximality has a number of very desirable properties that make sure it works well in optimisation contexts \cite{cooman2005a,huntley2011}, and it is well-justified from a behavioural point of view, as well as in a robustness approach, as we shall see presently.
\par
We can express a strict preference $\succ$ between two state sequence estimates $\fromxhatstate{1}$ and $\fromxstate{1}$ as follows:
\begin{equation*}
  \fromxhatstate{1}\succ\fromxstate{1}
  \asa\jointlpr{}(\indfromxhatstate{1}-\indfromxstate{1}\vert\fromout{1})>0.
\end{equation*}
On a behavioural interpretation, this expresses that a subject with lower prevision $\jointlpr{}(\cdot\vert\fromout{1})$ is disposed to pay some strictly positive amount of utility to replace the (gain associated with the) estimate  $\fromxstate{1}$ with the (gain associated with the) estimate $\fromxhatstate{1}$; see Ref.~\cite[Section~3.9]{walley1991} for more details.
Alternatively, from a robustness point of view, this expresses that for each conditional mass function $p(\cdot\vert\fromout{1})$ in the credal set associated with the updated lower prevision $\jointlpr{}(\cdot\vert\fromout{1})$, the state sequence $\fromxhatstate{1}$ has a posterior probability $p(\fromxhatstate{1}\vert\fromout{1})$ that is \emph{strictly higher} than the posterior probability $p(\fromxstate{1}\vert\fromout{1})$ of the state sequence $\fromxstate{1}$.

The binary relation $\succ$ thus defined is a strict partial order [an irreflexive and transitive binary relation] on the set of state sequences $\fromstates{1}$, and we consider an estimate $\fromxhatstate{1}$ to be \emph{optimal} when it is \emph{undominated}, or \emph{maximal}, in this strict partial order: 
\begin{align}\label{eq:globoptimals}
  \fromxhatstate{1}\in\globfromopt
  &\asa(\forall\fromxstate{1}\in\fromstates{1})
  \fromxstate{1}\not\succ\fromxhatstate{1}\notag\\
  &\asa(\forall\fromxstate{1}\in\fromstates{1})
  \jointlpr{}(\indfromxstate{1}-\indfromxhatstate{1}\vert\fromout{1})\leq0\notag\\
  &\asa(\forall\fromxstate{1}\in\fromstates{1})
  \jointlpr{}(\indfromout{1}[\indfromxstate{1}-\indfromxhatstate{1}])\leq0,
\end{align}
where the very useful last equivalence follows from Theorem~\ref{theorem:crucial}.
\emph{In summary then, the aim of this paper is to develop an efficient algorithm for finding the set of maximal estimates $\globfromopt$}.

Our statement in Section~\ref{sec:updatingtheihmm}, that any coherent updating method would yield the same results as regular extension, can now be justified. 
Since coherent updating is unique if $\jointlpr{}(\singfromout{1})>0$, we only need to motivate our statement in the special case that $\jointlpr{}(\singfromout{1})=0$ and $\jointupr{}(\singfromout{1})>0$. 

If we use regular extension to update our model, the optimal estimates are given by Eq.~\eqref{eq:globoptimals}. 
For the special case $\jointlpr{}(\singfromout{1})=0$ however, we find for all $\fromxstate{1}\in\fromstates{1}$ and $\fromxhatstate{1}\in\fromstates{1}$ that
\begin{equation*}
\jointlpr{}(\indfromout{1}[\indfromxstate{1}-\indfromxhatstate{1}])\leq\jointlpr{}(\indfromout{1})=\jointlpr{}(\singfromout{1})=0,
\end{equation*}
where the first inequality follows from the monotonicity of coherent lower previsions (as a consequence of \ref{C1} and \ref{C2}). 
Therefore, we find that if $\jointlpr{}(\singfromout{1})=0$, all sequences are optimal, resulting in $\globfromopt=\fromstates{1}$.

If we use natural extension to update our joint model, the optimal state sequences are still given by Eq.~\eqref{eq:globoptimals}, but the final equivalence would no longer hold because it uses Theorem~\ref{theorem:crucial}, which assumes the use of regular extension to perform updating of the joint model. 
However, for the special case of $\jointlpr{}(\singfromout{1})=0$, natural extension by definition leads to the updated model being equal to the vacuous one. 
Therefore, we find for all $\fromxstate{1}\in\fromstates{1}$ and $\fromxhatstate{1}\in\fromstates{1}$ that
\begin{equation*}
  \jointlpr{}(\indfromxstate{1}-\indfromxhatstate{1}\vert\fromout{1})
  =\min(\indfromxstate{1}-\indfromxhatstate{1})
  \leq0.
\end{equation*}
This implies that for the special case of $\jointlpr{}(\singfromout{1})=0$ and $\jointupr{}(\singfromout{1})>0$---identical to what we found for regular extension---natural extension also results in all sequences being optimal, meaning that $\globfromopt=\fromstates{1}$.

We have thus shown that, in the special case when $\jointlpr{}(\singfromout{1})=0$ and $\jointupr{}(\singfromout{1})>0$, the set of optimal sequences is the same, regardless of whether we use natural or regular extension to update our joint model. 
Since every other coherent updating method lies in between those two methods, $\globfromopt$ does not depend on the updating method, as long as it is coherent. 
If $\jointlpr{}(\singfromout{1})>0$, coherent updating is unique and thus equal to regular extension, thereby making this result trivial in that case. 
We can therefore conclude that the results in this paper do not depend on the particular updating method that is chosen, as long as it is coherent.

Instead of looking for the maximal state sequences, one could also use other decision criteria. 
A first approach that we will not consider here, could consist in trying to find the so-called $\Gamma$-\emph{maximin} state sequences $\fromxbarstate{1}$, which maximise the posterior lower probability:
\begin{equation*}
  \fromxbarstate{1}
  \in\argmax_{\fromxstate{1}\in\fromstates{1}}\jointlpr{}(\singfromxstate{1}\vert\fromout{1})
\end{equation*}
While it is well known that any such $\Gamma$-maximin sequence is in particular guaranteed to also be a maximal sequence, finding such $\Gamma$-maximin sequences seems to be a much more complicated affair.\footnote{Private communication from Cassio de Campos.} 
Of course, once we know all maximal solutions, we could determine which of them are the $\Gamma$-maximin solutions by comparing their posterior lower probabilities. 
As far as we can see, however, calculating these seems no trivial task from a computational point of view.

We expect similar computational difficulties with yet another approach, also not considered here, which consists in finding the so-called \emph{E-admissable} sequences. 
They are those sequences that maximise the expected gain for at least one conditional mass function $p(\cdot\vert\fromout{1})$ in the credal set associated with the updated lower prevision $\jointlpr{}(\cdot\vert\fromout{1})$. 
Similarly to the $\Gamma$-maximin solutions, the E-admissable ones are also known to be contained within the set of maximal ones that we will be constructing.

The main reason why our approach is so efficient compared to the other ones, is that we do not have to explicitly calculate the value of lower previsions, but only need to know their sign, thereby allowing us to work directly with the joint model, instead of the updated model.

\subsection{Maximal subsequences}
We shall see below that in order to find the set of maximal estimates, it is useful to consider more general sets of so-called maximal subsequences: for any $k\in\{1,\dots,n\}$ and $\zstate{k-1}\in\states{k-1}$, we define $\fromopt[\fromstates{k}]{z}{k}$:
\begin{equation}\label{eq:optimals}
  \fromxhatstate{k}\in\fromopt[\fromstates{k}]{z}{k}
  \asa(\forall\fromxstate{k}\in \fromstates{k})~
  \zinjointclpr[{\indfromout{k}[\indfromxstate{k}-\indfromxhatstate{k}]}]{k}{k-1}\leq0.
\end{equation}
The interpretation of these sets is immediate: consider the following part of the original iHMM, where we take $\zinstateclpr{k}{k-1}$ as the marginal model for the first state $\statevar{k}$:
\begin{center}
  \begin{tikzpicture}
    \tikzset{node distance=\afstandknopen, auto}
    % knopen tekenen:
    \node[knoopstate] (1)  at (-1,1) {$\statevar{k}$};
    \node[knoopstate] (3)  at ($(1)+(1.5*\afstandknopen , 0)$) {$\statevar{r}$};
    \node[knoopstate] (4) at ($(3)+(1.5*\afstandknopen, 0)$)  {$\statevar{n}$};
    \node[knoopobservatie] (5) [below of = 1]  {$\outvar{k}$};
    \node[knoopobservatie] (7) [below of = 3]  {$\outvar{r}$};
    \node[knoopobservatie] (8) [below of = 4] {$\outvar{n}$};
    % pijlen tekenen:
    % pijlen van toestand naar toestand:
    \draw[pijl4] (1) to node[] {} ($(1)+(\afstandvollelijn, 0)$);
    \draw[pijl1] ($(3)+(-\afstandvollelijn,0)$) to node[] {} (3);
    \draw[pijl3] ($(1)+(\afstandvollelijn, 0)$) to node[] {} ($(3)+(-\afstandvollelijn,0)$) ;
    \draw[pijl4] (3) to node[] {} ($(3)+(\afstandvollelijn, 0)$);
    \draw[pijl1] ($(4)+(-\afstandvollelijn,0)$) to node[] {} (4);
    \draw[pijl3] ($(3)+(\afstandvollelijn, 0)$) to node[] {} ($(4)+(-\afstandvollelijn,0)$) ;
    \draw[pijl1] (1) -- (5) ;
    \draw[pijl1] (3) -- (7) ;
    \draw[pijl1] (4) -- (8) ;
    % modellen erbij zetten
    \draw (1) node [above=15pt]{$\zinstateclpr{k}{k-1}$};
    \draw (3) node [above=15pt]{$\stateclpr{r}{r-1}$};
    \draw (4) node [above=15pt]{$\stateclpr{n}{n-1}$};
    \draw (5) node [below=17pt]{$\outclpr{k}$};
    \draw (7) node [below=17pt]{$\outclpr{r}$};
    \draw (8) node [below=17pt]{$\outclpr{n}$};
    \draw (1) node [left=20pt]{State subsequence:};
    \draw (5) node [left=20pt]{Output subsequence:};
  \end{tikzpicture}
  %\caption{Tree representation of a part of the original HMM}
  %\label{fig:subhmmtree}
\end{center}
Then, as we have argued in Section~\ref{sec:jointmodel}, the corresponding joint lower prevision on $\fromstateoutgambles{k}$ is precisely $\zinjointclpr{k}{k-1}$, and if we have a sequence of outputs $\fromout{k}$, then $\fromopt[\fromstates{k}]{z}{k}$ is the set of state sequence estimates that are undominated by any other estimate in $\fromstates{k}$.  
It should be clear that the set $\globfromopt$ we are eventually looking for, can also be written as $\optim\left(\fromstates{1}\vert\zstate{0},\fromout{1}\right)$.
%Moet er hier expliciet verantwoord worden dat
%Theorem~\ref{theorem:crucial} hier (mutatis mutandis) ook geldig is en dat
%daardoor die interpretatie wel degelijk correct is? ***

\subsection{Useful recursion equations}
Fix any $k$ in $\{1,\dots,n\}$.  If we look at Equation~\eqref{eq:optimals}, we see that it will be useful to derive a manageable expression for the lower prevision $\zinjointclpr[{\indfromout{k}[\indfromxstate{k}-\indfromxhatstate{k}]}]{k}{k-1}$.
This can be easily done (see Appendix~\ref{appendix}) using Equations~\eqref{eq:indrecurse}--\eqref{eq:upper:state:out:mass} together with a few algebraic manipulations. We consider three different cases.
If $\xhatstate{k}=\xstate{k}$ and $k\in\{1,\dots,n-1\}$ then, using the notation introduced in Section~\ref{sec:jointmodel}:
\begin{multline}
  \label{eq:target:equal}
  \zinjointclpr[\indfromout{k}[\indfromxstate{k}-\indfromxhatstate{k}]{k}{k-1}\\=
    \zinstateclupr[\singxhatstate{k}]{k}{k-1}\xhatinoutclupr[\singout{k}]{k}
    \odot\xhatinjointclpr[\indfromout{k+1}[\indfromxstate{k+1}-\indfromxhatstate{k+1}]{k+1}{k}.
\end{multline}
If $\xhatstate{n}=\xstate{n}$ then
\begin{equation}\label{eq:target:equal:final}
  \zinjointclpr[\indout{n}[\indxstate{n}-\indxhatstate{n}]{n}{n-1}=0.
\end{equation}
If $\xhatstate{k}\neq\xstate{k}$ and $k\in\{1,\dots,n\}$ then
\begin{equation}
  \label{eq:target:different}
  \zinjointclpr[\indfromout{k}[\indfromxstate{k}-\indfromxhatstate{k}]{k}{k-1}
  =\zinstateclpr[{\indxstate{k}\fromxlmem{k}-\indxhatstate{k}\fromxhatumem{k}}]{k}{k-1},
\end{equation}
where we define, for any $\fromzstate{k}\in\fromstates{k}$:
\begin{align}
  \fromzlmem{k}
  &\coloneqq\zinindclpr[{\indfromout{k}\indfromzstate{k+1}}]{k}=\zinoutclpr[\singout{k}]{k}
  \prod_{i=k+1}^n\zinoutclpr[\singout{i}]{i}\zinstateclpr[\singzstate{i}]{i}{i-1}\label{eq:beta}\\
  \fromzumem{k}
  &\coloneqq\zinindcupr[{\indfromout{k}\indfromzstate{k+1}}]{k}=\zinoutcupr[\singout{k}]{k}
  \prod_{i=k+1}^n\zinoutcupr[\singout{i}]{i}\zinstatecupr[\singzstate{i}]{i}{i-1}.\label{eq:alpha}
\end{align}
For any given sequence of states $\fromzstate{k}\in\fromstates{k}$, the $\fromzumem{k}$ and $\fromzlmem{k}$ can be found by simple backward recursion:
\begin{align}
  \fromzumem{k}
  &\coloneqq\fromzumem{k+1}\zinoutcupr[\singout{k}]{k}\zinstatecupr[\singzstate{k+1}]{k+1}{k}
  \label{eq:alpharecurs}\\
  \fromzlmem{k}
  &\coloneqq\fromzlmem{k+1}\zinoutclpr[\singout{k}]{k}\zinstateclpr[\singzstate{k+1}]{k+1}{k},
  \label{eq:betarecurs}
\end{align}
for $k\in\{1,\dots,n-1\}$, and starting from:
\begin{equation*}
  \fromzumem{n}=\zumem{n}\coloneqq\zinoutcupr[\singout{n}]{n}
  \text{ and }
  \fromzlmem{n}=\zlmem{n}\coloneqq\zinoutclpr[\singout{n}]{n}.
\end{equation*}

\section{The Principle of Optimality}
\label{sec:principle-of-optimality}
Determining the state sequences in $\globfromopt$ directly using Equation~\eqref{eq:globoptimals} clearly has exponential complexity (in the length of the chain).
We are now going to take a dynamic programming approach
\cite{bellman1957} to reducing this complexity by deriving a recursion
equation for the sets of optimal (sub)sequences $\fromopt[\fromstates{k}]{z}{k}$.

\begin{theorem}[Principle of Optimality]\label{theorem:optimality}
  For $k\in\{1,\dots,n-1\}$, all $\zstate{k-1}\in\states{k-1}$ and all $\fromxhatstate{k}\in\fromstates{k}$: if   $\zinstateclpr[\singxhatstate{k}]{k}{k-1}>0$ and $\xhatinoutclpr[\singout{k}]{k}>0$, then
  \begin{equation*}
    \fromxhatstate{k}\in\fromopt[\fromstates{k}]{z}{k}
    \dan\fromxhatstate{k+1}\in\explfromopt[\fromstates{k+1}]{\hat{x}}{k}{k+1}.
  \end{equation*}
\end{theorem}

As an immediate consequence, we find that
\begin{equation}\label{eq:pop1imprecies}
  \fromopt[\fromstates{k}]{z}{k}
  \subseteq\frommog[\fromstates{k}]{z}{k},
\end{equation}
with $\frommog[\fromstates{k}]{z}{k}$ being the set of sequences in $\fromstates{k}$ that can still be an element of $\fromopt[\fromstates{k}]{z}{k}$ according the the theorem above:
\begin{multline}\label{eq:pop2imprecies}
  \frommog[\fromstates{k}]{z}{k}\\
  \coloneqq
  \bigg(
  \bigcup_{\zstate{k}\in \pos_k(\zstate{k-1})}\zstate{k}\oplus\explfromopt[\fromstates{k+1}]{z}{k}{k+1}
  \bigg)
  \cup
  \bigg(
  \bigcup_{\zstate{k}\notin \pos_k(\zstate{k-1})}\zstate{k}\oplus\states{k+1:n}
  \bigg).
\end{multline}
Here $\oplus$ denotes concatenation of state sequences and the set of states $\pos_k(\zstate{k-1})\subseteq\states{k}$ is defined as
\begin{equation}\label{eq:pop3imprecies}
  \zstate{k}\in \pos_k(\zstate{k-1})
  \asa
  \zinstateclpr[\singzstate{k}]{k}{k-1}>0\text{ en }\zinoutclpr[\singout{k}]{k}>0.
\end{equation}
Equation~\eqref{eq:pop2imprecies} simplifies to
\begin{equation}\label{eq:pop4imprecies}
  \frommog[\fromstates{k}]{z}{k}
  =\bigcup_{\zstate{k}\in\states{k}}\zstate{k}\oplus\explfromopt[\fromstates{k+1}]{z}{k}{k+1}
\end{equation}
if all local lower previsions are positive, but this is not generally true in the more general case we are considering here, where only the upper previsions are required to be positive.

We also introduce the following notation:
\begin{equation}\label{def:doorsnede}
  \frommogDoor[\fromstates{k}]{z}{k}{\xstate{k:s}}
  \coloneqq\set{\fromzstate{k}\in\frommog[\fromstates{k}]{z}{k}}{\zstate{k:s}=\xstate{k:s}}
\end{equation}
for all $k\in\{1,\dots,n\}$, $s\in\{k,\dots,n\}$, $\zstate{k-1}\in\states{k-1}$, $\xstate{k:s}\in\states{k:s}$ and $\fromout{k}\in\fromouts{k}$.

\section{An algorithm for finding maximal state sequences}
\label{sec:algorithm}
We now use Equation~\eqref{eq:pop1imprecies} to devise an algorithm for constructing the set $\globfromopt$ of maximal state sequences in a recursive manner.

\subsection{Initial set-up using backward recursion}
We begin by defining a few auxiliary notions.
First of all, we consider the thresholds:
\begin{equation}\label{eq:threshold}
  \zinthreshold{k}
  \coloneqq\min\set{a\geq 0}{\zinstateclpr[\indxstate{k}-a\indxhatstate{k}]{k}{k-1}\leq0}
\end{equation}
for all $k\in\{1,\dots,n\}$, $\zstate{k-1}\in\states{k-1}$ and $\xstate{k},\xhatstate{k}\in\states{k}$.
\par
Next, we define
\begin{equation}\label{eq:alphabetamax}
  \xumemmax{k}
  \coloneqq\max_{\substack{\fromzstate{k}\in\fromstates{k}\\\zstate{k}=\xstate{k}}}\fromzumem{k}
  \text{ and }
  \xlmemmax{k}
  \coloneqq\max_{\substack{\fromzstate{k}\in\fromstates{k}\\\zstate{k}=\xstate{k}}}\fromzlmem{k}
\end{equation}
for all $k\in\{1,\dots,n\}$ and $\xstate{k}\in\states{k}$. 
Using Equations~\eqref{eq:alpharecurs}--\eqref{eq:betarecurs}, these can be calculated efficiently using the following backward recursive (dynamic programming) procedure:
\vspace{2mm}
\begin{align}\label{eq:alphamaxrecurs}
  \xumemmax{k}
    &=\max_{\zstate{k+1}\in\states{k+1}}\zumemmax{k+1}
    \xinoutcupr[\singout{k}]{k}\xinstatecupr[\singzstate{k+1}]{k+1}{k}\notag\\
    &=\xinoutcupr[\singout{k}]{k}
    \max_{\zstate{k+1}\in\states{k+1}}\zumemmax{k+1}\xinstatecupr[\singzstate{k+1}]{k+1}{k},
\end{align}
and
\begin{align}\label{eq:betamaxrecurs}
  \xlmemmax{k}
    &=\max_{\zstate{k+1}\in\states{k+1}}\zlmemmax{k+1}
    \xinoutclpr[\singout{k}]{k}\xinstateclpr[\singzstate{k+1}]{k+1}{k}\notag\\
    &=\xinoutclpr[\singout{k}]{k}
    \max_{\zstate{k+1}\in\states{k+1}}\zlmemmax{k+1}\xinstateclpr[\singzstate{k+1}]{k+1}{k},
\end{align}
for $k\in\{1,\dots,n-1\}$, starting from
\vspace{2mm}
\begin{equation}\label{eq:alphabetamaxn}
  \xumemmax{n}=\xumem{n}=\xinoutcupr[\singout{n}]{n}
  \text{ and }
  \xlmemmax{n}=\xlmem{n}=\xinoutclpr[\singout{n}]{n}.
\end{equation}
Finally, we let
\begin{equation}\label{eq:alphaopt}
  \umemopt{k}
  \coloneqq\max_{\substack{\xstate{k}\in\states{k}\\\xstate{k}\neq\xhatstate{k}}}\xlmemmax{k}\zinthreshold{k},
\end{equation}
for all $k\in\{1,\dots,n\}$, $\zstate{k-1}\in\states{k-1}$ and $\xhatstate{k}\in\states{k}$.

\subsection{Reformulation of the optimality condition}
It turns out that the $\umemopt{k}$, calculated by Equation~\eqref{eq:alphaopt}, are extremely useful. 
As proved in Appendix~\ref{appendix}, they allow us to significantly simplify Equation~\eqref{eq:optimals} as follows:
\begin{equation}\label{eq:criterion-at-k}
  \fromopt[\fromstates{k}]{z}{k}
  =\set{\fromxhatstate{k}\in\frommog[\fromstates{k}]{z}{k}}{\fromxhatumem{k}\geq\umemopt{k}},
\end{equation}
which, for $k=n$, reduces to
\begin{equation}\label{eq:criterion-at-n}
  \opt[\states{n}]{z}{n}
  =\set{\xhatstate{n}\in\states{n}}{\xhatumem{n}\geq\umemopt{n}}.
\end{equation}

\subsection{A recursive solution method}
The aim of the algorithm is to determine the set $\globfromopt$ efficiently. 
We will do so recursively.

For $k=n$, $\opt[\states{n}]{z}{n}$ can be determined in a straightforward manner for every $\zstate{n-1}\in\states{n-1}$ using Criterion~\eqref{eq:criterion-at-n}.

\begin{example}\label{exam:running}
  We consider a simple binary HMM with $\states{}=\{0,1\}$.  
  For $k=n$, the maximal elements are simply states, which are trivially represented. 
  We could for example find that $\explopt[\states{n}]{0}{n}=\{0,1\}$ for $\zstate{n-1}=0$, and $\explopt[\states{n}]{1}{n}=\{0\}$ for $\zstate{n-1}=1$.  
  \hfill $\blacklozenge$
\end{example}

Next, we let $k$ run \emph{backward} from $n-1$ to $1$. 
For each $k<n$ and all $\zstate{k-1}\in\states{k-1}$, we first build up the set $\frommog[\fromstates{k}]{z}{k}$, using its definition in Equation~\eqref{eq:pop2imprecies} and the results of the previous recursion step. 
This set is then used to determine $\fromopt[\states{k}]{z}{k}$ with Criterion~\eqref{eq:criterion-at-k}.

\begin{example} 
  We continue the discussion of Example~\ref{exam:running}.
  For $k=n-1$ and $\zstate{n-2}=0$, the set $\explfrommog[\fromstates{n-1}]{0}{}{n-1}$ is constructed using Equation~\eqref{eq:pop2imprecies}.
  If, for instance $\pos_{n-1}(0)=\{0,1\}$, this reduces to Equation~\eqref{eq:pop4imprecies} and we find that
  \begin{align*}
    \explfrommog[\fromstates{n-1}]{0}{}{n-1}
    &=\bigcup_{\zstate{n-1}\in\{0,1\}}\zstate{n-1}\oplus\opt[\states{n}]{z}{n}\\
    &=0\oplus \{0,1\} \cup 1\oplus \{0\}
    =\{00,01\}\cup\{10\}
    =\{00,01,10\}.
  \end{align*}
  Applying Criterion~\eqref{eq:criterion-at-k} to every element of this set, we find the set $\explfromopt[\fromstates{n-1}]{0}{}{n-1}$, which for instance could be equal to $\{00,10\}$. 
  For $\zstate{n-2}=0$, an analoguous method can be used.
  \hfill
  $\blacklozenge$
\end{example}

Continuing in this way, we eventually reach $k=1$, which yields the desired set of maximal sequences $\globfromopt=\optim\left(\fromstates{1}\vert\zstate{0},\fromout{1}\right)$.

The possible bottleneck in this solution lies in the use of Criterion~\eqref{eq:criterion-at-k}. 
While this criterion is already much more efficient than the original
one, it can still lead to an exponential complexity if the set
$\frommog[\fromstates{k}]{z}{k}$ has a number of elements that is
exponential in the length of the considered sequences. We therefore present a method that avoids checking the inequality in Criterion~\eqref{eq:criterion-at-k} for all elements of $\frommog[\fromstates{k}]{z}{k}$.

The first trick consists in using an efficient data structure to store the sets of optimal sequences. 
For $k=n$, this is simply a list of the elements. 
For $k<n$, we could also just list the optimal sequences, but this would imply storing the same information multiple times, since parts of those sequences will be the same. 
We therefore choose to represent this list of optimal sequences as a collection of tree structures.
The way these trees are constructed should be obvious from the following example.

\begin{example}\label{exam:structuur}
  Consider the following set of sequences:
  \begin{equation*}
    \{00001000, 00001010, 00001110, 00011110, 10001010, 10001110\}
  \end{equation*}
  By representing this set in this way, useful information gets lost and memory space is waisted.
  For example, some of these sequences all start out the same way.
  It would be much more efficient to store such common subsequences only once.
  \begin{center}
    \begin{tikzpicture}[scale=.9]
      \def\boomhor{12mm};
      \def\boomver{9mm};
      \def\boomverhoofd{60mm};
      % nodes for first tree
      \node (nul) at (10.6cm,0) {};
      \node (b0) at (0,0) {};
      \node[niets] (b00) at ($(b0)+(\boomhor,0.5*\boomver)$) {0};
      \node[niets] (b000) at ($(b00)+(\boomhor,0)$) {0};
      \node[niets] (b0000) at ($(b000)+(\boomhor,0)$) {0};
      \node[niets] (b00001) at ($(b0000)+(\boomhor,0)$) {1};
      \node[niets] (b000011) at ($(b00001)+(\boomhor,0)$) {1};
      \node[niets] (b0000111) at ($(b000011)+(\boomhor,0)$) {1};
      \node[niets] (b00001111) at ($(b0000111)+(\boomhor,0)$) {1};
      \node[niets] (b000011110) at ($(b00001111)+(\boomhor,0)$) {0};
      % \draw[boomlijngroen] (b0) -- (b00);
      \draw[boomlijngroen] (b00) -- (b000);
      \draw[boomlijngroen] (b000) -- (b0000);
      \draw[boomlijngroen] (b0000) -- (b00001);
      \draw[boomlijngroen] (b00001) -- (b000011);
      \draw[boomlijngroen] (b000011) -- (b0000111);
      \draw[boomlijngroen] (b0000111) -- (b00001111);
      \draw[boomlijngroen] (b00001111) -- (b000011110);
      % nodes for second tree
      \node[niets] (b00000) at ($(b00001)+(0,\boomver)$) {0};
      \node[niets] (b000001) at ($(b00000)+(\boomhor,0)$) {1};
      \node[niets] (b0000011) at ($(b000001)+(\boomhor,0)$) {1};
      \node[niets] (b00000111) at ($(b0000011)+(\boomhor,0)$) {1};
      \node[niets] (b000001110) at ($(b00000111)+(\boomhor,0)$) {0};
      % lines
      \draw[boomlijngroen] (b0000) -- (b00000);
      \draw[boomlijngroen] (b00000) -- (b000001);
      \draw[boomlijngroen] (b000001) -- (b0000011);
      \draw[boomlijngroen] (b0000011) -- (b00000111);
      \draw[boomlijngroen] (b00000111) -- (b000001110);
      \node[niets] (b0000010) at ($(b0000011)+(0,\boomver)$) {0};
      \node[niets] (b00000101) at ($(b0000010)+(\boomhor,0)$) {1};
      \node[niets] (b000001010) at ($(b00000101)+(\boomhor,0)$) {0};
      \draw[boomlijngroen] (b000001) -- (b0000010);
      \draw[boomlijngroen] (b0000010) -- (b00000101);
      \draw[boomlijngroen] (b00000101) -- (b000001010);
      \node[niets] (b00000100) at ($(b00000101)+(0,\boomver)$) {0};
      \node[niets] (b000001000) at ($(b00000100)+(\boomhor,0)$) {0};
      \draw[boomlijngroen] (b0000010) -- (b00000100);
      \draw[boomlijngroen] (b00000100) -- (b000001000);
      \node[niets] (b01) at ($(b0)+(\boomhor,-0.5*\boomver)$) {1};
      \node[niets] (b010) at ($(b01)+(\boomhor,0)$) {0};
      \node[niets] (b0100) at ($(b010)+(\boomhor,0)$) {0};
      \node[niets] (b01000) at ($(b0100)+(\boomhor,0)$) {0};
      \node[niets] (b010001) at ($(b01000)+(\boomhor,0)$) {1};
      \node[niets] (b0100010) at ($(b010001)+(\boomhor,0)$) {0};
      \node[niets] (b01000101) at ($(b0100010)+(\boomhor,0)$) {1};
      \node[niets] (b010001010) at ($(b01000101)+(\boomhor,0)$) {0};
      % \draw[boomlijngroen] (b0) -- (b01);
      \draw[boomlijngroen] (b01) -- (b010);
      \draw[boomlijngroen] (b010) -- (b0100);
      \draw[boomlijngroen] (b0100) -- (b01000);
      \draw[boomlijngroen] (b01000) -- (b010001);
      \draw[boomlijngroen] (b010001) -- (b0100010);
      \draw[boomlijngroen] (b0100010) -- (b01000101);
      \draw[boomlijngroen] (b01000101) -- (b010001010);
      \node[niets] (b0100011) at ($(b0100010)+(0,-\boomver)$) {1};
      \node[niets] (b01000111) at ($(b0100011)+(\boomhor,0)$) {1};
      \node[niets] (b010001110) at ($(b01000111)+(\boomhor,0)$) {0};
      \draw[boomlijn] (b010001) -- (b0100011);
      \draw[boomlijn] (b0100011) -- (b01000111);
      \draw[boomlijn] (b01000111) -- (b010001110);
    \end{tikzpicture}
    % \caption{Tree representation of a hidden Markov model}
    % \label{fig:treerepresentationexample}
  \end{center}
  We therefore prefer to represent the above set as the collection of trees depicted above.
  \hfill
  $\blacklozenge$
\end{example}

The next step is now to exploit this data structure in order to apply Criterion~\eqref{eq:criterion-at-k} efficiently. 
We start by constructing the set $\frommog[\fromstates{k}]{z}{k}$ and representing it in the same type of data structure.

\begin{example} \label{exam:running2}
  We consider the set of sequences in Example~\ref{exam:structuur} to be $\explfromopt[\fromstates{k+1}]{\hspace{1mm}0}{}{k+1}$, where $k=n-8$, since the length of the sequences is $8$. 
  Suppose we have already constructed this set in the previous recursion step. 
  Furthermore, for the sake of this example, lets assume that $0\in\pos_{k-1}(0)$ and $1\notin\pos_{k-1}(0)$. 
  We will now use Equation~\eqref{eq:pop2imprecies} to construct the set $\explfrommog[\fromstates{k}]{0}{}{k}$: 
  \begin{equation*}
    \explfrommog[\fromstates{k}]{0}{}{k}
    =0\oplus\explfromopt[\fromstates{k+1}]{0}{}{k+1}
    \cup1\oplus\states{k+1:n}.
  \end{equation*}
  The set $\explfrommog[\fromstates{k}]{0}{}{k}$ consist of two subsets, which we will construct separately. 
  The subset $0\oplus\explfromopt[\fromstates{k+1}]{0}{}{k+1}$ would normally take quite some effort to compose, since we have to concatenate $0$ with each individual element of $\explfromopt[\fromstates{k+1}]{0}{}{k+1}$. 
  However, using our representation, this comes down to adding one node and two links to the already existing data structure for $\explfromopt[\fromstates{k+1}]{0}{}{k+1}$:
  \begin{center}
    \begin{tikzpicture}[scale=.9]
      \def\boomhor{12mm};
      \def\boomver{9mm};
      \def\boomverhoofd{20mm};
      \def\eenbegin{12mm};
      \def\eeneind{41mm};
      \def\eenrand{6mm};
      
      \path[kplus] ($(\eenbegin,0)+(0,11mm)+(-\eenrand,\eenrand)$)
      rectangle ($(\eeneind,0) +(4mm,-5mm)+(\eenrand,-\eenrand)$);
      \path[kplus] ($(\eenbegin,0)+(0,11mm)+(-\eenrand,\eenrand)+(-50mm,4mm)$) rectangle ($(\eeneind,0) +(0,-5mm)+(\eenrand,-\eenrand)+(8mm,-4mm)$);
      
      \node[niets] (b0) at (0,0) {0};
      \node[niets] (b00) at ($(b0)+(\boomhor,0.5*\boomver)$) {0};
      \node[niets] (b000) at ($(b00)+(\boomhor,0)$) {0};
      \node (b000-) at ($(b000)+(\boomhor,0)$) {};
      
      \draw[boomlijngroen] (b0) -- (b00);
      \draw[boomlijngroen] (b00) -- (b000);
      \draw[boomstippellijn] (b000) -- (b000-);
      
      \node[niets] (b01) at ($(b0)+(\boomhor,-0.5*\boomver)$) {1};
      \node[niets] (b010) at ($(b01)+(\boomhor,0)$) {0};
      \node (b010-) at ($(b010)+(\boomhor,0)$) {};
      
      \draw[boomlijngroen] (b0) -- (b01);
      \draw[boomlijngroen] (b01) -- (b010);
      \draw[boomstippellijn] (b010) -- (b010-);
      
      \node (kplusnul) at ($(b0)+(\eenbegin+15mm,12mm)$) {$\explfromopt[\fromstates{k+1}]{\hspace{1mm}0}{}{k+1}$};
      \node (mog) at ($(b0)+(\eenbegin-30mm,12mm)$) {$0\oplus\explfromopt[\fromstates{k+1}]{0}{}{k+1}$};
    \end{tikzpicture}
    % \caption{Tree representation of a hidden Markov model}
    % \label{fig:constructmogdeel1}
  \end{center}
  Conceptually, we want to represent the set $1\oplus\states{k+1:n}$ as a tree, which would look like the figure below on the left.
  \begin{center}
    \begin{tikzpicture}[scale=.9]
      \def\boomhor{12mm};
      \def\boomver{9mm};
      \def\boomverhoofd{20mm};
      
      \def\eenbegin{12mm};
      \def\eeneind{33mm};
      \def\eenrand{6mm};
      
      \path[kplus] ($(\eenbegin,0)+(0,11mm)+(-\eenrand,\eenrand)$)
      rectangle ($(\eeneind,0) +(0,-23mm)+(\eenrand,-\eenrand)$);
      \path[kplus] ($(\eenbegin,0)+(0,11mm)+(-\eenrand,\eenrand)+(-25mm,4mm)$) rectangle ($(\eeneind,0) +(0,-23mm)+(\eenrand,-\eenrand)+(4mm,-4mm)$);
      
      % \node (nul) at (5cm,1cm) {};
      \node[niets] (b1) at (0,-9mm) {1};
      \node[niets] (b10) at ($(b1)+(\boomhor,\boomver)$) {0};
      \node[niets] (b100) at ($(b10)+(\boomhor,0.5*\boomver)$) {0};
      \node (b1000) at ($(b100)+(\boomhor,0.25*\boomver)$) {};
      \node (b1001) at ($(b100)+(\boomhor,-0.25*\boomver)$) {};
      
      \draw[boomlijngroen] (b1) -- (b10);
      \draw[boomlijngroen] (b10) -- (b100);
      \draw[boomstippellijn] (b100) -- (b1000);
      \draw[boomstippellijn] (b100) -- (b1001);
      
      \node[niets] (b101) at ($(b10)+(\boomhor,-0.5*\boomver)$) {1};
      \node (b1010) at ($(b101)+(\boomhor,0.25*\boomver)$) {};
      \node (b1011) at ($(b101)+(\boomhor,-0.25*\boomver)$) {};
      
      \draw[boomlijngroen] (b10) -- (b101);
      \draw[boomstippellijn] (b101) -- (b1010);
      \draw[boomstippellijn] (b101) -- (b1011);
      
      \node[niets] (b11) at ($(b1)+(\boomhor,-\boomver)$) {1};
      \node[niets] (b110) at ($(b11)+(\boomhor,0.5*\boomver)$) {0};
      \node (b1100) at ($(b110)+(\boomhor,0.25*\boomver)$) {};
      \node (b1101) at ($(b110)+(\boomhor,-0.25*\boomver)$) {};
      
      \draw[boomlijngroen] (b1) -- (b11);
      \draw[boomlijngroen] (b11) -- (b110);
      \draw[boomstippellijn] (b110) -- (b1100);
      \draw[boomstippellijn] (b110) -- (b1101);
      
      \node[niets] (b111) at ($(b11)+(\boomhor,-0.5*\boomver)$) {1};
      \node (b1110) at ($(b111)+(\boomhor,0.25*\boomver)$) {};
      \node (b1111) at ($(b111)+(\boomhor,-0.25*\boomver)$) {};
      
      \draw[boomlijngroen] (b11) -- (b111);
      \draw[boomstippellijn] (b111) -- (b1110);
      \draw[boomstippellijn] (b111) -- (b1111);
      
      \node (kplusnul) at ($(b0)+(\eenbegin+3mm,12mm)$) {$\states{k+1:n}$};
      \node (mog) at ($(b0)+(\eenbegin-18mm,12mm)$) {$1\oplus\states{k+1:n}$};
    \end{tikzpicture}
    % \caption{Tree representation of a hidden Markov model}
    % \label{fig:constructmogdeel2}
    \quad
    \begin{tikzpicture}[scale=.9]
      \def\boomhor{12mm};
      \def\boomver{9mm};
      \def\boomverhoofd{23mm};
      
      \def\eenbegin{12mm};
      \def\eeneind{20mm};
      \def\eenrand{6mm};
      
      \path[kplus] ($((\eenbegin,0)+(0,1mm)+(-\eenrand,\eenrand)$) rectangle
      ($(0,-1mm)+(\eeneind,0)+(\eenrand,-\eenrand)$);
      \path[kplus] ($(\eenbegin,0)+(0,1mm)+(-\eenrand,\eenrand)+(-40mm,4mm)$) rectangle ($(0,-1mm)+(\eeneind,0) +(\eenrand,-\eenrand)+(4mm,-4mm)$);
      
      \node[niets] (b1) at ($(b0)+(0,0)$) {1};
      \node (b1-) at ($(b1)+(7.1mm,0)$) {};
      
      \draw[boomlijngroen] (b1) -- (b1-);
      
      \node (kpluseen) at ($(b1)+(\eenbegin+3mm,0)$) {$\fromstates{k+1}$};
      \node (mog) at ($(b1)+(\eenbegin-29mm,0)$) {$1\oplus\states{k+1:n}$};
    \end{tikzpicture}
  \end{center}
  However, storing it this way in a computer is a bad idea, as this would mean constructing a complete binary tree, which is exponential in the depth of this tree. 
  We therefore remember that the set of sequences can be represented as a tree, without actually constructing it, as is depicted above on the right.
  \begin{center}
    \begin{tikzpicture}[scale=.9]
      \def\boomhor{12mm};
      \def\boomver{9mm};
      \def\boomverhoofd{23mm};
      
      \def\eenbegin{12mm};
      \def\eeneind{41mm};
      \def\eenrand{6mm};

      \path[kplus] ($(0,-\boomverhoofd)+(\eenbegin,0)+(0,1mm)+(-\eenrand,\eenrand)$) rectangle
      ($(0,-\boomverhoofd)+(2mm,-1mm)+(\eeneind,0)+(\eenrand,-\eenrand)$);
      \path[kplus] ($(\eenbegin,0)+(0,11mm)+(-\eenrand,\eenrand)$)
      rectangle ($(\eeneind,0) +(2mm,-5mm)+(\eenrand,-\eenrand)$);
      \path[kplus] ($(\eenbegin,0)+(0,11mm)+(-\eenrand,\eenrand)+(-40mm,4mm)$) rectangle ($(0,-\boomverhoofd)+(2mm,-1mm)+(\eeneind,0) +(\eenrand,-\eenrand)+(4mm,-4mm)$);
      
      \node (nul) at (5cm,1cm) {};
      \node[niets] (b0) at (0,0) {0};
      \node[niets] (b00) at ($(b0)+(\boomhor,0.5*\boomver)$) {0};
      \node[niets] (b000) at ($(b00)+(\boomhor,0)$) {0};
      \node (b000-) at ($(b000)+(\boomhor,0)$) {};
      
      \draw[boomlijngroen] (b0) -- (b00);
      \draw[boomlijngroen] (b00) -- (b000);
      \draw[boomstippellijn] (b000) -- (b000-);
      
      \node[niets] (b01) at ($(b0)+(\boomhor,-0.5*\boomver)$) {1};
      \node[niets] (b010) at ($(b01)+(\boomhor,0)$) {0};
      \node (b010-) at ($(b010)+(\boomhor,0)$) {};
      
      \draw[boomlijngroen] (b0) -- (b01);
      \draw[boomlijngroen] (b01) -- (b010);
      \draw[boomstippellijn] (b010) -- (b010-);
      
      \node[niets] (b1) at ($(b0)+(0,-\boomverhoofd)$) {1};
      \node (b1-) at ($(b1)+(7.1mm,0)$) {};
      
      \draw[boomlijngroen] (b1) -- (b1-);
      
      \node (kpluseen) at ($(b1)+(\eenbegin+3mm,0)$) {$\fromstates{k+1}$};
      \node (kplusnul) at ($(b0)+(\eenbegin+15mm,12mm)$) {$\explfromopt[\fromstates{k+1}]{\hspace{1mm}0}{}{k+1}$};
      \node (mog) at ($(b0)+(\eenbegin-25mm,12mm)$) {$\explfrommog[\fromstates{k}]{0}{}{k}$};
    \end{tikzpicture}
  %\caption{Tree representation of $\explfrommog[\fromstates{k}]{0}{}{k}$}
  % \label{fig:mog}
  \end{center}
  The set $\explfrommog[\fromstates{k}]{0}{}{k}$ we are looking for is then trivially constructed by joining the two subsets $0\oplus\explfromopt[\fromstates{k+1}]{0}{}{k+1}$ and $1\oplus\states{k+1:n}$, as depicted above.
  \hfill
  $\blacklozenge$
\end{example}

It follows from Equation~\eqref{eq:criterion-at-k} that the data structure representing $\fromopt[\fromstates{k}]{z}{k}$ is contained in the data structure representing $\frommog[\fromstates{k}]{z}{k}$. 
All that is now left to do is find this subset in an efficient manner.
We present a method that constructs a subset of $\frommog[\fromstates{k}]{z}{k}$, and will prove that this subset is indeed $\fromopt[\fromstates{k}]{z}{k}$.

We first define $\zumemoptfromto{k}{s}$ for every $k\in\{1,\dots,n\}$, $s\in\{k,\dots,n\}$, $\zstate{k-1}\in\states{k-1}$ and $\zstate{k:s}\in\states{k:s}$. 
If $s=k$, we let $\zumemoptfromto{k}{k}\coloneqq\zumemopt{k}$, defined by Equation~\eqref{eq:alphaopt}. 
$\zumemoptfromto{k}{s}$ is then recursively defined by
\begin{equation}\label{eq:optalgemeen}
  \zumemoptfromto{k}{s}
  =\frac{\zumemoptfromto{k}{s-1}}
  {\zinoutcupr[\singout{s-1}]{s-1}\zinstatecupr[\singzstate{s}]{s}{s-1}}
  \text{ for every $s\in\{k+1,\dots,n\}$.}
\end{equation}

\subsection*{Optimal tree construction}
The following method will select a subset out of a given set $\frommog[\fromstates{k}]{z}{k}$ constructed using Equation~\eqref{eq:pop2imprecies}.

First, for every $\xstate{k}\in\states{k}$, check whether
\begin{equation}\label{eq:immediatecondition}
  \xumemmax{k}\geq\xumemopt{k}.
\end{equation}
From now on, we will use the generic notation $\xhatstate{k}$ for those $\xstate{k}\in\states{k}$ for which this condition is satisfied.

Next, choose an arbitrary $\xhatstate{k}$ and check, for every $\xstate{k+1}\in\states{k+1}$ that has a non-empty set $\frommogDoor[\fromstates{k}]{z}{k}{\xhatstate{k}\oplus\xstate{k+1}}$, if the following condition is satisfied:
\begin{equation}\label{eq:nextimmediatecondition}
  \xumemmax{k+1}\geq\specialumemopt{k}{\xhatstate{k}\oplus\xstate{k+1}}.
\end{equation}
Notice that $\specialumemopt{k}{\xhatstate{k}\oplus\xstate{k+1}}$ can be easily calculated using Equation~\eqref{eq:optalgemeen}, because $\xhatumemopt{k}$ is already known from the previous recursion step. 
Denote those $\xstate{k+1}\in\states{k+1}$ for which the inequality \eqref{eq:nextimmediatecondition} is true generically by $\xhatstate{k+1}$ and concatenate them with the state $\xhatstate{k}$, creating a set of state sequences $\xhatstate{k:k+1}$. 
Do this for every $\xhatstate{k}$ of the previous step and bundle the sets, obtaining a larger set of state sequences $\xhatstate{k:k+1}$.

In a next step, consider an arbitrary $\xhatstate{k:k+1}$ and check, for every $\xstate{k+2}\in\states{k+2}$ that has a non-empty set $\frommogDoor[\fromstates{k}]{z}{k}{\xhatstate{k:k+1}\oplus\xstate{k+2}}$, if the following condition is satisfied:
\begin{equation}\label{eq:nextnextimmediatecondition}
  \xumemmax{k+2}\geq\specialumemopt{k}{\xhatstate{k:k+1}\oplus\xstate{k+2}}.
\end{equation}
As before, $\specialumemopt{k}{\xhatstate{k:k+1}\oplus\xstate{k+2}}$ can be calculated easily using Equation~\eqref{eq:optalgemeen}, since $\specialumemopt{k}{\xhatstate{k+1}}$ has already been calculated in the previous step.
Denote those $\xstate{k+2}\in\states{k+2}$ for which the inequality~\eqref{eq:nextnextimmediatecondition} holds generically by $\xhatstate{k+2}$ and concatenate them with $\xhatstate{k:k+1}$, creating a set of state sequences~$\xhatstate{k:k+2}$. 
Do this for every $\xhatstate{k:k+1}$ from the previous step and bundle the sets to obtain a larger set of state sequences $\xhatstate{k:k+2}$.

It should be clear that  we can go on this way, to eventually end up with a set of sequences $\xhatstate{k:n-1}$. 
Now consider an arbitrary $\xhatstate{k:n-1}$ and check, for every $\xstate{n}\in\states{n}$ that has a non-empty set $\frommogDoor[\fromstates{k}]{z}{k}{\xhatstate{k:n-1}\oplus\xstate{n}}$, if the following condition holds:
\begin{equation}\label{eq:lastimmediatecondition}
  \xumemmax{n}\geq\specialumemopt{k}{\xhatstate{k:n-1}\oplus\xstate{n}}.
\end{equation}
Denote those $\xstate{n}\in\states{n}$ for which this is the case as $\xhatstate{n}$, and concatenate them with $\xhatstate{k:n-1}$, creating a set of state sequences $\xhatstate{k:n}$. 
Do this for every $\xhatstate{k:n-1}$ from the previous step and bundle the sets to finally obtain a set of state sequences $\xhatstate{k:n}$, which is a subset of the set $\frommog[\fromstates{k}]{z}{k}$ we started out from.

\begin{theorem}\label{theorem:construction}
The subset of $\frommog[\fromstates{k}]{z}{k}$ that is obtained by
using the optimal tree construction is equal to $\fromopt[\fromstates{k}]{z}{k}$.
\end{theorem}

\begin{example}
We continue with Example~\ref{exam:running2}.
  Following the optimal tree construction, we start by checking for every $\xstate{k}\in\{0,1\}$ whether $\xumemmax{k}\geq\explxumemopt{k}{0}$. 
  Suppose this is the case. 
  We will symbolise this by giving the corresponding nodes in our representation a green colour, as in the leftmost part of the figure below.
  It then follows by Theorem~\ref{theorem:construction} that every sequence in $\explfromopt[\fromstates{k}]{0}{}{k}$ will either start with $0$ or $1$, since the set of $\xhatstate{k}$ is $\{0,1\}$. 
  In this example, this is of course trivial, but if the set of $\xhatstate{k}$ would have been $\{0\}$, we would have obtained the non-trivial result that every sequence in $\explfromopt[\fromstates{k}]{0}{}{k}$ starts with $0$. 
  We can represent this partial information about the set $\explfromopt[\fromstates{k}]{0}{}{k}$ in a trivial way, as in the rightmost part of the figure below.
  \begin{center}
    \begin{tikzpicture}[scale=.9]
      \def\boomhor{12mm};
      \def\boomver{9mm};
      \def\boomverhoofd{23mm};
      
      \def\eenbegin{12mm};
      \def\eeneind{41mm};
      \def\eenrand{6mm};
      
      \path[kplus] ($(0,-\boomverhoofd)+(\eenbegin,0)+(0,1mm)+(-\eenrand,\eenrand)$) rectangle
      ($(0,-\boomverhoofd)+(2mm,-1mm)+(\eeneind,0)+(\eenrand,-\eenrand)$);
      \path[kplus] ($(\eenbegin,0)+(0,11mm)+(-\eenrand,\eenrand)$)
      rectangle ($(\eeneind,0) +(2mm,-5mm)+(\eenrand,-\eenrand)$);
      \path[kplus] ($(\eenbegin,0)+(0,11mm)+(-\eenrand,\eenrand)+(-40mm,4mm)$) rectangle ($(0,-\boomverhoofd)+(2mm,-1mm)+(\eeneind,0) +(\eenrand,-\eenrand)+(4mm,-4mm)$);
      
      \node (nul) at (5cm,1cm) {};
      \node[ja] (b0) at (0,0) {0};
      \node[niets] (b00) at ($(b0)+(\boomhor,0.5*\boomver)$) {0};
      \node[niets] (b000) at ($(b00)+(\boomhor,0)$) {0};
      \node (b000-) at ($(b000)+(\boomhor,0)$) {};
      
      \draw[boomlijngroen] (b0) -- (b00);
      \draw[boomlijngroen] (b00) -- (b000);
      \draw[boomstippellijn] (b000) -- (b000-);
      
      \node[niets] (b01) at ($(b0)+(\boomhor,-0.5*\boomver)$) {1};
      \node[niets] (b010) at ($(b01)+(\boomhor,0)$) {0};
      \node (b010-) at ($(b010)+(\boomhor,0)$) {};
      
      \draw[boomlijngroen] (b0) -- (b01);
      \draw[boomlijngroen] (b01) -- (b010);
      \draw[boomstippellijn] (b010) -- (b010-);
      
      \node[ja] (b1) at ($(b0)+(0,-\boomverhoofd)$) {1};
      \node (b1-) at ($(b1)+(7.1mm,0)$) {};
      
      \draw[boomlijngroen] (b1) -- (b1-);
      
      % \draw[boomlijngroen] (b1) -- ($(b1)+(\eenbegin,0)+(-\eenrand,0)$);
      
      \node (kpluseen) at ($(b1)+(\eenbegin+3mm,0)$) {$\fromstates{k+1}$};
      \node (kplusnul) at ($(b0)+(\eenbegin+15mm,12mm)$) {$\explfromopt[\fromstates{k+1}]{\hspace{1mm}0}{}{k+1}$};
      \node (mog) at ($(b0)+(\eenbegin-25mm,12mm)$) {$\explfrommog[\fromstates{k}]{0}{}{k}$};
    \end{tikzpicture}
    \quad
    \begin{tikzpicture}[scale=.9]
      \def\boomverhoofd{10mm};
      \def\eenbegin{12mm};
      \def\eeneind{0mm};
      \def\eenrand{6mm};
      
      \path[kplus] ($(\eenbegin,0)+(0,0mm)+(-\eenrand,\eenrand)+(-13mm,1mm)$) rectangle ($(0,-\boomverhoofd)+(0,-0mm)+(\eeneind,0) +(\eenrand,-\eenrand)+(1mm,-1mm)$);
      
      \node[ja] (b0) at (0,0) {0};
      \node[ja] (b1) at ($(b0)+(0,-\boomverhoofd)$) {1};  
      % \node (mog) at ($(b0)+(\eenbegin-33mm,-5mm)$) {\large $\explfromopt[\fromstates{k}]{0}{}{k}$};
    \end{tikzpicture}
  \end{center}

  In the next step, we need to check some criteria for every $\xhatstate{k}$ we have found in the previous step. 
  We begin with $\xhatstate{k}=0$ and start by looking at $\xstate{k+1}=0$. 
  The set $\explfrommogDoor[\fromstates{k}]{0}{k}{\xhatstate{k}\oplus\xstate{k+1}}$ is then $\explfrommogDoor[\fromstates{k}]{0}{k}{00}$, which is simply the subset of sequences in $\explfrommog[\fromstates{k}]{0}{}{k}$ that start with $00$. In our tree representation of $\explfrommog[\fromstates{k}]{0}{}{k}$, checking whether this set is non-empty comes down to checking if the node $\xhatstate{k}=0$ has a daughter with value $0$. Since this is indeed the case, we need to check whether $\explumemmax{k+1}{0}\geq\explspecialumemopt{k}{\xhatstate{k}\oplus\xstate{k+1}}{0}=\explspecialumemopt{k}{00}{0}$. 
  Suppose this criterion is met, then we have found our first subsequence $\fromtoxhatstate{k}{k+1}$, namely $00$. 
  We symbolise this in the figure below by giving the child $\xstate{k+1}=0$ of the node $\xhatstate{k}=0$ a green colour. 

  The node $\xhatstate{k}=0$ also has a daughter $\xstate{k+1}=1$. 
  If $\explumemmax{k+1}{1}<\explspecialumemopt{k}{01}{0}$, this daughter gets coloured red and $01$ is not part of the set of sequences $\fromtoxhatstate{k}{k+1}$ we are constructing in this step. 
  By Theorem~\ref{theorem:construction}, this also means that none of the elements of $\explfromopt[\fromstates{k}]{0}{}{k}$ will start with the subsequence $01$.

  For $\xhatstate{k}=1$, we know that the tree representing the sequences in $\explfrommog[\fromstates{k}]{0}{}{k}$ that start with $1$ is a complete tree, which we have not explicitly constructed. 
  This does not create a problem, since we only need that tree to check whether $\frommogDoor[\fromstates{k}]{z}{k}{1\oplus\xstate{k+1}}$ is a non-empty set, which is a condition that is trivially met for all $\xstate{k+1}\in\states{k+1}$ because of the completeness of the set $\frommogDoor[\fromstates{k}]{z}{k}{1}$. 
  We are therefore left to check Criterion~\eqref{eq:nextimmediatecondition} for $\xhatstate{k}=1$ and every $\xstate{k+1}\in\{0,1\}$. 
  For $\xstate{k+1}=0$, we might for instance find that $\explumemmax{k+1}{0}<\explspecialumemopt{k}{10}{0}$ and for $\xstate{k+1}=1$ we might find that $\explumemmax{k+1}{1}\geq\explspecialumemopt{k}{11}{0}$.

  The results of these checks are summarised in the leftmost part of the figure below.
  The corresponding sequences $\fromtoxhatstate{k}{k+1}$, which by Theorem~\ref{theorem:construction} are the possible starting sequences for the elements of $\explfromopt[\fromstates{k}]{0}{}{k}$, can be easily stored and depicted in our tree representation; see the rightmost part of the following figure.
  \begin{center}
    \begin{tikzpicture}[scale=.9]
      \def\boomhor{12mm};
      \def\boomver{9mm};
      \def\boomverhoofd{32mm};
      
      \def\eenbegin{12mm};
      \def\eeneind{41mm};
      \def\eenrand{6mm};
      
      \path[kplus] ($(0,-\boomverhoofd)+(\eenbegin,0)+(0,11mm)+(-\eenrand,\eenrand)$) rectangle
      ($(0,-\boomverhoofd)+(2mm,-5mm)+(\eeneind,0)+(\eenrand,-\eenrand)$);
      \path[kplus] ($(\eenbegin,0)+(0,11mm)+(-\eenrand,\eenrand)$)
      rectangle ($(\eeneind,0) +(2mm,-5mm)+(\eenrand,-\eenrand)$);
      \path[kplus] ($(\eenbegin,0)+(0,11mm)+(-\eenrand,\eenrand)+(-40mm,4mm)$) rectangle ($(0,-\boomverhoofd)+(2mm,-5mm)+(\eeneind,0) +(\eenrand,-\eenrand)+(4mm,-4mm)$);
      
      \node (nul) at (5cm,1cm) {};
      \node[ja] (b0) at (0,0) {0};
      \node[ja] (b00) at ($(b0)+(\boomhor,0.5*\boomver)$) {0};
      \node[niets] (b000) at ($(b00)+(\boomhor,0)$) {0};
      \node (b000-) at ($(b000)+(\boomhor,0)$) {};
      
      \draw[boomlijngroen] (b0) -- (b00);
      \draw[boomlijngroen] (b00) -- (b000);
      \draw[boomstippellijn] (b000) -- (b000-);
      
      \node[nee] (b01) at ($(b0)+(\boomhor,-0.5*\boomver)$) {1};
      \node[niets] (b010) at ($(b01)+(\boomhor,0)$) {0};
      \node (b010-) at ($(b010)+(\boomhor,0)$) {};
      
      \draw[boomlijngroen] (b0) -- (b01);
      \draw[boomlijngroen] (b01) -- (b010);
      \draw[boomstippellijn] (b010) -- (b010-);
      
      \node[ja] (b1) at ($(b0)+(0,-\boomverhoofd)$) {1};
      \node[nee] (b10) at ($(b1)+(\boomhor,0.5*\boomver)$) {0};
      \node[ja] (b11) at ($(b1)+(\boomhor,-0.5*\boomver)$) {1};
      
      \draw[boomlijngroen] (b1) -- (b10);
      \draw[boomlijngroen] (b1) -- (b11);
      % \draw[boomlijngroen] (b1) -- ($(b1)+(\eenbegin,0)+(-\eenrand,0)$);
      
      \node (kpluseen) at ($(b1)+(\eenbegin+3mm,12mm)$) {$\fromstates{k+1}$};
      \node (kplusnul) at ($(b0)+(\eenbegin+15mm,12mm)$) {$\explfromopt[\fromstates{k+1}]{0}{}{k+1}$};
      \node (mog) at ($(b0)+(\eenbegin-25mm,12mm)$) {$\explfrommog[\fromstates{k}]{0}{}{k}$};
    \end{tikzpicture}
    \quad
    \begin{tikzpicture}[scale=.9]
      \def\boomverhoofd{10mm};
      \def\boomhor{12mm};
      \def\boomver{9mm};
      
      \def\eenbegin{12mm};
      \def\eeneind{12mm};
      \def\eenrand{6mm};
      
      \path[kplus] ($(\eenbegin,0)+(0,0mm)+(-\eenrand,\eenrand)+(-13mm,1mm)$) rectangle ($(0,-\boomverhoofd)+(0,-0mm)+(\eeneind,0) +(\eenrand,-\eenrand)+(1mm,-1mm)$);
      
      \node[ja] (b0) at (0,0) {0};
      \node[ja] (b00) at ($(b0)+(\boomhor,0)$) {0};
      
      \draw[boomlijngroen] (b0) -- (b00);
      
      \node[ja] (b1) at ($(b0)+(0,-\boomverhoofd)$) {1};
      \node[ja] (b11) at ($(b1)+(\boomhor,0)$) {1};
      
      \draw[boomlijngroen] (b1) -- (b11);
      % \node (mog) at ($(b0)+(\eenbegin-33mm,-5mm)$) {\large $\explfromopt[\fromstates{k}]{0}{}{k}$};
    \end{tikzpicture}
  \end{center}
  If we keep performing the steps of optimal tree construction in this way, Theorem~\ref{theorem:construction} states that the data structure that is built up while checking all these criteria represents the set $\explfromopt[\fromstates{k}]{0}{}{k}$. 
  This set might look like this:
  \begin{center}
    \begin{tikzpicture}[scale=.9]
      \def\boomverhoofd{9mm};
      \def\boomhor{12mm};
      \def\boomver{9mm};
      
      \def\eenbegin{12mm};
      \def\eeneind{96mm};
      \def\eenrand{6mm};
      
      \path[kplus]
      ($(\eenbegin,0)+(0,\boomver)+(-\eenrand,\eenrand)+(-13mm,1mm)$) rectangle
      ($(0,-\boomverhoofd)+(0,-2*\boomver)+(\eeneind,0)
      +(\eenrand,-\eenrand)+(1mm,-1mm)$);
      
      % \node (nul) at (10.6cm,0) {};
      \node[ja] (b0) at (0,0) {0};
      \node[ja] (b00) at ($(b0)+(\boomhor,0)$) {0};
      \node[ja] (b000) at ($(b00)+(\boomhor,0)$) {0};
      \node[ja] (b0000) at ($(b000)+(\boomhor,0)$) {0};
      
      \draw[boomlijngroen] (b0) -- (b00);
      \draw[boomlijngroen] (b00) -- (b000);
      \draw[boomlijngroen] (b000) -- (b0000);
      
      \node[ja] (b00000) at ($(b0000)+(\boomhor,0)$) {0};
      \node[ja] (b000001) at ($(b00000)+(\boomhor,0)$) {1};
      \node[ja] (b0000011) at ($(b000001)+(\boomhor,0)$) {1};
      \node[ja] (b00000111) at ($(b0000011)+(\boomhor,0)$) {1};
      \node[ja] (b000001110) at ($(b00000111)+(\boomhor,0)$) {0};
      
      \draw[boomlijngroen] (b0000) -- (b00000);
      \draw[boomlijngroen] (b00000) -- (b000001);
      \draw[boomlijngroen] (b000001) -- (b0000011);
      \draw[boomlijngroen] (b0000011) -- (b00000111);
      \draw[boomlijngroen] (b00000111) -- (b000001110);
      
      \node[ja] (b0000010) at ($(b0000011)+(0,\boomver)$) {0};
      \node[ja] (b00000101) at ($(b0000010)+(\boomhor,0)$) {1};
      \node[ja] (b000001010) at ($(b00000101)+(\boomhor,0)$) {0};
      
      \draw[boomlijngroen] (b000001) -- (b0000010);
      \draw[boomlijngroen] (b0000010) -- (b00000101);
      \draw[boomlijngroen] (b00000101) -- (b000001010);
      
      \node[ja] (b1) at ($(b0)+(0,-\boomverhoofd)$) {1};
      \node[ja] (b11) at ($(b1)+(\boomhor,0)$) {1};
      
      \draw[boomlijngroen] (b1) -- (b11);
      
      \node[ja] (b110) at ($(b11)+(\boomhor,0)$) {0};
      \node[ja] (b1100) at ($(b110)+(\boomhor,0)$) {0};
      
      \draw[boomlijngroen] (b11) -- (b110);
      \draw[boomlijngroen] (b110) -- (b1100);
      
      \node[ja] (b11001) at ($(b1100)+(\boomhor,0)$) {1};
      \node[ja] (b110011) at ($(b11001)+(\boomhor,0)$) {1};
      \node[ja] (b1100111) at ($(b110011)+(\boomhor,0)$) {1};
      \node[ja] (b11001111) at ($(b1100111)+(\boomhor,0)$) {1};
      \node[ja] (b110011110) at ($(b11001111)+(\boomhor,0)$) {0};
      
      \draw[boomlijngroen] (b1100) -- (b11001);
      \draw[boomlijngroen] (b11001) -- (b110011);
      \draw[boomlijngroen] (b110011) -- (b1100111);
      \draw[boomlijngroen] (b1100111) -- (b11001111);
      \draw[boomlijngroen] (b11001111) -- (b110011110);
      
      \node[ja] (b1101) at ($(b110)+(\boomhor,-\boomver)$) {1};
      \node[ja] (b11011) at ($(b1101)+(\boomhor,0)$) {1};
      \node[ja] (b110111) at ($(b11011)+(\boomhor,0)$) {1};
      \node[ja] (b1101111) at ($(b110111)+(\boomhor,0)$) {1};
      \node[ja] (b11011111) at ($(b1101111)+(\boomhor,0)$) {1};
      \node[ja] (b110111110) at ($(b11011111)+(\boomhor,0)$) {0};
      
      \draw[boomlijngroen] (b110) -- (b1101);
      \draw[boomlijngroen] (b1101) -- (b11011);
      \draw[boomlijngroen] (b11011) -- (b110111);
      \draw[boomlijngroen] (b110111) -- (b1101111);
      \draw[boomlijngroen] (b1101111) -- (b11011111);
      \draw[boomlijngroen] (b11011111) -- (b110111110);
      
      \node[ja] (b110111111) at ($(b11011111)+(\boomhor,-\boomver)$) {1};
      
      \draw[boomlijngroen] (b11011111) -- (b110111111);
      
      \node (mog) at ($(b0)+(\eenbegin-0mm,9mm)$) {$\explfromopt[\fromstates{k}]{0}{}{k}$};
    \end{tikzpicture}
  \end{center}
  Figure~\ref{fig:constructionofopt} should clarify how this set was constructed. 
  Notice that we have indeed never explicitly constructed the set $\fromstates{k+1}$ in the tree representation, since every time we reached a red node, the descendants of this node were not constructed.
  \hfill
  $\blacklozenge$
\end{example}

\begin{figure}[h]
  \centering
  \begin{tikzpicture}[scale=.9]
    \def\boomhor{12mm};
    \def\boomver{9mm};
    \def\boomverhoofd{35mm};
    
    \def\eenbegin{12mm};
    \def\eeneind{96mm};
    \def\eenrand{6mm};
    
    \path[kplus] ($(0,-\boomverhoofd)+(\eenbegin,0)+(0,5mm)+(-\eenrand,\eenrand)$) rectangle
    ($(0,-\boomverhoofd)+(0,-32mm)+(\eeneind,0)+(\eenrand,-\eenrand)$);
    \path[kplus] ($(\eenbegin,0)+(0,32mm)+(-\eenrand,\eenrand)$)
    rectangle ($(\eeneind,0) +(0,-14mm)+(\eenrand,-\eenrand)$);
    \path[kplus] ($(\eenbegin,0)+(0,40mm)+(-\eenrand,\eenrand)+(-13mm,4mm)$) rectangle ($(0,-\boomverhoofd)+(0,-32mm)+(\eeneind,0) +(\eenrand,-\eenrand)+(4mm,-4mm)$);
    
    \node (nul) at (5cm,1cm) {};
    \node[ja] (b0) at (0,0) {0};
    \node[ja] (b00) at ($(b0)+(\boomhor,0.5*\boomver)$) {0};
    
    \node[ja] (b000) at ($(b00)+(\boomhor,0)$) {0};
    \node[ja] (b0000) at ($(b000)+(\boomhor,0)$) {0};
    \node[nee] (b00001) at ($(b0000)+(\boomhor,0)$) {1};
    \node[niets] (b000011) at ($(b00001)+(\boomhor,0)$) {1};
    \node[niets] (b0000111) at ($(b000011)+(\boomhor,0)$) {1};
    \node[niets] (b00001111) at ($(b0000111)+(\boomhor,0)$) {1};
    \node[niets] (b000011110) at ($(b00001111)+(\boomhor,0)$) {0};
    
    \draw[boomlijngroen] (b0) -- (b00);
    \draw[boomlijngroen] (b00) -- (b000);
    \draw[boomlijngroen] (b000) -- (b0000);
    \draw[boomlijngroen] (b0000) -- (b00001);
    \draw[boomlijngroen] (b00001) -- (b000011);
    \draw[boomlijngroen] (b000011) -- (b0000111);
    \draw[boomlijngroen] (b0000111) -- (b00001111);
    \draw[boomlijngroen] (b00001111) -- (b000011110);
    
    \node[ja] (b00000) at ($(b00001)+(0,\boomver)$) {0};
    \node[ja] (b000001) at ($(b00000)+(\boomhor,0)$) {1};
    \node[ja] (b0000011) at ($(b000001)+(\boomhor,0)$) {1};
    \node[ja] (b00000111) at ($(b0000011)+(\boomhor,0)$) {1};
    \node[ja] (b000001110) at ($(b00000111)+(\boomhor,0)$) {0};
    
    \draw[boomlijngroen] (b0000) -- (b00000);
    \draw[boomlijngroen] (b00000) -- (b000001);
    \draw[boomlijngroen] (b000001) -- (b0000011);
    \draw[boomlijngroen] (b0000011) -- (b00000111);
    \draw[boomlijngroen] (b00000111) -- (b000001110);
    
    \node[ja] (b0000010) at ($(b0000011)+(0,\boomver)$) {0};
    \node[ja] (b00000101) at ($(b0000010)+(\boomhor,0)$) {1};
    \node[ja] (b000001010) at ($(b00000101)+(\boomhor,0)$) {0};
    
    \draw[boomlijngroen] (b000001) -- (b0000010);
    \draw[boomlijngroen] (b0000010) -- (b00000101);
    \draw[boomlijngroen] (b00000101) -- (b000001010);
    
    \node[nee] (b00000100) at ($(b00000101)+(0,\boomver)$) {0};
    \node[niets] (b000001000) at ($(b00000100)+(\boomhor,0)$) {0};
    
    \draw[boomlijngroen] (b0000010) -- (b00000100);
    \draw[boomlijngroen] (b00000100) -- (b000001000);
    
    \node[nee] (b01) at ($(b0)+(\boomhor,-0.5*\boomver)$) {1};
    \node[niets] (b010) at ($(b01)+(\boomhor,0)$) {0};
    \node[niets] (b0100) at ($(b010)+(\boomhor,0)$) {0};
    \node[niets] (b01000) at ($(b0100)+(\boomhor,0)$) {0};
    \node[niets] (b010001) at ($(b01000)+(\boomhor,0)$) {1};
    \node[niets] (b0100010) at ($(b010001)+(\boomhor,0)$) {0};
    \node[niets] (b01000101) at ($(b0100010)+(\boomhor,0)$) {1};
    \node[niets] (b010001010) at ($(b01000101)+(\boomhor,0)$) {0};

    \draw[boomlijngroen] (b0) -- (b01);
    \draw[boomlijngroen] (b01) -- (b010);
    \draw[boomlijngroen] (b010) -- (b0100);
    \draw[boomlijngroen] (b0100) -- (b01000);
    \draw[boomlijngroen] (b01000) -- (b010001);
    \draw[boomlijngroen] (b010001) -- (b0100010);
    \draw[boomlijngroen] (b0100010) -- (b01000101);
    \draw[boomlijngroen] (b01000101) -- (b010001010);
    
    \node[niets] (b0100011) at ($(b0100010)+(0,-\boomver)$) {1};
    \node[niets] (b01000111) at ($(b0100011)+(\boomhor,0)$) {1};
    \node[niets] (b010001110) at ($(b01000111)+(\boomhor,0)$) {0};
    
    \draw[boomlijn] (b010001) -- (b0100011);
    \draw[boomlijn] (b0100011) -- (b01000111);
    \draw[boomlijn] (b01000111) -- (b010001110);
    
    \node[ja] (b1) at ($(b0)+(0,-\boomverhoofd)$) {1};
    \node[nee] (b10) at ($(b1)+(\boomhor,0.5*\boomver)$) {0};
    \node[ja] (b11) at ($(b1)+(\boomhor,-0.5*\boomver)$) {1};
    
    \draw[boomlijngroen] (b1) -- (b10);
    \draw[boomlijngroen] (b1) -- (b11);
    
    \node[ja] (b110) at ($(b11)+(\boomhor,0)$) {0};
    \node[ja] (b1100) at ($(b110)+(\boomhor,0)$) {0};
    
    \draw[boomlijngroen] (b11) -- (b110);
    \draw[boomlijngroen] (b110) -- (b1100);
    
    \node[nee] (b111) at ($(b11)+(\boomhor,-\boomver)$) {1};
    
    \draw[boomlijngroen] (b11) -- (b111);
    
    \node[ja] (b11001) at ($(b1100)+(\boomhor,0)$) {1};
    \node[ja] (b110011) at ($(b11001)+(\boomhor,0)$) {1};
    \node[ja] (b1100111) at ($(b110011)+(\boomhor,0)$) {1};
    \node[ja] (b11001111) at ($(b1100111)+(\boomhor,0)$) {1};
    \node[ja] (b110011110) at ($(b11001111)+(\boomhor,0)$) {0};
    
    \draw[boomlijngroen] (b1100) -- (b11001);
    \draw[boomlijngroen] (b11001) -- (b110011);
    \draw[boomlijngroen] (b110011) -- (b1100111);
    \draw[boomlijngroen] (b1100111) -- (b11001111);
    \draw[boomlijngroen] (b11001111) -- (b110011110);
    
    \node[nee] (b11000) at ($(b1100)+(\boomhor,\boomver)$) {0};
    \node[nee] (b110010) at ($(b11001)+(\boomhor,\boomver)$) {0};
    \node[nee] (b1100110) at ($(b110011)+(\boomhor,\boomver)$) {0};
    \node[nee] (b11001110) at ($(b1100111)+(\boomhor,\boomver)$) {0};
    \node[nee] (b110011111) at ($(b11001111)+(\boomhor,-\boomver)$) {1};
    
    \draw[boomlijngroen] (b1100) -- (b11000);
    \draw[boomlijngroen] (b11001) -- (b110010);
    \draw[boomlijngroen] (b110011) -- (b1100110);
    \draw[boomlijngroen] (b1100111) -- (b11001110);
    \draw[boomlijngroen] (b11001111) -- (b110011111);
    
    \node[ja] (b1101) at ($(b110)+(\boomhor,-\boomver)$) {1};
    \node[ja] (b11011) at ($(b1101)+(\boomhor,-\boomver)$) {1};
    \node[ja] (b110111) at ($(b11011)+(\boomhor,0)$) {1};
    \node[ja] (b1101111) at ($(b110111)+(\boomhor,0)$) {1};
    \node[ja] (b11011111) at ($(b1101111)+(\boomhor,0)$) {1};
    \node[ja] (b110111110) at ($(b11011111)+(\boomhor,0)$) {0};
    
    \draw[boomlijngroen] (b110) -- (b1101);
    \draw[boomlijngroen] (b1101) -- (b11011);
    \draw[boomlijngroen] (b11011) -- (b110111);
    \draw[boomlijngroen] (b110111) -- (b1101111);
    \draw[boomlijngroen] (b1101111) -- (b11011111);
    \draw[boomlijngroen] (b11011111) -- (b110111110);
    
    \node[nee] (b11010) at ($(b1101)+(\boomhor,0)$) {0};
    \node[nee] (b110110) at ($(b11011)+(\boomhor,\boomver)$) {0};
    \node[nee] (b1101110) at ($(b110111)+(\boomhor,\boomver)$) {0};
    \node[nee] (b11011110) at ($(b1101111)+(\boomhor,\boomver)$) {0};
    
    \draw[boomlijngroen] (b1101) -- (b11010);
    \draw[boomlijngroen] (b11011) -- (b110110);
    \draw[boomlijngroen] (b110111) -- (b1101110);
    \draw[boomlijngroen] (b1101111) -- (b11011110);
    
    \node[ja] (b110111111) at ($(b11011111)+(\boomhor,-\boomver)$) {1};
    
    \draw[boomlijngroen] (b11011111) -- (b110111111);
    
    % \draw[boomlijngroen] (b1) -- ($(b1)+(\eenbegin,0)+(-\eenrand,0)$);
    
    \node (kpluseen) at ($(b1)+(\eenbegin+4mm,-32mm)$) {$\fromstates{k+1}$};
    \node (kplusnul) at ($(b0)+(\eenbegin+16mm,32mm)$) {$\explfromopt[\fromstates{k+1}]{\hspace{1mm}0}{}{k+1}$};
    \node (mog) at ($(b0)+(\eenbegin+1mm,44mm)$) {$\explfrommog[\fromstates{k}]{0}{}{k}$};
  \end{tikzpicture}
  \caption{Clarification of the construction of $\explfrommog[\fromstates{k}]{0}{}{k}$}
  \label{fig:constructionofopt}
\end{figure}

\subsection{Additional comments}\label{subsec:comments}
All that is needed in order to produce the $\umem$- and $\lmem$-functions are assessments for the lower and upper transition and emission probabilities:
\begin{center}
  $\zinstateclpr[\singzstate{k}]{k}{k-1}$, $\zinstatecupr[\singzstate{k}]{k}{k-1}$, $\zinoutclpr[\singout{k}]{k}$ and $\zinoutcupr[\singout{k}]{k}$
\end{center}
for all $k\in\{1,\dots,n\}$, $\zstate{k-1}\in\states{k-1}$, $\zstate{k}\in\states{k}$ and $\out{k}\in\outs{k}$. 
The most conservative coherent models $\stateclpr{k}{k-1}$ that correspond to such assessments are $2$-monotone \cite{campos1994,cooman2005e}.
Due to their comonotone additivity \cite{cooman2005e}, this implies that:
\begin{equation*}
  \zinstateclpr[\indxstate{k}-a\indxhatstate{k}]{k}{k-1}
  =\zinstateclpr[\singxstate{k}]{k}{k-1}-a\zinstatecupr[\singxhatstate{k}]{k}{k-1}
\end{equation*}
for all $a\geq0$, and therefore Equation~\eqref{eq:threshold} leads to
\begin{equation}\label{eq:minimal-threshold}
  \zinthreshold{k}
  =\frac{\zinstateclpr[\singxstate{k}]{k}{k-1}}{\zinstatecupr[\singxhatstate{k}]{k}{k-1}}.
\end{equation}
The right-hand side is the smallest possible value of the threshold
$\zinthreshold{k}$ corresponding to the assessments
$\zinstateclpr[\singxstate{k}]{k}{k-1}$ and
$\zinstatecupr[\singxhatstate{k}]{k}{k-1}$, leading to the most
conservative inferences, and therefore the largest possible sets of
maximal sequences, that correspond to these assessments.

\section{Discussion of the algorithm's complexity}\label{sec:complexity}
\subsection{Preparatory calculations}
We begin with the preparatory calculations of the quantities in Equations~\eqref{eq:threshold}--\eqref{eq:alphaopt}.
For the thresholds $\zinthreshold{k}$ in Equation~\eqref{eq:threshold}, the computational complexity is clearly cubic in the number of states, and linear in the number of nodes. 
Calculating the $\xumemmax{k}$ and $\xlmemmax{k}$ in Equations~\eqref{eq:alphamaxrecurs} and~\eqref{eq:betamaxrecurs} is  linear in the number of nodes, and quadratic in the number of states.
The complexity of finding the $\umemopt{k}$ in Equation~\eqref{eq:alphaopt}
is linear in the number of nodes, and cubic in the number of
states.

\subsection{Complexity of the optimal tree construction}
The computational complexity of the optimal tree construction is less trivial. 
Let us start by noting that this construction essentially consists in repeating the same small step over and over again, namely adding a state $\xhatstate{s}$ to an already constructed $\xhatstate{k:s-1}$.\footnote{If $s=k$, we identify $\xhatstate{k:s-1}=\xhatstate{k:k-1}$ with a sequence of length zero.}

To perform such a step for a sequence $\xhatstate{k:s-1}$, we first have to check for all $\xstate{s}\in\states{s}$ whether $\frommogDoor[\fromstates{k}]{z}{k}{\xhatstate{k:s-1}\oplus\xstate{s}}$ is non-empty. 
This can be done in constant time, since our representation reduces this step to checking whether the node $\xstate{s}$ is a daughter of $\xhatstate{s-1}$ in the data structure of $\frommogDoor[\fromstates{k}]{z}{k}{\xhatstate{k:s-1}}$. 
Next, for those $\xstate{s}\in\states{s}$ for which this is indeed the case, we need to check if $\xumemmax{s}\geq\specialumemopt{k}{\xhatstate{k:s-1}\oplus\xstate{s}}$. 
Checking those two criteria for every $\xstate{s}\in\states{s}$ will from now on be called \emph{performing a search step}, and its complexity is linear in the number of states. 
Those $\xstate{s}\in\states{s}$ that meet both criteria will be noted as $\xhatstate{s}$ and concatenated with $\xhatstate{k:s-1}$.

We will now prove that performing such a search step will always yield at least one $\xhatstate{s}$ that can be concatenated with $\xhatstate{k:s-1}$. 

\begin{theorem}\label{theorem:essentialstep}
  Consider an arbitrary sequence $\xhatstate{k:s-1}$ that is created while performing the optimal tree construction, with $k\in\{1,...,n\}$ and $s\in\{k,...,n\}$. 
  Then there is always at least one $\xstate{s}\in\states{s}$ for which both $\frommogDoor[\fromstates{k}]{z}{k}{\xhatstate{k:s-1}\oplus\xstate{s}}$ is non-empty and the inequality $\xumemmax{s}\geq\specialumemopt{k}{\xhatstate{k:s-1}\oplus\xstate{s}}$ holds.
\end{theorem}

\begin{example}
  In our visual representations, this means that every green node will alway have at least one green child, which implies that all green sequences will have length $n-k+1$. 
  \begin{center}
    \begin{tikzpicture}[scale=.9]
      \def\boomhor{12mm};
      \def\boomver{9mm};
      \def\boomverhoofd{35mm};
      
      \def\eenbegin{12mm};
      \def\eeneind{96mm};
      \def\eenrand{6mm};
      
      \path[roodkader] ($(\eenbegin,0)+(0,32mm)+(-\eenrand,\eenrand)+(5*\boomhor+1mm,-2*\boomver-2mm)$)
      rectangle ($(\eeneind,0)
      +(0,-14mm)+(\eenrand,-\eenrand)+(-1mm,3*\boomver+2mm)$);
      \path[roodkader] ($(\eenbegin,0)+(0,32mm)+(-\eenrand,\eenrand)+(1mm,-4*\boomver-2mm)$)
      rectangle ($(\eeneind,0) +(0,-14mm)+(\eenrand,-\eenrand)+(-1mm,2mm)$);
      
      \node (nul) at (5cm,1cm) {};
      \node[ja] (b0) at (0,0) {0};
      \node[ja] (b00) at ($(b0)+(\boomhor,0.5*\boomver)$) {0};
      
      \node[ja] (b000) at ($(b00)+(\boomhor,0)$) {0};
      \node[ja] (b0000) at ($(b000)+(\boomhor,0)$) {0};
      \node[nee] (b00001) at ($(b0000)+(\boomhor,0)$) {1};
      \node[niets] (b000011) at ($(b00001)+(\boomhor,0)$) {1};
      \node[niets] (b0000111) at ($(b000011)+(\boomhor,0)$) {1};
      \node[niets] (b00001111) at ($(b0000111)+(\boomhor,0)$) {1};
      \node[niets] (b000011110) at ($(b00001111)+(\boomhor,0)$) {0};
      
      \draw[boomlijngroen] (b0) -- (b00);
      \draw[boomlijngroen] (b00) -- (b000);
      \draw[boomlijngroen] (b000) -- (b0000);
      \draw[boomlijngroen] (b0000) -- (b00001);
      \draw[boomlijngroen] (b00001) -- (b000011);
      \draw[boomlijngroen] (b000011) -- (b0000111);
      \draw[boomlijngroen] (b0000111) -- (b00001111);
      \draw[boomlijngroen] (b00001111) -- (b000011110);
      
      \node[ja] (b00000) at ($(b00001)+(0,\boomver)$) {0};
      \node[ja] (b000001) at ($(b00000)+(\boomhor,0)$) {1};
      \node[ja] (b0000011) at ($(b000001)+(\boomhor,0)$) {1};
      \node[nee] (b00000111) at ($(b0000011)+(\boomhor,0)$) {1};
      \node[niets] (b000001110) at ($(b00000111)+(\boomhor,0)$) {0};
      
      \draw[boomlijngroen] (b0000) -- (b00000);
      \draw[boomlijngroen] (b00000) -- (b000001);
      \draw[boomlijngroen] (b000001) -- (b0000011);
      \draw[boomlijngroen] (b0000011) -- (b00000111);
      \draw[boomlijngroen] (b00000111) -- (b000001110);
      
      \node[ja] (b0000010) at ($(b0000011)+(0,\boomver)$) {0};
      \node[ja] (b00000101) at ($(b0000010)+(\boomhor,0)$) {1};
      \node[ja] (b000001010) at ($(b00000101)+(\boomhor,0)$) {0};
      
      \draw[boomlijngroen] (b000001) -- (b0000010);
      \draw[boomlijngroen] (b0000010) -- (b00000101);
      \draw[boomlijngroen] (b00000101) -- (b000001010);
      
      \node[nee] (b00000100) at ($(b00000101)+(0,\boomver)$) {0};
      \node[niets] (b000001000) at ($(b00000100)+(\boomhor,0)$) {0};
      
      \draw[boomlijngroen] (b0000010) -- (b00000100);
      \draw[boomlijngroen] (b00000100) -- (b000001000);
      
      \node[ja] (b01) at ($(b0)+(\boomhor,-0.5*\boomver)$) {1};
      \node[ja] (b010) at ($(b01)+(\boomhor,0)$) {0};
      \node[ja] (b0100) at ($(b010)+(\boomhor,0)$) {0};
      \node[ja] (b01000) at ($(b0100)+(\boomhor,0)$) {0};
      \node[ja] (b010001) at ($(b01000)+(\boomhor,0)$) {1};
      \node[nee] (b0100010) at ($(b010001)+(\boomhor,0)$) {0};
      \node[niets] (b01000101) at ($(b0100010)+(\boomhor,0)$) {1};
      \node[niets] (b010001010) at ($(b01000101)+(\boomhor,0)$) {0};
      
      \draw[boomlijngroen] (b0) -- (b01);
      \draw[boomlijngroen] (b01) -- (b010);
      \draw[boomlijngroen] (b010) -- (b0100);
      \draw[boomlijngroen] (b0100) -- (b01000);
      \draw[boomlijngroen] (b01000) -- (b010001);
      \draw[boomlijngroen] (b010001) -- (b0100010);
      \draw[boomlijngroen] (b0100010) -- (b01000101);
      \draw[boomlijngroen] (b01000101) -- (b010001010);
      
      \node[nee] (b0100011) at ($(b0100010)+(0,-\boomver)$) {1};
      \node[niets] (b01000111) at ($(b0100011)+(\boomhor,0)$) {1};
      \node[niets] (b010001110) at ($(b01000111)+(\boomhor,0)$) {0};
      
      \draw[boomlijn] (b010001) -- (b0100011);
      \draw[boomlijn] (b0100011) -- (b01000111);
      \draw[boomlijn] (b01000111) -- (b010001110);   
    \end{tikzpicture}
  \end{center}  
  The situation depicted above is therefore impossible.
  \hfill
  $\blacklozenge$
\end{example}

Next, notice that every optimal sequence $\xhatstate{k:n}$ yielded by the optimal tree construction is built up by adding extra states $\xhatstate{s}$ to an already constructed sequence $\xhatstate{k:s-1}$, repeating this for $s$ going from $k$ to $n$. 
Adding such a state means performing one search step, but Theorem~\ref{theorem:essentialstep} implies that performing a search step also means adding at least one state. 
Therefore, constructing one maximal sequence $\xhatstate{k:n}$ will never take more search steps than the length of this sequence. 
Since performing one search step is linear in the number of states, constructing one maximal sequence is linear in the length of the sequence and the number of states. 
Determining the set $\fromopt[\fromstates{k}]{z}{k}$ of all maximal sequences will thus be linear in the number of sequences, in the length of the sequences and in the number of states.

\subsection{The recursive construction of the solutions}
To construct $\globfromopt$ recursively, we let $k$ run from $n$ to $1$. 
For a fixed $k$, we construct the set $\fromopt[\fromstates{k}]{z}{k}$ for every $\zstate{k-1}\in\states{k-1}$, by means of the optimal tree construction. 
We have already shown that constructing such a set is linear in the number of sequences, the length of the sequences and the number of states. 
This means that performing this recursive construction is quadratic in the length of the sequences, quadratic in number of states and roughly speaking\footnote{For every $k$, constructing the set $\fromopt[\fromstates{k}]{z}{k}$ has linear complexity in the number of maximal elements at that stage.} linear in the number of maximal sequences.

\subsection{General complexity}
The complete algorithm consist of the preparatory calculations and the recursive construction of the solutions. 
We conclude that it is quadratic in the number of nodes, cubic in the number of states, and roughly speaking linear in the number of maximal sequences.

\subsection{Comparison with Viterbi's algorithm}\label{sec:compare-viterbi}
For precise HMMs, the state sequence estimation problem can be solved very efficiently by the Viterbi algorithm \cite{rabiner1989,viterbi1967}, whose complexity is linear in the number of nodes, and quadratic in the number of states. 
However, this algorithm only emits a single optimal (most probable) state sequence, even in cases where there are multiple (equally probable) optimal solutions: this of course simplifies the problem.
If we would content ourselves with giving only a single maximal solution, the ensuing version of our algorithm would have a complexity that is similar to Viterbi's.

So, to allow for a fair comparison between Viterbi's algorithm and ours, we would need to alter Viterbi's algorithm in such a way that it no longer resolves ties arbitrarily, and emits all (equally probable) optimal state sequences. 
This new version will remain linear in the number of nodes, and quadratic in the number of states, but will also have added complexity. 
This can easily be seen by noting that emitting the optimal sequences will be linear in the number of them and thus possibly exponential, if all possible solutions would for example be equally probable.

For the complexity for the most time-consuming part of our algorithm (the recursive construction of the solutions), the only difference is this: Viterbi's approach is linear and ours is quadratic in the number of nodes.  
Where does this difference come from? In iHMMs we have mutually incomparable solutions, whereas in pHMMs the optimal solutions are indifferent, or equally probable. 
This makes sure that the algorithm for pHMMs requires no forward loops, as is the case in the EstiHMM algorithm, when we perform the optimal tree construction. 
We believe that this added complexity is a reasonable price to pay for the robustness that working with imprecise-probabilistic models offers.

\section{Some experiments}
\label{sec:experiments}
While a linear complexity in the number of maximal sequences is probably the best we can hope for, we also see that we will only be able to find all maximal sequences efficiently provided their number is reasonably small.
Should it, say, tend to increase exponentially with the length of the chain, then no algorithm, however cleverly designed, could overcome this hurdle. 
Because this number of maximal sequences is so important, we study its behaviour in more detail. 
In order to do so, we take a closer look at how this number of maximal sequences depends on the transition probabilities of the model, and how it evolves when we let the imprecision of the local models grow. 
We shall see that this number displays very interesting behaviour that can be explained, and even predicted to some extent. 
To allow for easy visualisation, we limit this discussion to binary iHMMs, where both the state and output variables can assume only two possible values, say $0$ and $1$. 

\subsection{Describing a binary stationary iHMM}
We first consider a binary stationary HMM. 
The (precise) transition probabilities for going from one state to the next are completely determined by numbers in the unit interval: the probability $p$ to go from state $0$ to state $0$, and the probability $q$ to go from state $1$ to state $0$.
To further pin down the HMM we also need to specify the (marginal) probability $m$ for the first state to be $0$, and the two emission probabilities: the probability $r$ of emitting output $0$ from state $0$ and the probability $s$ of emitting output $0$ from state $1$.
\par
In this binary case, all coherent imprecise-probabilistic models can be found by contamination: taking convex mixtures of precise models, with mixture coefficient $1-\epsilon$, and the vacuous model, with mixture coefficient $\epsilon$, leading to a so-called linear-vacuous model \cite{walley1991}.
To simplify the analysis, we let the emission model remain precise, and use the same mixture coefficient $\epsilon$ for the marginal and the transition models.
As $\epsilon$ ranges from zero to one, we then evolve from a precise HMM towards an iHMM with vacuous marginal and transition models (and precise emission models).

\subsection{Explaining the basic ideas using a chain of length two}
We now examine the behaviour of an iHMM of length two, with the following (precise) probabilities fixed:\footnote{This choice is of course arbitrary. Different values would yield comparable results.}
\begin{equation*}
  \text{$m=0.1$, $r=0.8$ and $s=0.3$}.
\end{equation*}
Fixing an output sequence and a value for $\epsilon$, we can use our algorithm to calculate the corresponding numbers of maximal state sequences as $p$ and $q$ range over the unit interval. 
The results can be represented conveniently in the form of a heat plot. 
The plots below correspond to the output sequence $\fromtoout{1}{2}=01$.
\par
\begin{wrapfigure}[20]{r}{0pt}
  \begin{minipage}[r]{229pt}
    % \begin{flushright}
    \footnotesize
    \begin{tikzpicture}
      \node[anchor=south west,inner sep=0] (plot) at (0,0) {\includegraphics[width=33mm]{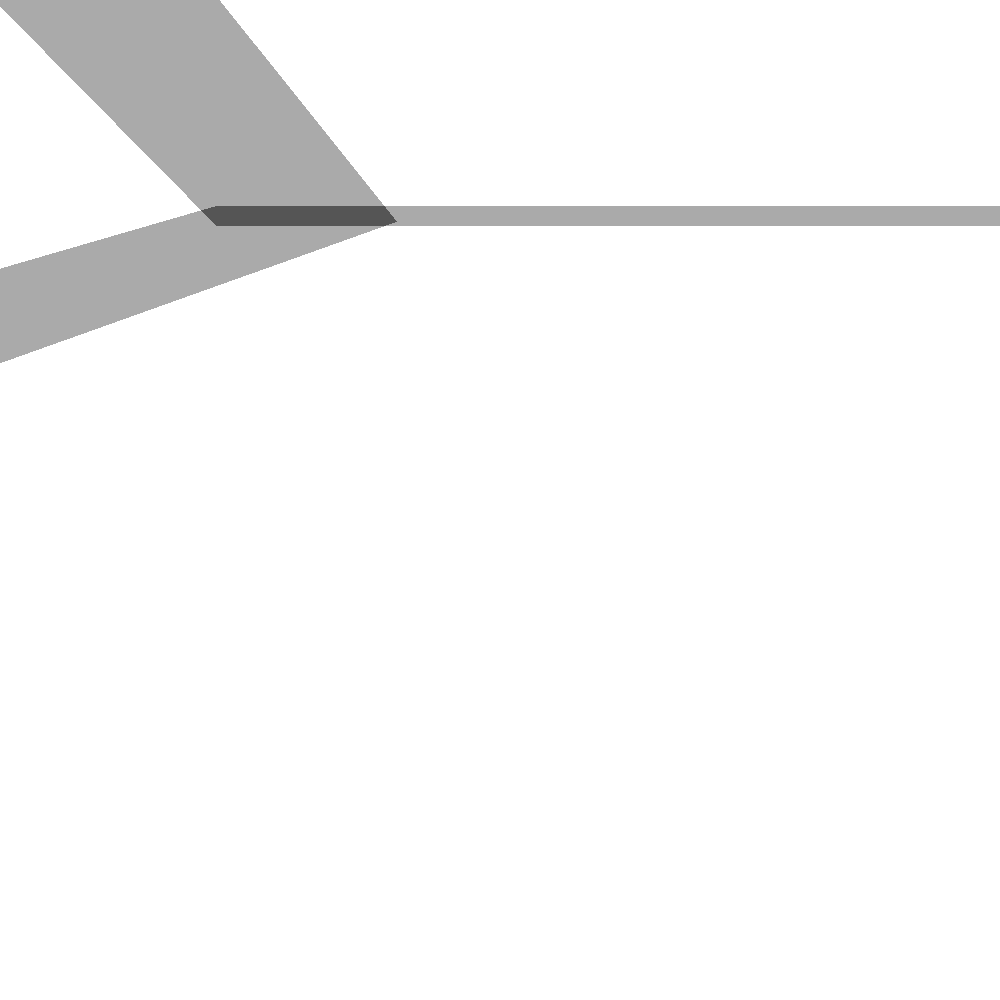}};
      \begin{scope}[x={(plot.south east)},y={(plot.north west)}]
        \draw (0,0) rectangle (1,1) ;
        \node[below] at (0,0) {$0$}; 
        \node[below] at (1,0) {$1$};
        \node[below=5pt] at (0.5,0) {$p$}; 
        \node[left] at (0,0) {$0$}; 
        \node[left] at (0,1) {$1$}; 
        \node[left=5pt] at (0,0.5) {$q$}; 
        \node[comment] at (.75,.25) {$\epsilon=2\%$};
      \end{scope}
    \end{tikzpicture}
    \begin{tikzpicture}
      \node[anchor=south west,inner sep=0] (plot) at (0,0) {\includegraphics[width=33mm]{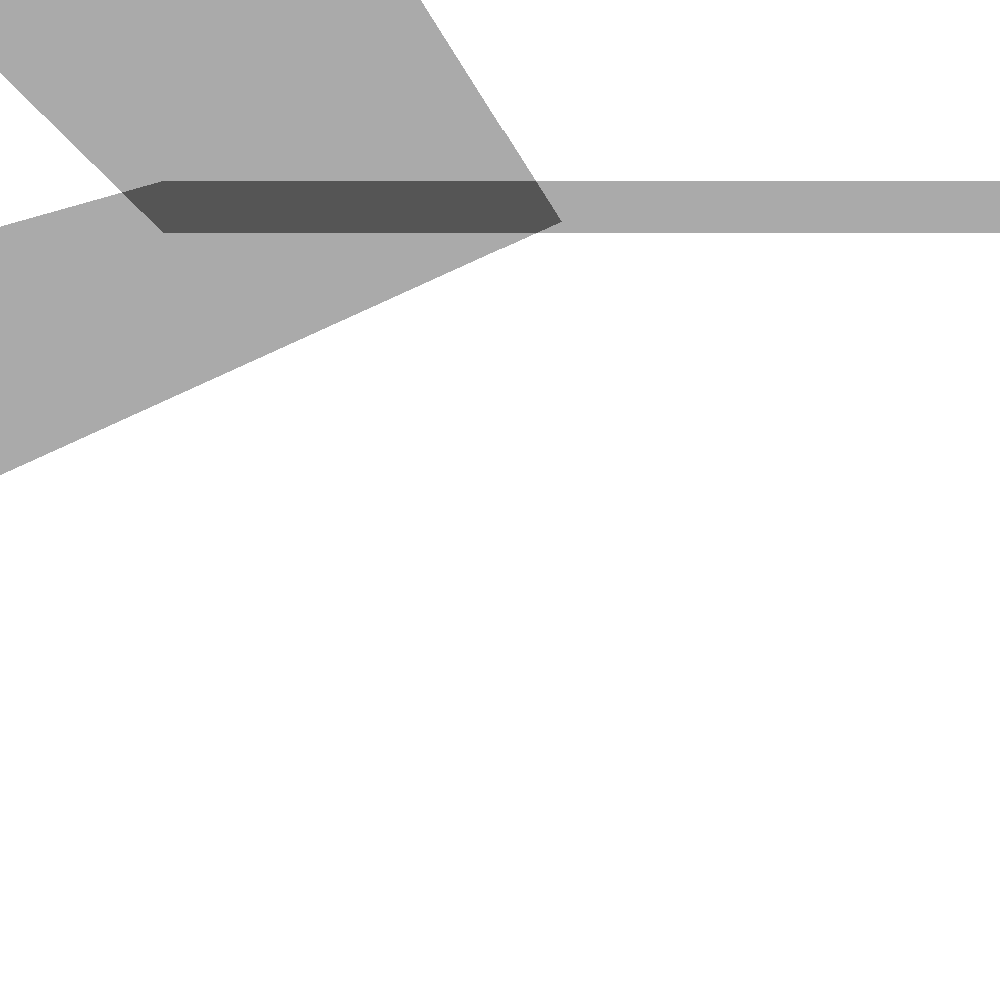}};
      \begin{scope}[x={(plot.south east)},y={(plot.north west)}]
        \draw (0,0) rectangle (1,1) ;
        \node[below] at (0,0) {$0$}; 
        \node[below] at (1,0) {$1$};
        \node[below=5pt] at (0.5,0) {$p$}; 
        \node[left] at (0,0) {$0$}; 
        \node[left] at (0,1) {$1$}; 
        \node[left=5pt] at (0,0.5) {$q$}; 
        \node[comment] at (.75,.25) {$\epsilon=5\%$};
      \end{scope}
    \end{tikzpicture}\\
    \begin{tikzpicture}
      \node[anchor=south west,inner sep=0] (plot) at (0,0) {\includegraphics[width=33mm]{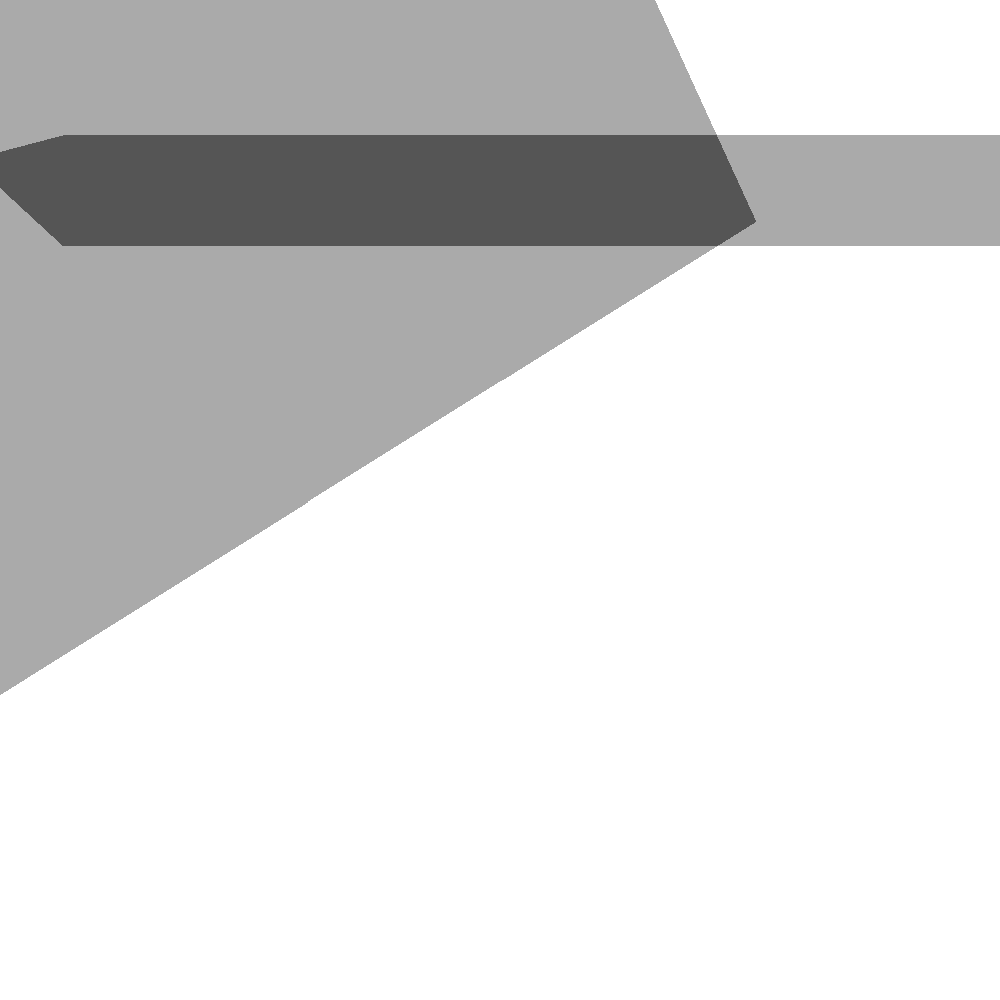}};
      \begin{scope}[x={(plot.south east)},y={(plot.north west)}]
        \draw (0,0) rectangle (1,1) ;
        \node[below] at (0,0) {$0$}; 
        \node[below] at (1,0) {$1$};
        \node[below=5pt] at (0.5,0) {$p$}; 
        \node[left] at (0,0) {$0$}; 
        \node[left] at (0,1) {$1$}; 
        \node[left=5pt] at (0,0.5) {$q$}; 
        \node[comment] at (.75,.25) {$\epsilon=10\%$};
      \end{scope}
    \end{tikzpicture}
    \begin{tikzpicture}
      \node[anchor=south west,inner sep=0] (plot) at (0,0) {\includegraphics[width=33mm]{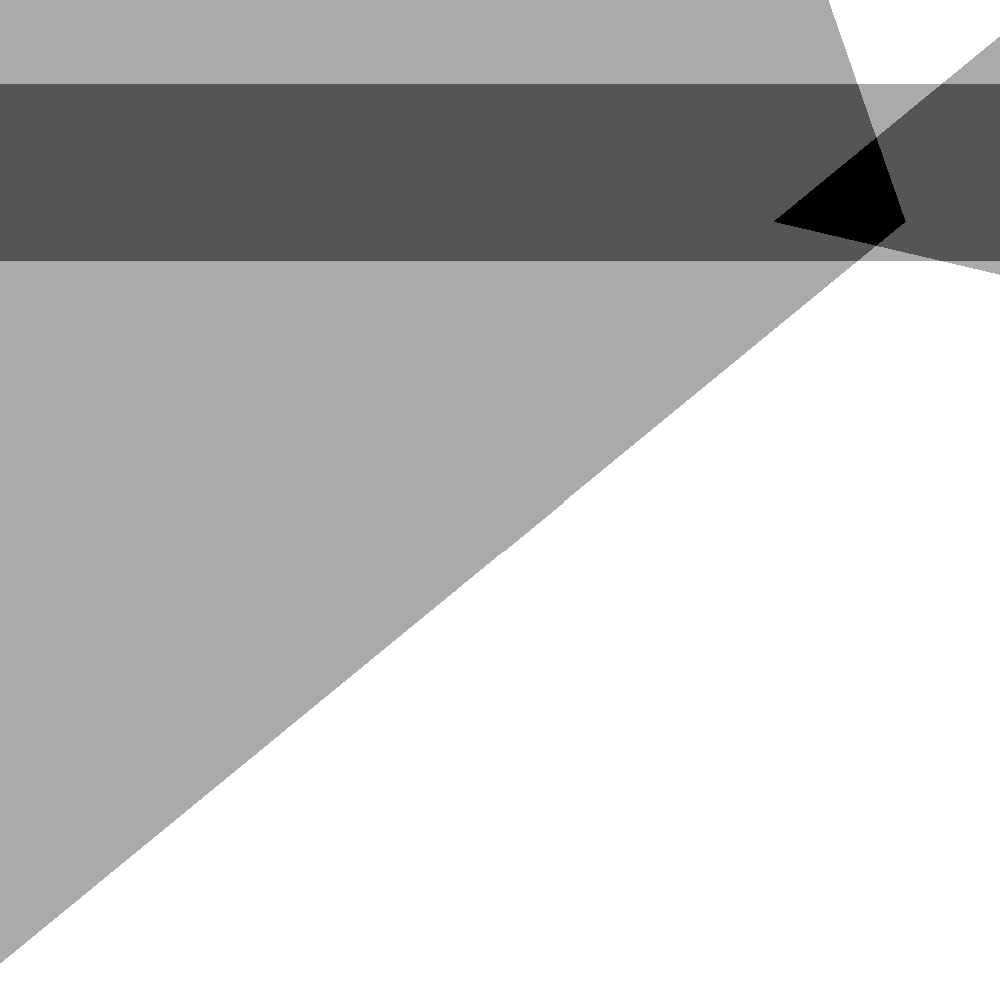}};
      \begin{scope}[x={(plot.south east)},y={(plot.north west)}]
        \draw (0,0) rectangle (1,1) ;
        \node[below] at (0,0) {$0$}; 
        \node[below] at (1,0) {$1$};
        \node[below=5pt] at (0.5,0) {$p$}; 
        \node[left] at (0,0) {$0$}; 
        \node[left] at (0,1) {$1$}; 
        \node[left=5pt] at (0,0.5) {$q$}; 
        \node[comment] at (.75,.25) {$\epsilon=15\%$};
      \end{scope}
    \end{tikzpicture}
    % \end{flushright}
  \end{minipage}
\end{wrapfigure}
The number of maximal state sequences clearly depends on the transition probabilities $p$ and $q$. 
In the rather large parts of `probability space' that are coloured white, we get a single maximal sequence---as we would for HMMs---, but there are contiguous regions where we see a higher number appear. 
In the present example (binary chain of length two), the highest possible number of maximal sequences is of course four. 
In the dark grey area, there are three maximal sequences, and two in the light grey regions. 
The plots show what happens when we let $\epsilon$ increase: the grey areas expand and the number of maximal sequences increases. 
For $\epsilon=15\%$, we even find a small area (coloured black) where all four possible state sequences are maximal: locally, due to the relatively high imprecision of our local models, we cannot give any useful robust estimates of the state sequence producing the output sequence $\fromtoout{1}{2}=01$.
\par
For small $\epsilon$, the areas with more than one maximal state sequence are quite small and seem to resemble strips that narrow down to lines as $\epsilon$ tends to zero. 
This suggests that we should be able to explain at least qualitatively where these areas come from by looking at compatible precise models: the regions where an iHMM produces different maximal (mutually incomparable) sequences, are widened versions of loci of indifference for precise HMMs. 
\par
By a \emph{locus of indifference}, we mean the set of $(p,q)$ that correspond to two given state sequences $\fromtoxstate{1}{2}$ and $\fromtoxhatstate{1}{2}$ having equal posterior probability:
\begin{equation*}
  p(\fromtoxstate{1}{2}\vert\fromtoout{1}{2})=p(\fromtoxhatstate{1}{2}\vert\fromtoout{1}{2}),
\end{equation*}
or, provided that $p(\fromtoout{1}{2})>0$,
\begin{equation*}
  p(\fromtoxstate{1}{2},\fromtoout{1}{2})=p(\fromtoxhatstate{1}{2},\fromtoout{1}{2}).
\end{equation*}
In our example where $\fromtoout{1}{2}=01$, we find the following expressions for each of the four possible state sequences:
\begin{align*}
  p(00,01)&=mr(1-r)p\\
  p(01,01)&=mr(1-s)(1-p)\\
  p(10,01)&=(1-m)s(1-r)q\\
  p(11,01)&=(1-m)s(1-s)(1-q)
\end{align*}
By equating any two of these expressions, we express that the corresponding two state sequences have an equal posterior probability. 
Since the resulting equations are a function of $p$ and $q$ only, each of these six possible combinations defines a locus of indifference.
All of them are depicted as lines in the following figure.

\begin{wrapfigure}[20]{r}{0pt}
  \begin{tikzpicture}[scale=6.9]
    \draw (0,0) rectangle (1,1);
    \node (hulp) at (0,1.05) {};
    \node[below] at (0,0) {$0$}; 
    \node[below] at (1,0) {$1$};
    \node[below=10pt] at (0.5,0) {$p$}; 
    \node[left] at (0,0) {$0$}; 
    \node[left] at (0,1) {$1$}; 
    \node[left=10pt] at (0,0.5) {$q$}; 
    \draw[thin] (7/9,0) -- (7/9,1) node[sloped,pos=.5,fill=white,rotate=180,scale=1.25] {\tiny$00-01$};
    \draw[thin] (0,0) -- (1,8/27)  node[sloped,pos=.5,fill=white,scale=1.25] {\tiny$00-10$};
    \draw[thin] (0,1) -- (1,173/189) node[sloped,pos=.5,fill=white,scale=1.25] {\tiny$00-11$};
    \draw[thin] (1,0) -- (1/28,1) node[sloped,pos=.5,fill=white,scale=1.25] {\tiny$01-10$};
    \draw[thin] (1,1) -- (0,19/27) node[sloped,pos=.5,fill=white,scale=1.25] {\tiny$01-11$};
    \draw[thin] (0,7/9) -- (1,7/9) node[sloped,pos=.1,fill=white,scale=1.25] {\tiny$10-11$};
    \draw[blue,very thick] (1/4,7/9) -- (1,7/9);
    \draw[blue,very thick] (1/4,7/9) -- (1/28,1);
    \draw[blue,very thick] (1/4,7/9) -- (0,19/27);
    \node[circle,fill=blue!20,inner sep=2.5pt] at (1/3,1/3) {$11$};
    \node[circle,fill=blue!20,inner sep=2.5pt] at (1/3,8/9) {$10$};
    \node[circle,fill=blue!20,inner sep=2.5pt] at (1/15,13/15) {$01$};
  \end{tikzpicture}
\end{wrapfigure}
Parts of these loci, depicted in blue (darker and bolder in monochrome versions of this paper) demarcate the three regions where the state sequences $01$, $10$ and $11$ are optimal (have the highest posterior probability).
\par
What happens when the transition models become imprecise? 
Roughly speaking, nearby values of the original $p$ and $q$ enter the picture, effectively turning the loci (lines) of indifference into bands of incomparability: the emergence of regions with two and more maximal sequences can be seen to originate from the loci of indifference; compare the figure for these loci with the heat plots given above.

\subsection{Extending the argument to a chain of length three}
For a chain of length three, we can determine the loci of indifference for precise models in a completely analogous manner. 
If we use the same marginal model and emission model as in the previous example, the resulting lines of indifference for the output sequence $000$ look as follows.
%\begin{center}

\begin{minipage}[l]{230pt}
  \begin{tikzpicture}
    \node[anchor=south west,inner sep=0] (plot) at (0,0) {\includegraphics[width=70mm]{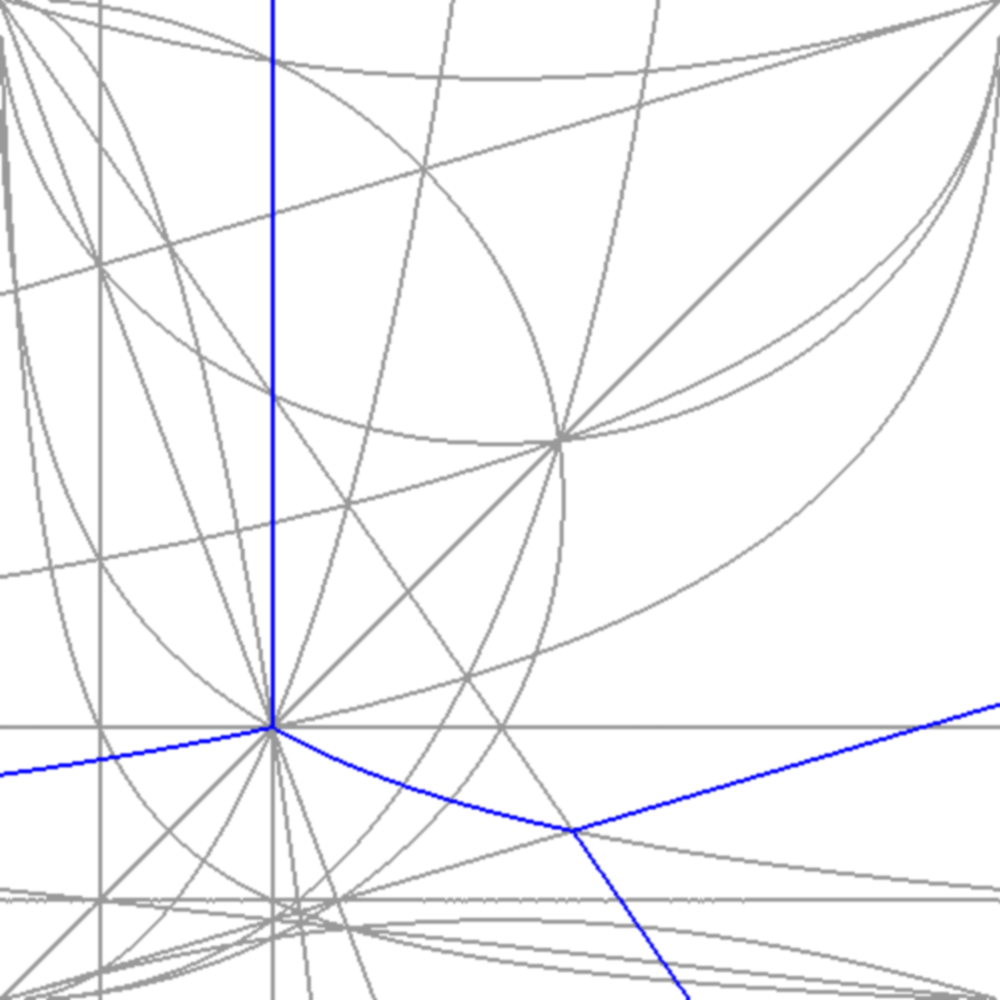}};
    \begin{scope}[x={(plot.south east)},y={(plot.north west)}]
      \draw (0,0) rectangle (1,1) ;
      \node[below] at (0,0) {$0$}; 
      \node[below] at (1,0) {$1$};
      \node[below=5pt] at (0.5,0) {$p$}; 
      \node[left] at (0,0) {$0$}; 
      \node[left] at (0,1) {$1$}; 
      \node[left=5pt] at (0,0.5) {$q$}; 
    \end{scope}
  \end{tikzpicture}
\end{minipage}
\begin{minipage}[r]{180pt}
\footnotesize
  \begin{tikzpicture}
    \node[anchor=south west,inner sep=0] (plot) at (0,0) {\includegraphics[width=33mm]{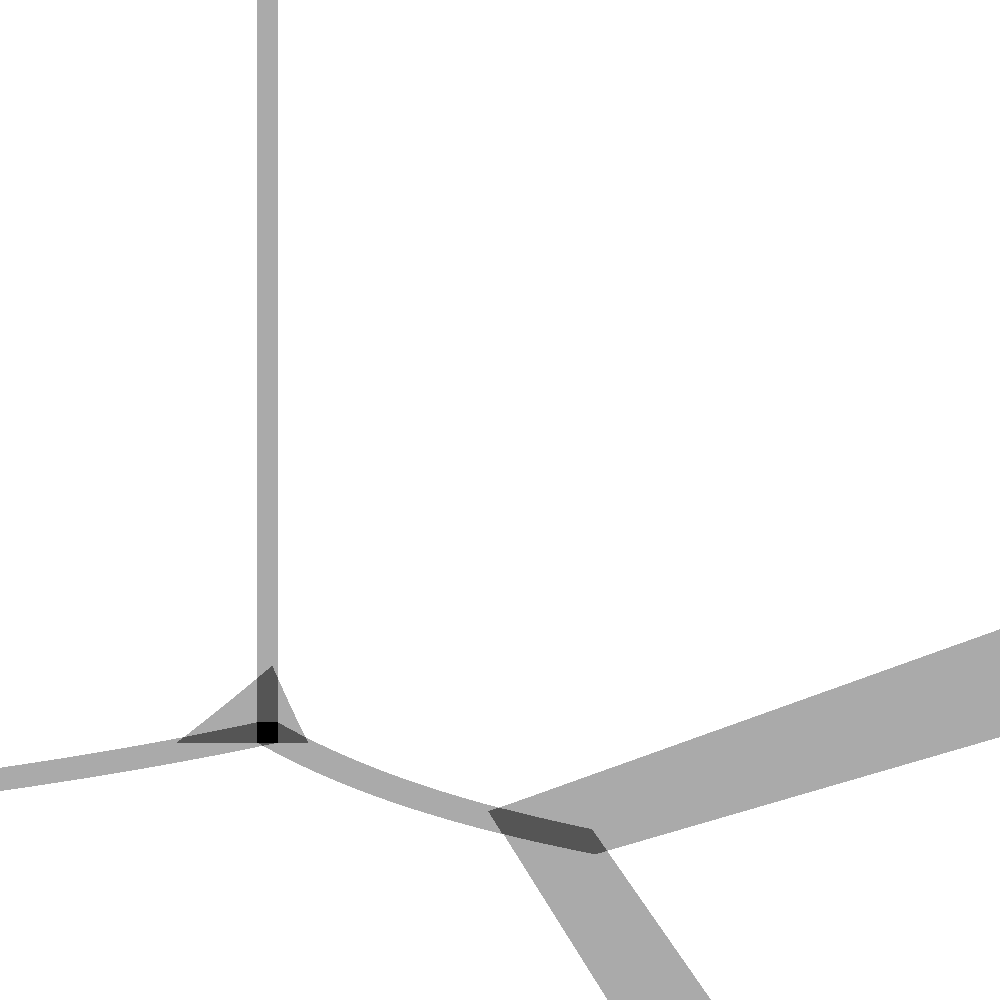}};
    \begin{scope}[x={(plot.south east)},y={(plot.north west)}]
      \draw (0,0) rectangle (1,1) ;
      \node[below] at (0,0) {$0$}; 
      \node[below] at (1,0) {$1$};
      \node[below=5pt] at (0.5,0) {$p$}; 
      \node[left] at (0,0) {$0$}; 
      \node[left] at (0,1) {$1$}; 
      \node[left=5pt] at (0,0.5) {$q$}; 
      \node[comment] at (.60,.85) {$\epsilon=2\%$};
    \end{scope}
  \end{tikzpicture}
  \begin{tikzpicture}
    \node[anchor=south west,inner sep=0] (plot) at (0,0) {\includegraphics[width=33mm]{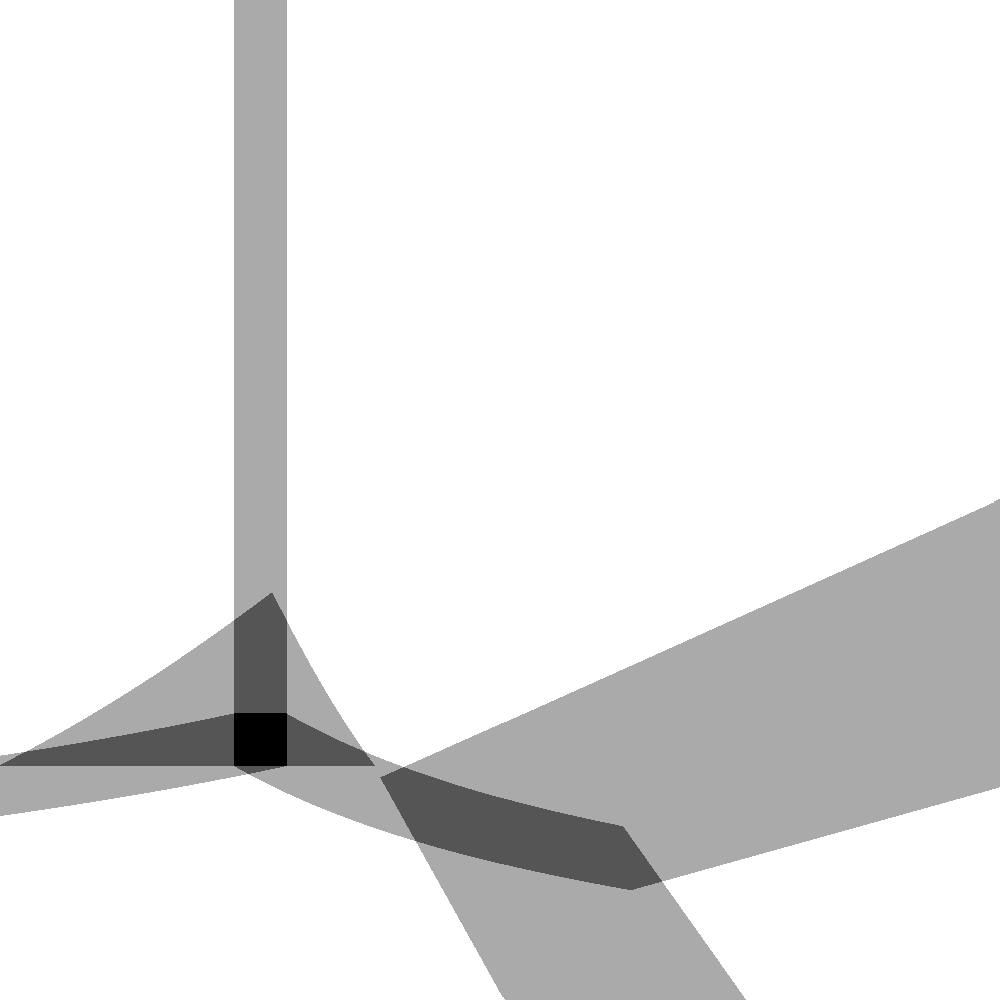}};
    \begin{scope}[x={(plot.south east)},y={(plot.north west)}]
      \draw (0,0) rectangle (1,1) ;
      \node[below] at (0,0) {$0$}; 
      \node[below] at (1,0) {$1$};
      \node[below=5pt] at (0.5,0) {$p$}; 
      \node[left] at (0,0) {$0$}; 
      \node[left] at (0,1) {$1$}; 
      \node[left=5pt] at (0,0.5) {$q$}; 
      \node[comment] at (.60,.85) {$\epsilon=5\%$};
    \end{scope}
  \end{tikzpicture}
\end{minipage}
%\end{center}

\begin{wrapfigure}[10]{r}{0pt}
\begin{minipage}[r]{243pt}
%\begin{center}
\footnotesize
  \begin{tikzpicture}
    \node[anchor=south west,inner sep=0] (plot) at (0,0) {\includegraphics[width=33mm]{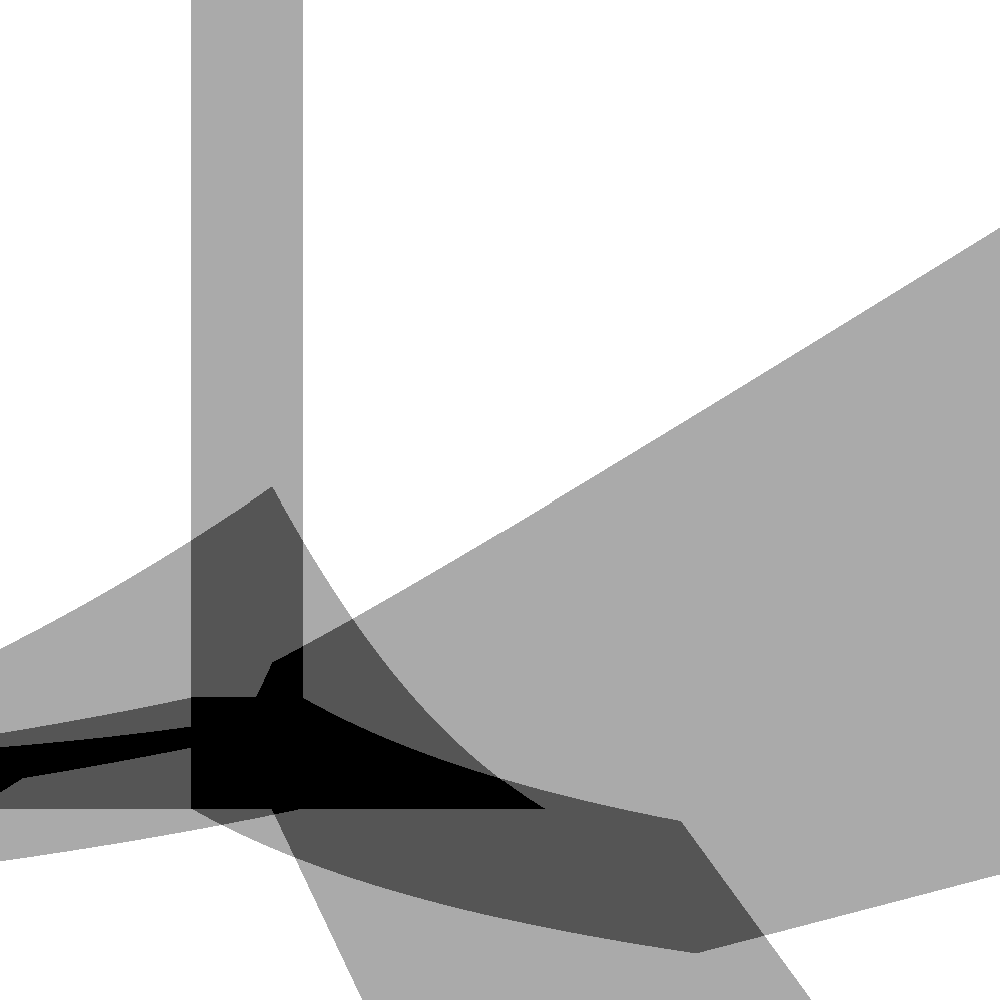}};
    \begin{scope}[x={(plot.south east)},y={(plot.north west)}]
      \draw (0,0) rectangle (1,1) ;
      \node[below] at (0,0) {$0$}; 
      \node[below] at (1,0) {$1$};
      \node[below=5pt] at (0.5,0) {$p$}; 
      \node[left] at (0,0) {$0$}; 
      \node[left] at (0,1) {$1$}; 
      \node[left=5pt] at (0,0.5) {$q$}; 
      \node[comment] at (.60,.85) {$\epsilon=10\%$};
    \end{scope}
  \end{tikzpicture}
\hspace{3mm}
  \begin{tikzpicture}
    \node[anchor=south west,inner sep=0] (plot) at (0,0) {\includegraphics[width=33mm]{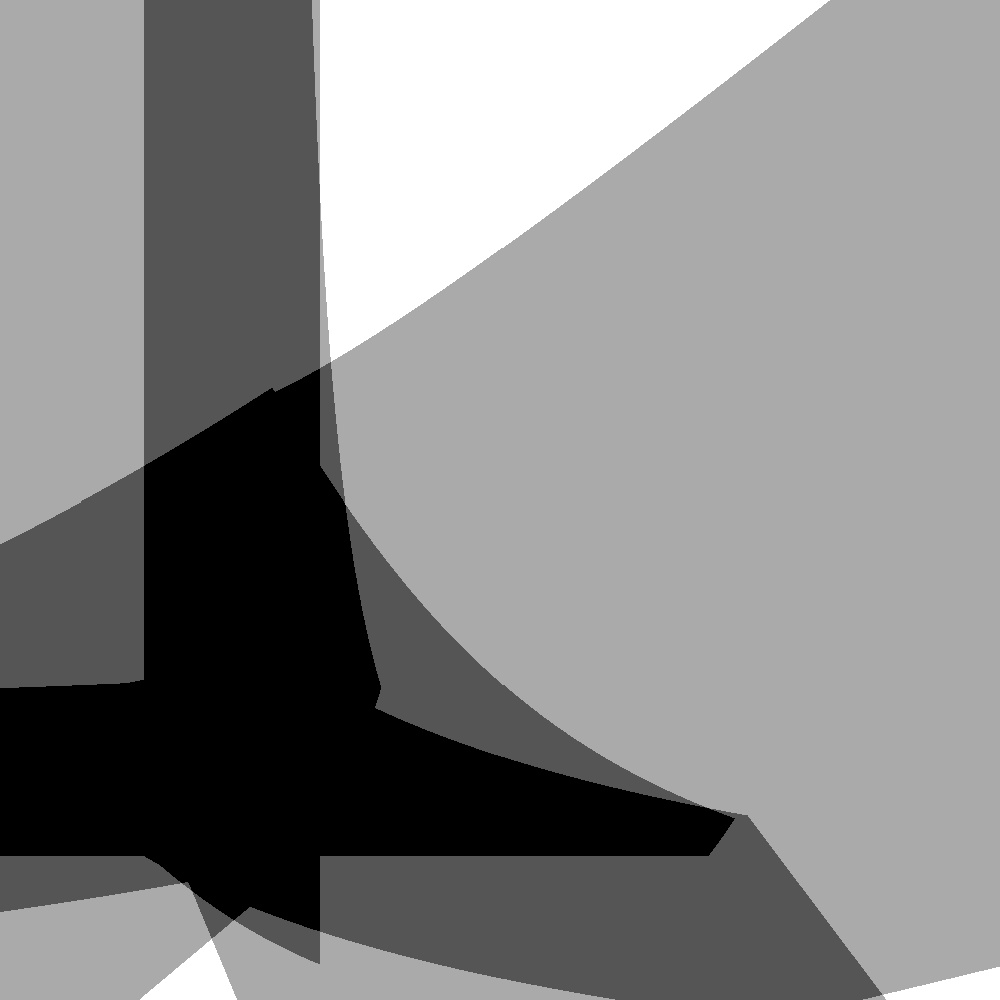}};
    \begin{scope}[x={(plot.south east)},y={(plot.north west)}]
      \draw (0,0) rectangle (1,1) ;
      \node[below] at (0,0) {$0$}; 
      \node[below] at (1,0) {$1$};
      \node[below=5pt] at (0.5,0) {$p$}; 
      \node[left] at (0,0) {$0$}; 
      \node[left] at (0,1) {$1$}; 
      \node[left=5pt] at (0,0.5) {$q$}; 
      \node[comment] at (.60,.85) {$\epsilon=15\%$};
    \end{scope}
  \end{tikzpicture}
%\end{center}
\end{minipage}
\end{wrapfigure}
\mbox{}\\
If we compare this with the visualisation of the number of
maximal elements for the same sequence, the resemblance is again quite
striking. Not that in this example, the black areas correspond to a
number of maximal sequences that is at least four.

\section{Showing off the algorithm's power}
\label{sec:example}

%*** DE FIGUUR RECHTS DIENT NOG GOED GEPLAATST TE WORDEN ALS DE LENGTE
%VAN HET GEDEELTE HIERBOVEN EENMAAL VASTLIGT ***
In order to demonstrate that our algorithm is indeed quite efficient, we let it determine the maximal sequences for a random output sequence of length $100$. 
\par
We consider the same binary stationary HMM as we presented above, but with the following precise marginal and emission probabilities:
\begin{equation*}
  \text{$m=0.1$, $r=0.98$ and $s=0.01$}.
\end{equation*}
In practical applications, the probability for an output variable to have the same value as the corresponding hidden state variable is usually quite high, which explains why we have chosen $r$ and $s$ to be close to $1$ and to $0$, respectively. 
In contrast with the previous experiments, we do not let the transition probabilities vary, but fix them to the following values:
\begin{equation*}
  p=0.6
  \text{ and }
  q=0.5.
\end{equation*}
The iHMM we use to determine the maximal sequences is then generated by mixing these precise local models with a vacuous one, using the same mixture coefficient $\epsilon$ for the marginal, transition and emission models.
In Figure~\ref{fig:100}, we display the five maximal sequences corresponding to the highlighted output sequence, and $\epsilon=2\%$. 
Since the emission probabilities were chosen to be quite accurate, it is no surprise that the output sequence itself is one of the maximal sequences. 
In addition, we have indicated in bold face the state values that differ from the outputs in the output sequence. 
We see that the model represents more indecision about the values of the state variables as we move further away from the end of the sequence. 
This is a result of a phenomenon called \emph{dilation}, which---as has been
\begin{wrapfigure}[51]{R}{0pt}
%\begin{figure}

\hspace{1mm}
 \begin{tikzpicture}[scale=3]
  \begin{scope}[node distance=10pt]
  \node (k1) [rotate=-90,fill=blue!10] at (0,5)
  {\textls
  {1110000011010000100010000111111111101111010000101101101100001101001001100110101100011011000010111001}};%observaties
  \node [rotate=-90,above of=k1,node distance=15pt] (k2)
  {\textls
  {1110000011010000100010000111111111101111010000101101101100001101001001100110101100011011000010111001}};%eps=2%
  \node [rotate=-90,above of=k2] (k3)
  {\textls
  {11{\bf0}0000011010000100010000111111111101111010000101101101100001101001001100110101100011011000010111001}};%eps=2%
  \node [rotate=-90,above of=k3] (k4)
  {\textls
  {11100000{\bf0}1010000100010000111111111101111010000101101101100001101001001100110101100011011000010111001}};%eps=2%
  \node [rotate=-90,above of=k4] (k5)
  {\textls
  {11100000110{\bf0}0000100010000111111111101111010000101101101100001101001001100110101100011011000010111001}};%eps=2%
  \node [rotate=-90,above of=k5] (k6)
  {\textls
  {1110000011010000{\bf0}00010000111111111101111010000101101101100001101001001100110101100011011000010111001}};%eps=2%
  \end{scope}
\end{tikzpicture}
%\captionsetup{singlelinecheck=off}
\caption{}
\label{fig:100}
\end{wrapfigure}

\noindent
noted in another paper \cite{cooman2009}---tends to occur when inferences in a credal tree proceed from the leaves towards the root. 
\par
As for the efficiency of our algorithm: it took about $0.2$ seconds to calculate these $5$ maximal sequences.\footnote{Running a Python programme on a 2012 MacBookPro.} 
The reason why this could be done so fast is that the algorithm is linear in the number of solutions, which in this case is only $5$. 
If we let $\epsilon$ grow to for example $5\%$, the number of maximal sequences for the same output sequence is $764$ and these can be determined in about $32$ seconds.
This demonstrates that the complexity is indeed linear in the number of solutions and that the algorithm can efficiently calculate the maximal sequences even for long output sequences.

\section{An application in optical character recognition}
\label{sec:app}

As a first and simple toy application, we use the EstiHMM algorithm to try and detect mistakes in words. 
A written word is regarded as a hidden sequence $\fromxstate{1}$, generating an output sequence $\fromout{1}$ by artificially corrupting the word. 
In this way, we simulate observation processes that are not perfectly reliable, such as the output of an Optical Character Recognition (OCR) device. 
This leads to observed output sequences that may contain errors, which we will try and detect. 
We compare our results with those of the Viterbi algorithm and show that our algorithm offers a more robust solution. 

\subsection{Generating the HMM}
A local uncertainty model must be identified for each original and observed letter: a marginal model $\statelpr{1}$ for the first letter $\statevar{1}$ of the original word, a transition model $\stateclpr{k}{k-1}$ for the subsequent letters $\statevar{k}$, with $k\in\{2,\dots,n\}$, and an emission model $\outclpr{k}$ for the observed letters $\outvar{k}$, with $k\in\{1,\dots,n\}$. 
For the sake of simplicity, we assume stationarity, making the transition and emission models independent of~$k$.

For the identification of the local models of the iHMM, we use the imprecise Dirichlet model (IDM, \cite{walley1996b}). 
For example, for the marginal model $\statelpr{1}$, applying the IDM leads to the following simple identification:
\begin{equation*}
  \statelprone[\sing{x}]=\frac{n_x}{s+\sum_{z\in\states{}}n_z}
  \text{ and }
  \stateuprone[\sing{x}]=\frac{s+n_x}{s+\sum_{z\in\states{}}n_z},\qquad\qquad\qquad
\end{equation*}
where $n_z$ counts the words in the sample text for which the first letter $\statevar{1}=\zstate{}$ and $s$ is a (positive real) hyperparameter that expresses the degree of caution in the inferences. 
In this example, we let $s=2$.
For the transition and emission models, we can proceed similarly, by counting the transitions of one character to another, respectively in the original word or during the observation process. 
In this way we obtain lower and upper transition and emission probabilities for singletons, which, as pointed out in Section~\ref{subsec:comments}, suffice to run the algorithm. 
Note that if $s$ were chosen to be zero, the local models would become precise and the EstiHMM algorithm would reduce to the Viterbi algorithm (or a version of it that does not resolve ties arbitrarily, see Section~\ref{sec:compare-viterbi}).

For the identification of the local models in the precise HMM, we use a similar but now precise Dirichlet model approach, with a Perks's prior that has the same prior strength $s=2$.
As an example, for the precise marginal model $\statepr{1}$, this leads to the following simple identification:
\begin{equation*}
  \stateprone[\sing{x}]=\frac{\nicefrac{s}{\lvert\states{}\rvert}+n_x}{s+\sum_{z\in\states{}}n_z},
\end{equation*}
where $\lvert\states{}\rvert$ is the number of states.

\subsection{Results}
Let us first discuss a specific example of the difference between the actual results we obtained using the Viterbi and the EstiHMM algorithms, in order to illustrate an important advantage of the latter. 
OCR software has mistakenly read the Italian word QUANTO as OUANTO.
Using a precise model, the Viterbi algorithm does not correct this mistake, as it suggests that the original correct word is DUANTO.
The EstiHMM algorithm on the other hand, using an imprecise model, returns CUANTO, DUANTO, FUANTO and QUANTO as maximal (undominated) solutions, including the correct one. 
Of course we would still have to pick the correct solution out of this set of suggestions---for example by using a dictionary or a human opinion---, but by using the EstiHMM algorithm, we have managed to reduce the search space from all possible five letter words to the much smaller set of four words given above. 
Notice that the solution of the Viterbi algorithm is included in the maximal solutions EstiHMM returns. 
One can easily prove that this will always be the case.

To simulate an OCR device, we have artificially corrupted the first $200$ words of the first {\itshape canto} of Dante’s \emph{Divina Commedia}, resulting in $137$ correctly read words and $63$ words containing errors. 
We try and correct these errors using both the EstiHMM and the Viterbi algorithm, and compare both approaches. 
The results are summarised in the following table.
\begin{center}
  \begin{tabular}{llll}
    % \rowcolor{gray!10}
    & \textit{total number} & \textit{correct after OCR} & \textit{wrong after OCR}\\[.5ex]
    \textit{total number} & $200$ ($100\%$) & $137$ ($68.5\%$) & $63$ ($31.5\%$)\\[.5ex]
    \textbf{Viterbi} & & & \\
    \textit{correct solution} & $157$ ($78.5\%$)	& $132$ & $25$\\
    \textit{wrong solution} & $43$ ($21.5\%$) & $5$ & $38$\\[.5ex]
    \textbf{EstiHMM}& & & \\
    \textit{correct solution included} & $172$ ($86\%$) & $137$ & $35$\\
    \textit{correct solution not included} & $28$ ($14\%$) & $0$ & $28$\\
  \end{tabular}
\end{center}
For the Viterbi algorithm, the main conclusion is that applying it to the output of the OCR device results in a decreased number of incorrect words. 
The number of correct words rises from $68.5\%$ to $78.5\%$. 
However, the Viterbi algorithm also introduces new errors for $5$ correctly read words. 

The EstiHMM algorithm manages to suggest the original correct word as one of her solutions in $86\%$ of the cases. 
Assuming we are able to detect this correct word, the percentage of
correct words rises from $68.5\%$ to $86\%$ by applying the EstiHMM
algorithm, thereby outperforming the Viterbi algorithm by almost $10\%$. 
Secondly, we also notice that the EstiHMM algorithm has never introduced new errors in words that were already correct.

Of course, since the EstiHMM algorithm allows for multiple solutions, instead of a single one, it is no surprise that we manage to increase the amount of times we suggest the correct solution. 
This would happen even if we added random extra solutions to the solution of the Viterbi algorithm. 
Giving extra solutions can only be seen as an improvement if this is done smartly. 
To investigate this, we distinguish between the cases where the EstiHMM algorithm returns a single solution, and those where it returns multiple solutions; and look at how the Viterbi and EstiHMM algorithms compare in those two cases.

The EstiHMM algorithm returned a single solution for $155$ of the $200$ words. 
As we have already mentioned above, this single solution will always coincide with the one given by the Viterbi algorithm. 
The results for the EstiHMM (and Viterbi) algorithms are summarised in the following table. 
\begin{center}
  \begin{tabular}{llll}
    % \rowcolor{gray!10}
    \textbf{EstiHMM (single solutions)} & \textit{total number} & \textit{correct after OCR} & \textit{wrong after OCR}\\[.5ex]
    \textit{total number} & $155$ ($100\%$) & $129$ ($83.2\%$) & $26$ ($16.8\%$)\\
    \textit{single correct solution} & $134$ ($86.5\%$) & $129$ & $5$\\
    \textit{single wrong solution} & $21$ ($13.5\%$) & $0$ & $21$\\
  \end{tabular}
\end{center}
The percentage of words correctly read by the OCR software is now $83.2\%$ instead of the global $68.5\%$.
When the result of the EstiHMM algorithm is a single solution, this serves as an indication that the word we are trying to correct has a fairly high probability of already being correct. 
We also see that the eventual percentage of correct words is $86.5\%$,
which is only a slight improvement over the $83.2\%$ that were already correct before applying the algorithms.

Next, we look at the remaining $45$ words, for which the EstiHMM algorithm returns more than one maximal element.
In this case, we do see a significant difference between the results of the Viterbi and the EstiHMM algorithm, since the Viterbi algorithm never returned more than one solution.\footnote{In theory, the Viterbi algorithm can return multiple indifferent solutions, but in practice it almost never does.}
The results for both algorithms are listed in the following table.
\begin{center}
  \begin{tabular}{llll}% \rowcolor{gray!10}
    & \textit{total number} & \textit{correct after OCR} &
    \textit{wrong after OCR}\\[.5ex]
    \textit{total number} & $45$ ($100\%$) & $8$ ($17.8\%$) & $37$ ($82.2\%$)\\[.5ex]
    \textbf{EstiHMM (multiple solutions)}&&&\\
    \textit{correct solution included} & $38$ ($84.4\%$) & $8$ & $30$\\
    \textit{correct solution not included} & $7$ ($15.6\%$) & $0$ & $7$\\[.5ex]
    \textbf{Viterbi}&&&\\
    \textit{correct solution} & $23$ ($51.1\%$) & $3$ & $20$\\
    \textit{wrong solution} & $22$ ($48.9\%$) & $5$ & $17$\\
  \end{tabular}
\end{center}
A first and very important conclusion to be drawn from this table, is that EstiHMM's being indecisive serves as a rather strong indication that the word we are applying the algorithm to does indeed contain errors: when the EstiHMM algorithm returns multiple solutions, the original word has been incorrectly read by the OCR software in $82.2\%$ of cases.

A second conclusion, related to the first, is that EstiHMM's being indecisive also serves as an indication that the result returned by the Viterbi algorithm is less reliable: the percentage of correct words after applying the Viterbi algorithm has dropped to $51.1\%$, in contrast with the global percentage of $78.5\%$. 
The EstiHMM algorithm, however, still gives the correct word as one of
its solutions in $84.4\%$ of cases, which is almost as high as its
global percentage of $86\%$.
If the set given by the EstiHMM algorithm contains the correct solution, the Viterbi algorithm manages to
pick this correct solution out of the set in $60.5\%$ of cases.
We see that the EstiHMM algorithm seems to notice that we are dealing with more difficult words and therefore gives us multiple solutions, between which it cannot decide.

We conclude from this experiment that EstiHMM can be usefully applied to make the results of the Viterbi algorithm more robust, and to gain an appreciation of where it is likely to go wrong. 
If the EstiHMM algorithm returns multiple solutions, this serves as an indication for robustness issues that would occur if we solved the same problem with the Viterbi algorithm. 
In that case, EstiHMM returns multiple solutions, between which it cannot decide, whereas the Viterbi algorithm will pick one out of this set in a fairly arbitrary way---depending on the choice of the prior---, thereby increasing the amount of errors made. 
The advantage of our method is that it detects such robustness issues, leaving us with the option of solving them in different ways. 
A first method would be to pick the correct word out of the set of possible solutions in some non-arbitrary way.
For the current application this could be done using a dictionary or a human expert. 
Another method for dealing with robustness issues would be to conclude that we need more data in order to build a better model, less sensitive to the choice of the prior. 
After applying the EstiHMM algorithm again, using the new model, we could check whether the robustness issues have been satisfactorily dealt with.

\section{Conclusions}
\label{sec:conclusion}
Interpreting the graphical structure of an imprecise hidden Markov model as a credal network under epistemic irrelevance leads to an efficient algorithm for finding the maximal (undominated) state sequences for a given output sequence. 
Preliminary simulations show that, even for transition models with non-negligible imprecision, the number of maximal elements seems to be reasonably low in fairly large regions of parameter space, with high numbers of maximal elements concentrated in fairly small regions.
It remains to be seen whether this observation can be corroborated by a deeper theoretical analysis.

A first and simple toy application clearly shows that the EstiHMM algorithm is able to robustify the results of the Viterbi algorithm. 
Not only does it reduce the amount of wrong conclusions by giving extra possible solutions, but it does so in an intelligent manner. 
It adds extra solutions in the specific cases where the Viterbi algorithm has robustness issues, thereby also serving as an indicator of the reliability of the result given by the Viterbi algorithm. 
An interesting further avenue of research would be to compare the EstiHMM algorithm with other methods that also try to robustify the Viterbi algorithm. 
Although most of these methods start from a precise model and introduce safety rather than imprecision by for example trying to find the $k$ most probable solutions, their practical applications are similar. 
A comparison of their results with ours could therefore prove to be interesting. 
We leave this as a topic of future research.

It is not clear to us, at this point, whether ideas similar to the ones we discussed above could be used to derive similarly efficient algorithms for imprecise hidden Markov models whose graphical structure is interpreted as a credal network under strong independence \cite{cozman2000}.
This could be interesting and relevant, as the more stringent independence condition leads to joint models that are less imprecise, and therefore produce fewer maximal state sequences (although they will be contained in our solutions).

\section*{Acknowledgements}
Jasper De Bock is a Ph.D.~Fellow of the Research Foundation - Flanders (FWO) at Ghent University, and has developed the algorithm described here in the
context of his Master's thesis, in close cooperation with Gert de Cooman, who acted as his thesis supervisor. 
The present article describes the main results of this Master's thesis.
\par
Research by De Cooman has been supported by SBO project~060043 of the IWT-Vlaanderen. 
This paper has benefitted from discussions with Marco Zaffalon, Alessandro Antonucci, Alessio Benavoli, Cassio de Campos, Erik Quaeghebeur and Filip Hermans.
We are grateful to Marco Zaffalon for providing travel funds allowing us to visit IDSIA and discuss practical applications. 

\bibliographystyle{plain}
%\bibliography{general}

\begin{thebibliography}{10}

\bibitem{bellman1957}
Richard Bellman.
\newblock {\em Dynamic Programming}.
\newblock Princeton University Press, Princeton, 1957.

\bibitem{cozman2000}
Fabio~G. Cozman.
\newblock Credal networks.
\newblock {\em Artificial Intelligence}, 120:199--233, 2000.

\bibitem{campos1994}
L.~M. {d}e Campos, J.~F. Huete, and S.~Moral.
\newblock Probability intervals: a tool for uncertain reasoning.
\newblock {\em International Journal of Uncertainty, Fuzziness and
  Knowledge-Based Systems}, 2:167--196, 1994.

\bibitem{cooman2009}
Gert {d}e Cooman, Filip Hermans, Alessandro Antonucci, and Marco Zaffalon.
\newblock Epistemic irrelevance in credal nets: the case of imprecise {M}arkov
  trees.
\newblock {\em International Journal of Approximate Reasoning}, 51:1029--1052,
  2010.

\bibitem{cooman2011a}
Gert {d}e Cooman, Enrique Miranda, and Marco Zaffalon.
\newblock Independent natural extension.
\newblock {\em Artificial Intelligence}, 2010.
\newblock Accepted for publication.

\bibitem{cooman2005a}
Gert {d}e Cooman and Matthias C.~M. Troffaes.
\newblock Dynamic programming for deterministic discrete-time systems with
  uncertain gain.
\newblock {\em International Journal of Approximate Reasoning}, 39:257--278,
  2005.

\bibitem{cooman2005e}
Gert {d}e Cooman, Matthias C.~M. Troffaes, and Enrique Miranda.
\newblock $n$-{M}onotone exact functionals.
\newblock {\em Journal of Mathematical Analysis and Applications},
  347:143--156, 2008.

\bibitem{huntley2011}
Nathan Huntley and Matthias C.~M. Troffaes.
\newblock Normal form backward induction for decision trees with coherent lower
  previsions.
\newblock {\em Annals of Operations Research}, 2010.
\newblock Submitted for publication.

\bibitem{miranda2008a}
Enrique Miranda.
\newblock A survey of the theory of coherent lower previsions.
\newblock {\em International Journal of Approximate Reasoning}, 48(2):628--658,
  January 2008.

\bibitem{miranda2009a}
Enrique Miranda.
\newblock Updating coherent lower previsions on finite spaces.
\newblock {\em Fuzzy Sets and Systems}, 160(9):1286--1307, January 2009.

\bibitem{miranda2006b}
Enrique Miranda and Gert {d}e Cooman.
\newblock Marginal extension in the theory of coherent lower previsions.
\newblock {\em International Journal of Approximate Reasoning}, 46(1):188--225,
  September 2007.

\bibitem{rabiner1989}
Lawrence~R. Rabiner.
\newblock A tutorial on {HMM} and selected applications in speech recognition.
\newblock {\em Proceedings of the IEEE}, 77(2):257--286, February 1989.

\bibitem{troffaes2007}
Matthias C.~M. Troffaes.
\newblock Decision making under uncertainty using imprecise probabilities.
\newblock {\em International Journal of Approximate Reasoning}, 45(1):17--29,
  January 2007.

\bibitem{viterbi1967}
Andrew~J. Viterbi.
\newblock Error bounds for convolutional codes and an asymptotically optimum
  decoding algorithm.
\newblock {\em IEEE Transactions on Information Theory}, 13(2):260--269, 1967.

\bibitem{walley1991}
Peter Walley.
\newblock {\em Statistical Reasoning with Imprecise Probabilities}.
\newblock Chapman and Hall, London, 1991.

\bibitem{walley1996b}
Peter Walley.
\newblock Inferences from multinomial data: learning about a bag of marbles.
\newblock {\em Journal of the Royal Statistical Society, Series B}, 58:3--57,
  1996.
\newblock With discussion.

\end{thebibliography}

\appendix

\section{Proofs of main results}\label{appendix}
In this appendix, we justify the formulas~\eqref{eq:lower:state:out:mass}, \eqref{eq:upper:state:out:mass}, \eqref{eq:target:equal}, \eqref{eq:target:equal:final}, \eqref{eq:target:different}, \eqref{eq:criterion-at-k} and \eqref{eq:criterion-at-n}; and we give proofs for Proposition~\ref{prop:PenEpos} and Theorems~\ref{theorem:crucial}--\ref{theorem:essentialstep}.
We will frequently use terms such as positive, negative, decreasing and increasing. 
We therefore start by clarifying what we mean by them. 
For $x\in\reals$, we say that $x$ is \emph{positive} if $x>0$, \emph{negative} if $x<0$, \emph{non-negative} if $x\geq0$ and \emph{non-positive} if $x\leq0$.
We call a real-valued function $f$ defined on $\reals$:
\begin{enumerate}[label=\upshape(\roman*),leftmargin=*]
\item \emph{increasing} if $(\forall x,y\in\reals)(x>y\dan f(x)>f(y))$;
\item \emph{decreasing} if $(\forall x,y\in\reals)(x>y\dan f(x)<f(y))$;
\item \emph{non-decreasing} if $(\forall x,y\in\reals)(x>y\dan f(x)\geq f(y))$;
\item \emph{non-increasing} if $(\forall x,y\in\reals)(x>y\dan f(x)\leq f(y))$.
\end{enumerate}

\begin{proof}[Proof of Equation~\eqref{eq:lower:state:out:mass}] % Checked by Gert
  For all $k\in\{1,\dots,n\}$, $\zstate{k-1}\in\states{k-1}$, $\fromzstate{k}\in\fromstates{k}$ and $\fromout{k}\in\fromouts{k}$ we infer from Equation~\eqref{eq:jointrecurse} that
  \begin{align*}
    \zinjointclpr[\indfromzstate{k}\indfromout{k}]{k}{k-1}
    &=\zinstateclpr[{\indclpr[\indfromzstate{k}\indfromout{k}]{k}}]{k}{k-1}\\
    &=\biggzinstateclpr[{\sum_{\xstate{k}\in\states{k}}\indsing{\xstate{k}}
      \xinindclpr[{\indzstate{k}(\xstate{k})\indfromzstate{k+1}\indfromout{k}}]{k}}]{k}{k-1}\\
    &=\zinstateclpr[{\indsing{\zstate{k}}
      \zinindclpr[{\indfromzstate{k+1}\indfromout{k}}]{k}}]{k}{k-1}.\\
    \intertext{Since $\zinindclpr[{\indfromzstate{k+1}\indfromout{k}}]{k}\geq0$ by~\ref{C1}, we see that~\ref{C2} transforms the above into}
    &=\zinstateclpr[{\indsing{\zstate{k}}}]{k}{k-1}
    \zinindclpr[{\indfromzstate{k+1}\indfromout{k}}]{k},\\
    \intertext{which can be reformulated as}
    &=\zinstateclpr[{\indsing{\zstate{k}}}]{k}{k-1}
    \zinoutclpr[{\indout{k}}]{k}\zinjointclpr[{\indfromzstate{k+1}\indfromout{k+1}}]{k+1}{k}\\
    &=\zinstateclpr[{\sing{\zstate{k}}}]{k}{k-1}
    \zinoutclpr[{\singout{k}}]{k}\zinjointclpr[{\indfromzstate{k+1}\indfromout{k+1}}]{k+1}{k},
  \end{align*}
  if we take into account Equation~\eqref{eq:factorisation}, since $\zinjointclpr[{\indfromzstate{k+1}\indfromout{k+1}}]{k+1}{k}\geq0$ by~\ref{C1}.
  
  Repeating these steps again and again eventually yields Equation~\eqref{eq:lower:state:out:mass}: 
  \begin{equation*}
    \zinjointclpr[\indfromzstate{k}\indfromout{k}]{k}{k-1}
    =\prod_{i=k}^n\zinstateclpr[\singzstate{i}]{i}{i-1}\zinoutclpr[\singout{i}]{i}.
  \end{equation*}
  In the last step, for $k=n$, we have used the equality $\zinindclpr[\singout{n}]{n}=\zinoutclpr[\singout{n}]{n}$, which follows from Equation~\eqref{eq:indrecurse}.
\end{proof}

\begin{proof}[Proof of Equation~\eqref{eq:upper:state:out:mass}] % Checked by Gert
  For all $k\in\{1,\dots,n\}$, $\zstate{k-1}\in\states{k-1}$, $\fromzstate{k}\in\fromstates{k}$ and $\fromout{k}\in\fromouts{k}$ we infer from conjugacy and Equation~\eqref{eq:jointrecurse} that
  \begin{align*}
    \zinjointcupr[\indfromzstate{k}\indfromout{k}]{k}{k-1}
    &=-\zinjointclpr[-\indfromzstate{k}\indfromout{k}]{k}{k-1}\\
    &=-\zinstateclpr[{\indclpr[-\indfromzstate{k}\indfromout{k}]{k}}]{k}{k-1}\\
    &=-\biggzinstateclpr[{\sum_{\xstate{k}\in\states{k}}\indsing{\xstate{k}}
      \xinindclpr[{-\indzstate{k}(\xstate{k})\indfromzstate{k+1}\indfromout{k}}]{k}}]{k}{k-1}\\
    &=-\zinstateclpr[{\indsing{\zstate{k}}
      \zinindclpr[{-\indfromzstate{k+1}\indfromout{k}}]{k}}]{k}{k-1}\\
    &=-\zinstateclpr[{-\indsing{\zstate{k}}
      (-\zinindclpr[{-\indfromzstate{k+1}\indfromout{k}}]{k})})]{k}{k-1}.\\
    \intertext{Since $-\zinindclpr[{-\indfromzstate{k+1}\indfromout{k}}]{k}=\zinindcupr[{\indfromzstate{k+1}\indfromout{k}}]{k}\geq0$ by conjugacy and Lemma~\ref{lemma:basiscoherent}, we see that \ref{C2} and Equation~\eqref{eq:toegevoegdheid} transform the above into}
    &=-\big(-\zinindclpr[{-\indfromzstate{k+1}\indfromout{k}}]{k}\big)
    \zinstateclpr[{-\indsing{\zstate{k}}}]{k}{k-1}\\
    &=-\zinstatecupr[{\indsing{\zstate{k}}}]{k}{k-1}
    \zinindclpr[{-\indfromzstate{k+1}\indfromout{k}}]{k},\\
    \intertext{which can be reformulated as}
    &=-\zinstatecupr[{\indsing{\zstate{k}}}]{k}{k-1}
    \zinoutcupr[{\indout{k}}]{k}\zinjointclpr[{-\indfromzstate{k+1}\indfromout{k+1}}]{k+1}{k}\\
    &=\zinstatecupr[{\indsing{\zstate{k}}}]{k}{k-1}
    \zinoutcupr[{\indout{k}}]{k}\zinjointcupr[{\indfromzstate{k+1}\indfromout{k+1}}]{k+1}{k}\\
    &=\zinstatecupr[{\sing{\zstate{k}}}]{k}{k-1}
    \zinoutcupr[{\singout{k}}]{k}\zinjointcupr[{\indfromzstate{k+1}\indfromout{k+1}}]{k+1}{k},
  \end{align*}
   using conjugacy and Equation~\eqref{eq:factorisation}, since $\zinjointclpr[{-\indfromzstate{k+1}\indfromout{k+1}}]{k+1}{k}\leq0$. 
  This last inequality is true because we know that $\zinjointclpr[{-\indfromzstate{k+1}\indfromout{k+1}}]{k+1}{k}=-\zinjointcupr[{\indfromzstate{k+1}\indfromout{k+1}}]{k+1}{k}$ by conjugacy and that $\zinjointcupr[{\indfromzstate{k+1}\indfromout{k+1}}]{k+1}{k}\geq 0$ by Lemma~\ref{lemma:basiscoherent}.
  
  Repeating the steps above again and again, eventually yields Equation~\eqref{eq:upper:state:out:mass}:
  \begin{equation*}
    \zinjointcupr[\indfromzstate{k}\indfromout{k}]{k}{k-1}
    =\prod_{i=k}^n\zinstatecupr[\singzstate{i}]{i}{i-1}\zinoutcupr[\singout{i}]{i}.
  \end{equation*}
  In the last step, for $k=n$, we have used the equality $\zinindcupr[\singout{n}]{n}=\zinoutcupr[\singout{n}]{n}$, which follows from Equation~\eqref{eq:indrecurse} and conjugacy.
\end{proof}

\begin{lemma}\label{lemma:basiscoherent}
  Consider a coherent lower prevision $\jointlpr{}$ on $\stategambles{}$.
  Then $\min{f}\leq\jointlpr{}(f)\leq\jointupr{}(f)\leq\max{f}$ for all $f\in\stategambles{}$ and $\jointlpr{}(f)=\jointupr{}(\mu)=\mu$ for all $\mu\in\reals$.  
\end{lemma}

\begin{proof}%[Proof of Lemma~\ref{lemma:basiscoherent}]  % Checked by Gert
  We prove the inequalities in $\min{f}\leq\jointlpr{}(f)\leq\jointupr{}(f)\leq\max{f}$ one by one. 
  The first one is the same as \ref{C1}. 
  It follows by~\ref{C3} that $\jointlpr{}(f-f)\geq\jointlpr{}(f)+\jointlpr{}(-f)$ and, since we know by \ref{C2} that $\jointlpr{}(0)=0\jointlpr{}(0)=0$, this implies that $\jointlpr{}(f)\leq-\jointlpr{}(-f)=\jointupr{}(f)$, using conjugacy for the last equality. 
  For the gamble $-f$, \ref{C1} yields that $\min{-f}\leq\jointlpr{}(-f)$ which implies that $\max{f}=-\min{-f}\geq-\jointlpr{}(-f)=\jointupr{}(f)$.

  To conclude, $\jointlpr{}(f)=\jointupr{}(\mu)=\mu$ follows by applying these inequalities for $f=\mu$.
\end{proof}

\begin{proof}[Proof of Proposition~\ref{prop:PenEpos}] % Checked by Gert
  Observe that
  \begin{equation*}
    \xinjointcupr[\indfromout{k}]{k}{k-1}
    =\biggxinjointcupr[\indfromout{k}\sum_{\fromzstate{k}\in\states{k:n}}\indfromzstate{k}]{k}{k-1}
    \geq\biggxinjointcupr[\indfromout{k}\indsing{\fromzstate{k}^*}]{k}{k-1}
    >0,
  \end{equation*}
  where $\fromzstate{k}^*$ is any element of $\states{k:n}$. 
  The equality follows from $\sum_{\fromzstate{k}\in\states{k:n}}\indfromzstate{k}=1$, the first inequality from Lemma~\ref{lemma:puntsgewijs}\ref{eig:coh:nietstriktpuntsgewijsdalend}, and the second one from the positivity assumption~\eqref{eq:assumption} and Equation~\eqref{eq:upper:state:out:mass}.

  In the same way, we can easily prove that
  \begin{equation*}
    \xinindcupr[\singfromout{k}]{k}
    =\biggxinindcupr[\indfromout{k}\sum_{\fromzstate{k+1}\in\states{k+1:n}}\indfromzstate{k+1}]{k}
    \geq\biggxinindcupr[\indfromout{k}\indsing{\fromzstate{k+1}^*}]{k}>0.
  \end{equation*}
  This time, we have used the positivity assumption~\eqref{eq:assumption} and Equation~\eqref{eq:Eboven} for the last inequality.
\end{proof}

\begin{proof}[Proof of Theorem~\ref{theorem:crucial}] % Checked by Gert
  Consider the function $\rho$ defined by $\rho(\mu)\coloneqq\jointlpr{}(\indfromout{1}[\indfromxstate{1}-\indfromxhatstate{1}-\mu])$ for all real $\mu$.
  It follows from Equation~\eqref{eq:GBR} that $\jointlpr{}(\indfromxstate{1}-\indfromxhatstate{1}\vert\fromout{1})$ is $\rho$'s rightmost zero, and we also know that $\rho(0)=\jointlpr{}(\indfromout{1}[\indfromxstate{1}-\indfromxhatstate{1}])$. 
  $\rho$ is non-increasing and continuous by Lemma \ref{lemma:rho}\ref{eig:nscc}, and has at least one zero by Lemma~\ref{lemma:rho}\ref{eig:zekereennulpunt}. 
  Hence, if $\rho(0)>0$, then $\rho$ has at least one positive zero and $\jointlpr{}(\indfromxstate{1}-\indfromxhatstate{1}\vert\fromout{1})>0$. 
  If $\rho(0)<0$, then $\rho$ has only negative zeroes and $\jointlpr{}(\indfromxstate{1}-\indfromxhatstate{1}\vert\fromout{1})<0$. 
  Hence, proving the theorem comes down to proving that $\rho(0)=0$ implies that $\rho(\epsilon)<0$ for all $\epsilon>0$, since this in turn implies that $\jointlpr{}(\indfromxstate{1}-\indfromxhatstate{1}\vert\fromout{1})=0$. 
  We now prove this implication. 
  We consider two different cases.

  \emph{The case $\xstate{1}=\xhatstate{1}$}.
  For any real $\epsilon>0$:
  \begin{align}\label{eq:bewijsnuleerstegeval}
    \rho(\epsilon)
    &=\jointlpr{}(\indfromout{1}[\indfromxstate{1}-\indfromxhatstate{1}-\epsilon])\notag\\
    &=\statelpr{1}({\indclpr[{\indfromout{1}[\indfromxstate{1}-\indfromxhatstate{1}-\epsilon]}]{1}})\notag\\
    &=\statelpr{1}\bigg(\indxstate{1}\xinindclpr[{\indfromout{1}[\indfromxstate{2}-\indfromxhatstate{2}-\epsilon]}]{1}
    +\sum_{\zstate{1}\neq\xstate{1}}\indzstate{1}\zinindclpr[{-\epsilon\indfromout{1}}]{1}\bigg).
  \end{align}
  The coefficients $\zinindclpr[{-\epsilon\indfromout{1}}]{1}$ can be written as $-\epsilon\zinindcupr[{\singfromout{1}}]{1}$ by conjugacy and \ref{C2}, which makes them negative, decreasing functions of $\epsilon$, since $\zinindcupr[{\singfromout{1}}]{1}>0$ by the positivity assumption~\eqref{eq:assumption} and Proposition~\ref{prop:PenEpos}.
  
  For the coefficient $\xinindclpr[{\indfromout{1}[\indfromxstate{2}-\indfromxhatstate{2}-\epsilon]}]{1}$, we consider two possible cases.

  If $\xinindclpr[{\indfromout{1}[\indfromxstate{2}-\indfromxhatstate{2}]}]{1}>0$, we know that $\xinindclpr[{\indfromout{1}[\indfromxstate{2}-\indfromxhatstate{2}-\epsilon]}]{1}$ is a decreasing function of $\epsilon$ by Lemma \ref{lemma:rho}\ref{eig:dalendalspositief}. 
  Therefore, the argument of $\statelpr{1}$ in Equation~\eqref{eq:bewijsnuleerstegeval} decreases pointwise in $\epsilon$, which by Lemma~\ref{lemma:puntsgewijs}\ref{eig:coh:puntsgewijsdalend} implies that $\rho(\epsilon)$ is a decreasing function of $\epsilon$ and therefore $\rho(\epsilon)<\rho(0)=0$. 
  
  If, on the other hand, $\xinindclpr[{\indfromout{1}[\indfromxstate{2}-\indfromxhatstate{2}]}]{1}\leq0$, we know by Lemma~\ref{lemma:puntsgewijs}\ref{eig:coh:nietstriktpuntsgewijsdalend} that $\xinindclpr[{\indfromout{1}[\indfromxstate{2}-\indfromxhatstate{2}-\epsilon]}]{1}\leq0$, implying that
  \begin{align*}
    \rho(\epsilon)
    &\leq\statelpr{1}\bigg(\sum_{\zstate{1}\neq\xstate{1}}\indzstate{1}\zinindclpr[{-\epsilon\indfromout{1}}]{1}\bigg)\\
    &\leq\statelpr{1}\left(\indzstate{1*}\underline{E}_{1}(-\epsilon\indfromout{1}\vert\zstate{1*})\right)
    =-\epsilon\overline{E}_{1}(\singfromout{1}\vert\zstate{1*})\stateupr{1}\singzstate{1*}<0.
  \end{align*}
  In this expression, $\zstate{1*}$ is an arbitrary $\zstate{1}\neq\xstate{1}$. 
  The first two inequalities are due to Lemma~\ref{lemma:puntsgewijs}\ref{eig:coh:nietstriktpuntsgewijsdalend}. 
  Conjugacy and~\ref{C2} yield the equality and the last inequality is a consequence of the positivity assumption~\eqref{eq:assumption} and Proposition~\ref{prop:PenEpos}. 
  Also in this case, therefore, we find that $\rho(\epsilon)<0$.
  
  \emph{The case $\xstate{1}\neq\xhatstate{1}$}.
  For any real $\epsilon>0$:
  \begin{align}\label{eq:bewijsnultweedegeval}
    \rho(\epsilon)
    &=\jointlpr{}({\indfromout{1}[\indfromxstate{1}-\indfromxhatstate{1}-\epsilon]})\notag\\
    &=\statelpr{1}({\indclpr[{\indfromout{1}[\indfromxstate{1}-\indfromxhatstate{1}-\epsilon]}]{1}})\notag\\
    &=\statelpr{1}\bigg(\indxstate{1}\xinindclpr[{\indfromout{1}[\indfromxstate{2}-\epsilon]}]{1}
    +\indxhatstate{1}\xhatinindclpr[{\indfromout{1}[-\indfromxhatstate{2}-\epsilon]}]{1}\notag\\
    &\qquad\qquad\qquad\qquad\qquad\qquad\qquad\qquad
    +\sum_{\zstate{1}\neq\xstate{1},\/\xhatstate{1}}\indzstate{1}\zinindclpr[{-\epsilon\indfromout{1}}]{1}\bigg)
  \end{align}
  In the proof for the case $\xstate{1}=\xhatstate{1}$, we have already shown that the coefficients $\zinindclpr[{-\epsilon\indfromout{1}}]{1}$ are negative, decreasing functions of $\epsilon$. 
  Together with Lemma~\ref{lemma:puntsgewijs}\ref{eig:coh:nietstriktpuntsgewijsdalend}, this implies that $\xhatinindclpr[{\indfromout{1}[-\indfromxhatstate{2}-\epsilon]}]{1}\leq\xhatinindclpr[{-\epsilon\indfromout{1}}]{1}<0$, which in turn by Lemma~\ref{lemma:rho}\ref{eig:dalendalsnegatief} implies that $\xhatinindclpr[{\indfromout{1}[-\indfromxhatstate{2}-\epsilon]}]{1}$ is a decreasing function of $\epsilon$.
  All that is left to consider is the coefficient $\xinindclpr[{\indfromout{1}[\indfromxstate{2}-\epsilon]}]{1}$. 
  There are two possibilities.
  
  If $\xinindclpr[{\indfromout{1}\indfromxstate{2}}]{1}>0$, then Lemma~\ref{lemma:rho}\ref{eig:dalendalspositief} implies that   $\xinindclpr[{\indfromout{1}[\indfromxstate{2}-\epsilon]}]{1}$ is a decreasing function of~$\epsilon$. 
  Therefore, the argument of $\statelpr{1}$ in Equation~\eqref{eq:bewijsnultweedegeval} decreases pointwise in~$\epsilon$, which by Lemma~\ref{lemma:puntsgewijs}\ref{eig:coh:puntsgewijsdalend} implies that $\rho(\epsilon)$ is a decreasing function of $\epsilon$ and therefore $\rho(\epsilon)<\rho(0)=0$. 
  
  If, on the other hand, $\xinindclpr[{\indfromout{1}\indfromxstate{2}}]{1}=0$, then we know that $\xinindclpr[{\indfromout{1}[\indfromxstate{2}-\epsilon]}]{1}\leq0$ by Lemma~\ref{lemma:puntsgewijs}\ref{eig:coh:nietstriktpuntsgewijsdalend}, implying that
  \begin{align*}
    \rho(\epsilon)
    &\leq\statelpr{1}(\indxhatstate{1}\xhatinindclpr[{\indfromout{1}[-\indfromxhatstate{2}-\epsilon]}]{1})\\
    &\leq\statelpr{1}(\indxhatstate{1}\xhatinindclpr[{-\epsilon\indfromout{1}}]{1})
    =-\epsilon\overline{E}_{1}(\singfromout{1}\vert\xhatstate{1})\stateupr{1}(\singxhatstate{1})<0.
  \end{align*}
  The first two inequalities follow from Lemma~\ref{lemma:puntsgewijs}\ref{eig:coh:nietstriktpuntsgewijsdalend}. 
  Conjugacy and~\ref{C2} yield the equality, and the last inequality is a consequence of the positivity assumption~\eqref{eq:assumption} and Proposition~\ref{prop:PenEpos}. 
  Also in this case, then, we find that $\rho(\epsilon)<0$.
\end{proof}

\begin{lemma}\label{lemma:rho}
  Let $\jointlpr{}$ be a coherent lower prevision on $\stategambles{}$.
  For any $f\in\stategambles{}$ and $y\in\mathcal{Y}$, consider the real-valued map $\rho$ defined on $\reals$ by $\rho(\mu)\coloneqq\jointlpr{}(\indsing{y}[f-\mu])$ for all real $\mu$.
  Then the following statements hold:
  \begin{enumerate}[leftmargin=*,label=\upshape(\roman*)]
  \item\label{eig:nscc} $\rho$ is non-increasing, concave and continuous. 
  \item\label{eig:zekereennulpunt} $\rho$ has at least one zero. 
  \item\label{eig:dalend} If $\jointlpr{}(\sing{y})>0$, then $\rho$ is decreasing and has a unique zero.
  \item\label{eig:identischnul} If $\jointupr{}(\sing{y})=0$, then $\rho$ is identically zero.
  \item\label{eig:nulendalend} If $\jointlpr{}(\sing{y})=0$ and $\jointupr{}(\sing{y})>0$, then $\rho$ is zero on $(-\infty,\jointlpr{}(f\vert y)]$, and negative and decreasing on $(\jointlpr{}(f\vert y),+\infty)$.
  \item\label{eig:dalendalspositief} If $\rho(a)>0$ for some $a$, then $\rho$ is decreasing and has a unique zero. 
  \item\label{eig:dalendalsnegatief} If $\rho$ is negative on an interval $(a,b)$, then it is also decreasing on $(a,b)$.
  \end{enumerate}
\end{lemma}

\begin{proof}%[Proof of Lemma~\ref{lemma:rho}]  % Checked by Gert
  We start by proving \ref{eig:nscc}. 
  It follows directly from
  Lemma~\ref{lemma:puntsgewijs}\ref{eig:coh:nietstriktpuntsgewijsdalend} that $\rho$ is non-increasing in $\mu$. 
  Now consider $\mu_1$ and $\mu_2$ in $\reals$ and $0\leq\lambda\leq1$. 
  $\rho$ is concave because
  \begin{align*}
    \rho(\lambda\mu_1+(1-\lambda)\mu_2)
    &=\jointlpr{}(\indsing{y}[f-(\lambda\mu_1+(1-\lambda)\mu_2)])\\
    &=\jointlpr{}(\lambda\indsing{y}[f-\mu_1]+(1-\lambda)\indsing{y}[f-\mu_2])\\
    &\geq\jointlpr{}(\lambda\indsing{y}[f-\mu_1])+\jointlpr{}((1-\lambda)\indsing{y}[f-\mu_2])\\
    &=\lambda\jointlpr{}(\indsing{y}[f-\mu_1])+(1-\lambda)\jointlpr{}(\indsing{y}[f-\mu_2])\\
    &=\lambda\rho(\mu_1)+(1-\lambda)\rho(\mu_2),
  \end{align*}
  where the inequality follows from~\ref{C3} and the subsequent step is due to~\ref{C2}.
  To prove that $\rho(\mu)$ is continuous, consider any $\mu_1$ and $\mu_2$ in $\reals$, then we see that
  \begin{align*}
    \rho(\mu_2)
    &=\jointlpr{}(\indsing{y}[f-\mu_2])
    =\jointlpr{}(\indsing{y}[f-\mu_1+(\mu_1-\mu_2)])\\
    &=\jointlpr{}(\indsing{y}[f-\mu_1]+\indsing{y}(\mu_1-\mu_2))
    \geq\jointlpr{}(\indsing{y}[f-\mu_1])+\jointlpr{}(\indsing{y}(\mu_1-\mu_2))\\
    &=\rho(\mu_1)-\jointlupr{}(\sing{y})\odot(\mu_2-\mu_1),
  \end{align*}
  where the inequality follows from~\ref{C3}, and the last equality is due to conjugacy and~\ref{C2}.
  Hence $\lvert\rho(\mu_1)-\rho(\mu_2)\rvert\leq\lvert\mu_2-\mu_1\rvert\jointupr{}(\sing{y})$, which proves that $\rho$ is Lipschitz continuous, and therefore also continuous.
  
  To prove \ref{eig:zekereennulpunt}, notice that $\rho(\min{f})=\jointlpr{}(\indsing{y}[f-\min{f}])\geq\jointlpr{}(\indsing{y}[\min{f}-\min{f}])=0$ and $\rho(\max{f})=\jointlpr{}(\indsing{y}[f-\max{f}])E\leq\jointlpr{}(\indsing{y}[\max{f}-\max{f}])=0$.
  The inequalities are a consequence of Lemma~\ref{lemma:puntsgewijs}\ref{eig:coh:nietstriktpuntsgewijsdalend}, and the last equalities follow from Lemma~\ref{lemma:basiscoherent}. 
  Since $\rho(\mu)$ is continuous, this implies the existence of a zero between $\min{f}$ and $\max{f}$.

  Property \ref{eig:dalend} can be proved by considering $\mu_1$ and $\mu_2$ in $\reals$ with $\mu_2>\mu_1$. 
  If $\jointlpr{}(\sing{y})>0$, we see that $\rho$ is decreasing, since
  \begin{align*}
    \rho(\mu_1)
    &=\jointlpr{}(\indsing{y}[f-\mu_1])
    =\jointlpr{}(\indsing{y}[f-\mu_2+(\mu_2-\mu_1)])\\
    &=\jointlpr{}(\indsing{y}[f-\mu_2]+\indsing{y}(\mu_2-\mu_1))
    \geq\jointlpr{}(\indsing{y}[f-\mu_2])+\jointlpr{}(\indsing{y}(\mu_2-\mu_1))\\
    &=\rho(\mu_2)+(\mu_2-\mu_1)\jointlpr{}(\sing{y})
    >\rho(\mu_2),
  \end{align*}
  where the first inequality follows from~\ref{C3} and the last equality from~\ref{C2}. 
  We know by~\ref{eig:zekereennulpunt} that $\rho$ has at least one zero, which must be unique because $\rho$ is decreasing.
  
  To prove \ref{eig:identischnul}, first note that $\jointupr{}(\sing{y})=0$ also implies $\jointlpr{}(\sing{y})=0$, because of Lemma~\ref{lemma:basiscoherent}. 
  Now fix $\mu$ in $\reals$ and choose $a$ and $b$ in $\reals$ such that
  \begin{equation*}
    a<\min\{0,\min\{f-\mu\}\}\leq\max\{0,\max\{f-\mu\}\}<b.
  \end{equation*}
  Then at the same time $\rho(\mu)=\jointlpr{}(\indsing{y}[f-\mu])\geq\jointlpr{}(\indsing{y}a)=a\jointupr{}(\sing{y})=0$ and $\rho(\mu)=\jointlpr{}(\indsing{y}[f-\mu])\leq\jointlpr{}(\indsing{y}b)=b\jointlpr{}(\sing{y})=0$,
  using Lemma~\ref{lemma:puntsgewijs}\ref{eig:coh:nietstriktpuntsgewijsdalend}, \ref{C2} and conjugacy. 
  We conclude that $\rho(\mu)=0$ for any $\mu$ in $\reals$.
  
  The proof of \ref{eig:nulendalend} starts by noticing that $\rho(\mu)\geq0$ for $\mu\in(-\infty,\jointlpr{}(f\vert y)]$ and $\rho(\mu)<0$ for $\mu\in(\jointlpr{}(f\vert y),+\infty)$, due to the definition of $\jointlpr{}(f\vert y)$ (see Equation~\eqref{eq:GBR}), and the fact that $\rho$ is non-increasing by~\ref{eig:nscc}.
  In the proof of~\ref{eig:identischnul}, we have already shown that $\rho$ is non-positive if $\jointlpr{}(\sing{y})=0$, which allows us to conclude that $\rho(\mu)=0$ for $\mu\in(-\infty,\jointlpr{}(f\vert y)]$. 
  We are left to prove that $\rho$ is decreasing on the interval $(\jointlpr{}(f\vert y),+\infty)$. 
  We will do so by contradiction. 
  Suppose that $\rho$ is not decreasing on that interval, then there are $\mu_1$ and $\mu_2$ in this interval, such that $\mu_2>\mu_1$ and $0>\rho(\mu_2)\geq\rho(\mu_1)$. 
  Since $\rho$ is zero on $(-\infty,\jointlpr{}(f\vert y))$, we can also choose $\mu_0<\mu_1$ such that $\rho(\mu_0)=0$. 
  The existence of such $\mu_0$, $\mu_1$ and $\mu_2$ contradicts the concavity of $\rho$, established by~\ref{eig:nscc}. 

  To prove \ref{eig:dalendalspositief}, observe that $\jointupr{}(\sing{y})\geq\jointlpr{}(\sing{y})\geq0$ by Lemma~\ref{lemma:basiscoherent}.
  This implies that the three cases considered in~\ref{eig:dalend}, \ref{eig:identischnul} and~\ref{eig:nulendalend} are exhaustive and mutually exclusive.
  If there is an $a$ for which $\rho(a)>0$, we can only have the case considered in~\ref{eig:dalend}, which implies that $\rho$ is decreasing and has a unique zero.

  It now only remains to prove~\ref{eig:dalendalsnegatief}. 
  By repeating the argument in the proof of~\ref{eig:dalendalspositief}, we see that $\rho$ is negative on an interval $(a,b)$, only the cases considered in~\ref{eig:dalend} and~\ref{eig:nulendalend} can obtain. \
  For~\ref{eig:dalend}, $\rho$ is decreasing on its entire domain. 
  For~\ref{eig:nulendalend}, $\rho$ is definitely decreasing on $(a,b)$.
\end{proof}

\begin{lemma}\label{lemma:puntsgewijs}
  Consider a coherent lower prevision $\jointlpr{}$ on $\stategambles{}$ and two gambles $f,g\in\stategambles{}$.
  \begin{enumerate}[label=\upshape(\roman*),leftmargin=*] 
  \item\label{eig:coh:puntsgewijsdalend} If $f(x)>g(x)$ for all $\xstate{}\in\states{}$, then $\jointlpr{}(f)>\jointlpr{}(g)$. 
  \item\label{eig:coh:nietstriktpuntsgewijsdalend} If $f(x)\geq g(x)$ for all $\xstate{}\in\states{}$, then $\jointlpr{}(f)\geq\jointlpr{}(g)$.  
  \end{enumerate} 
\end{lemma}

\begin{proof}% [Proof of Lemma~\ref{lemma:puntsgewijs}]  % Checked by Gert
  We start with \ref{eig:coh:puntsgewijsdalend}.
  If $f-g$ is pointwise positive, then $\min(f-g)>0$ and therefore $\jointlpr{}(f-g)\geq\min{(f-g)}>0$, using \ref{C1} for the first inequality. 
  It follows from~\ref{C3} that $\jointlpr{}(f)=\jointlpr{}((f-g)+g)\geq\jointlpr{}(f-g)+\jointlpr{}(g)$, and therefore that $\jointlpr{}(f)-\jointlpr{}(g)\geq\jointlpr{}(f-g)>0$, whence indeed $\jointlpr{}(f)>\jointlpr{}(g)$.

  The proof for \ref{eig:coh:nietstriktpuntsgewijsdalend} is analogous, but now we only have that $\min(f-g)\geq0$, implying that $\jointlpr{}(f)-\jointlpr{}(g)\geq\jointlpr{}(f-g)\geq\min{(f-g)}\geq0$.
\end{proof}

\begin{proof}[Proof of Equation~\eqref{eq:target:equal}] % Checked by Gert
  Let $\fromtarget{k}\coloneqq\indfromout{k}[\indfromxstate{k}-\indfromxhatstate{k}]$.
  Since $k\in\{1,\dots,n-1\}$ and $\xhatstate{k}=\xstate{k}$, this implies that
  \begin{align*}
    \fromtarget{k}
    &=\indfromout{k}[\indfromxstate{k}-\indfromxhatstate{k}]\\
    &=\indout{k}\indxstate{k}\indfromout{k+1}[\indfromxstate{k+1}-\indfromxhatstate{k+1}]\\
    &=\indout{k}\indxstate{k}\fromtarget{k+1},
  \end{align*}
  which in turn implies that
  \begin{align*}
    \zinjointclpr[\fromtarget{k}]{k}{k-1}
    &=\zinstateclpr[{\indclpr[{\indout{k}\indxstate{k}\fromtarget{k+1}}]{k}}]{k}{k-1}\notag\\
    &=\zinstateclpr[{\indxstate{k}\xinindclpr[{\indout{k}\fromtarget{k+1}}]{k}}]{k}{k-1}\notag\\
    &=\zinstateclupr[\singxstate{k}]{k}{k-1}\odot\xinindclpr[{\indout{k}\fromtarget{k+1}}]{k}\notag\\
    &=\zinstateclupr[\singxstate{k}]{k}{k-1}\xinoutclupr[\singout{k}]{k}\odot\xinjointclpr[\fromtarget{k+1}]{k+1}{k},
  \end{align*}
  proving Equation~\eqref{eq:target:equal}.
  The first equality follows from Equation~\eqref{eq:jointrecurse}. 
  The second equality holds because $\indxstate{k}(\zstate{k})=0$ for all $\zstate{k}\neq\xstate{k}$, implying that $\indclpr[{\indout{k}\indxstate{k}\fromtarget{k+1}}]{k}=\indxstate{k}\xinindclpr[{\indout{k}\fromtarget{k+1}}]{k}$. 
  The third equality is follows from conjugacy and~\ref{C2}, and the last one follows from Equation~\eqref{eq:factorisation}.
\end{proof}

\begin{proof}[Proof of Equation~\eqref{eq:target:equal:final}] % Checked by Gert
  Since $\xhatstate{n}=\xstate{n}$, Lemma~\ref{lemma:basiscoherent} yields:
  \begin{equation*}
    \zinjointclpr[\indout{n}[\indxstate{n}-\indxhatstate{n}]{n}{n-1}
    =\zinjointclpr[\indout{n}[\indxstate{n}-\indxstate{n}]{n}{n-1}
    =\zinjointclpr[0]{n}{n-1}=0.
    \qedhere
  \end{equation*}
\end{proof}

\begin{proof}[Proof of Equation~\eqref{eq:target:different}] % Checked by Gert
  If $k\in\{1,\dots,n\}$ and $\xhatstate{k}\neq\xstate{k}$, then
  \begin{multline*}
    \zinjointclpr[\indfromout{k}[\indfromxstate{k}-\indfromxhatstate{k}]{k}{k-1}\\
    \begin{aligned}
      &=\zinstateclpr[{\indclpr[{\indfromout{k}[\indfromxstate{k}-\indfromxhatstate{k}]}]{k}}]{k}{k-1}\\
      &=\zinstateclpr[{
        \indxstate{k}{\xinindclpr[{\indfromout{k}\indfromxstate{k+1}}]{k}}
        +\indxhatstate{k}{\xhatinindclpr[{-\indfromout{k}\indfromxhatstate{k+1}}]{k}}
      }]{k}{k-1}\\
      &=\zinstateclpr[{
        \indxstate{k}{\xinindclpr[{\indfromout{k}\indfromxstate{k+1}}]{k}}
        -\indxhatstate{k}{\xhatinindcupr[{\indfromout{k}\indfromxhatstate{k+1}}]{k}}
      }]{k}{k-1}\\
      &=\zinstateclpr[{\indxstate{k}\fromxlmem{k}-\indxhatstate{k}\fromxhatumem{k}}]{k}{k-1},
    \end{aligned}
  \end{multline*}
  proving Equation~\eqref{eq:target:different}. 
  The reasons why all these equalities hold, are analogous to the ones given in the proof of Equation~\eqref{eq:target:equal}.
\end{proof}

\begin{proof}[Proof of Theorem~\ref{theorem:optimality}] % Checked by Gert
  Fix $k\in\{1,\dots,n-1\}$, $\zstate{k-1}\in\states{k-1}$ and $\fromxhatstate{k}\in\fromstates{k}$.
  We assume that $\fromxhatstate{k+1}\notin\explfromopt[\fromstates{k+1}]{\hat{x}}{k}{k+1}$ and then show that $\fromxhatstate{k}\notin\fromopt[\fromstates{k}]{z}{k}$.
  It follows from the assumption that $\xhatinjointclpr[\indfromout{k+1}[\indfromxstate{k+1}-\indfromxhatstate{k+1}]{k+1}{k}>0$ for some $\fromxstate{k+1}\in\states{k+1}$.
  Now prefix this state sequence $\fromxstate{k+1}$ with the state $\xhatstate{k}$ to form the state sequence $\fromxstate{k}$, implying that $\xstate{k}=\xhatstate{k}$.
  We then infer from Equation~\eqref{eq:target:equal} that
  \begin{multline*}
    \zinjointclpr[\indfromout{k}[\indfromxstate{k}-\indfromxhatstate{k}]{k}{k-1}\\
    =\zinstateclpr[\singxhatstate{k}]{k}{k-1}\xhatinoutclpr[\singout{k}]{k}
    \xhatinjointclpr[\indfromout{k+1}[\indfromxstate{k+1}-\indfromxhatstate{k+1}]{k+1}{k}
    >0,
  \end{multline*}
  which tells us that indeed $\fromxhatstate{k}\notin\fromopt[\fromstates{k}]{z}{k}$.
\end{proof}

\begin{proof}[Proof of Equations~\eqref{eq:criterion-at-k} and \eqref{eq:criterion-at-n}] % Checked by Gert
  First, we consider $k=n$. 
  For every $\zstate{n-1}\in\states{n-1}$, we determine $\opt[\states{n}]{z}{n}$ as the set of those elements $\xhatstate{n}$ of $\states{n}$ for which
  \begin{equation*}\label{eq:criterion-at-n tussenresultaat1}
    (\forall\xstate{n}\in\states{n}\setminus\singxhatstate{n})
    \zinstateclpr[{\indxstate{n}\xlmemmax{n}-\indxhatstate{n}\xhatumem{n}}]{n}{n-1}\leq0,
  \end{equation*}
  as this condition is equivalent to the optimality condition~\eqref{eq:optimals} for $k=n$, taking into account Equations~\eqref{eq:target:equal:final}, \eqref{eq:target:different} and~\eqref{eq:alphabetamaxn}.
  We now show that this condition is also equivalent to 
  \begin{equation}\label{eq:criterion-at-n tussenresultaat2}
    (\forall\xstate{n}\in\states{n}\setminus\singxhatstate{n})
    \xhatumem{n}\geq\xlmemmax{n}\zinthreshold{n},
  \end{equation}
  To see this, we consider two different cases. 
  For those $\xstate{n}$ for which $\xlmemmax{n}=0$, the inequalities $\zinstateclpr[{\indxstate{n}\xlmemmax{n}-\indxhatstate{n}\xhatumem{n}}]{n}{n-1}\leq0$ and $\xhatumem{n}\geq\xlmemmax{n}\zinthreshold{n}$ are both trivially satisfied since $\xhatumem{n}=\xhatinoutcupr[\singout{n}]{n}>0$ by the positivity assumption~\eqref{eq:assumption}. 
  If $\xlmemmax{n}>0$, both inequalities are equivalent because of~\ref{C2} and Equation~\eqref{eq:threshold}:
  \begin{align*}
    \zinstateclpr[{\indxstate{n}\xlmemmax{n}-\indxhatstate{n}\xhatumem{n}}]{n}{n-1}\leq0
    &\asa\biggzinstateclpr[{\indxstate{n}-\indxhatstate{n}\frac{\xhatumem{n}}{\xlmemmax{n}}}]{n}{n-1}\leq0\\
    &\asa\frac{\xhatumem{n}}{\xlmemmax{n}}\geq\zinthreshold{n}\\
    &\asa\xhatumem{n}\geq\xlmemmax{n}\zinthreshold{n}.
  \end{align*}
  Using Equation~\eqref{eq:alphaopt}, Equation~\eqref{eq:criterion-at-n tussenresultaat2} can now be reformulated as $\xhatumem{n}\geq\umemopt{n}$, which completes the proof of the equivalence.
  \par
  Next, we consider any $k\in\{1,\dots,n-1\}$.  
  Fix $\zstate{k-1}\in\states{k-1}$, then we must determine $\fromopt[\fromstates{k}]{z}{k}$. 
  We know from the Principle of Optimality~\eqref{eq:pop1imprecies} that we can limit the candidate optimal sequences $\fromxhatstate{k}$ to the set $\frommog[\fromstates{k}]{z}{k}$. 
  Consider any such $\fromxhatstate{k}$, then we must check for any $\fromxstate{k}\in\fromstates{k}$ whether $\zinjointclpr[{\indfromout{k}[\indfromxstate{k}-\indfromxhatstate{k}]}]{k}{k-1}\leq0$; see Equation~\eqref{eq:optimals}.

  If $\fromxstate{k}$ is such that $\xstate{k}=\xhatstate{k}$, this inequality is automatically satisfied.
  Indeed, if $\xhatstate{k}\notin\pos_k(\zstate{k-1})$, then we infer from Equation~\eqref{eq:pop3imprecies} that $\zinstateclpr[\singxhatstate{k}]{k}{k-1}=0$ or $\xhatinoutclpr[\singout{k}]{k}=0$, and then Equation~\eqref{eq:target:equal} tells us that $\zinjointclpr[{\indfromout{k}[\indfromxstate{k}-\indfromxhatstate{k}]}]{k}{k-1}=0$.
  If $\xhatstate{k}\in\pos_k(\zstate{k-1})$, we know from Equation~\eqref{eq:pop2imprecies} that $\fromxhatstate{k+1}\in\explfromopt[\fromstates{k+1}]{\hat{x}}{k}{k+1}$, which implies that $\xhatinjointclpr[{\indfromout{k+1}[\indfromxstate{k+1}-\indfromxhatstate{k+1}]}]{k+1}{k}\leq0$.
  Hence $\zinjointclpr[{\indfromout{k}[\indfromxstate{k}-\indfromxhatstate{k}]}]{k}{k-1}\leq0$, again by Equation~\eqref{eq:target:equal}.

  This means we can limit ourselves to checking the inequality for those $\fromxstate{k}$ for which $\xstate{k}\neq\xhatstate{k}$.
  So fix any $\xstate{k}\neq\xhatstate{k}$, then we must check whether
  \begin{equation*}
    (\forall\fromxstate{k+1}\in\fromstates{k+1})\\
    \zinstateclpr[{\indxstate{k}\fromxlmem{k}-\indxhatstate{k}\fromxhatumem{k}}]{k}{k-1}\leq0;
  \end{equation*}
  see Equation~\eqref{eq:target:different}.
  By Equation~\eqref{eq:alphabetamax} and Lemma~\ref{lemma:puntsgewijs}, this is equivalent to
  \begin{equation*}
    \zinstateclpr[{\indxstate{k}\xlmemmax{k}-\indxhatstate{k}\fromxhatumem{k}}]{k}{k-1}\leq0,
  \end{equation*}
  which can in turn be seen to be equivalent to $\fromxhatumem{k}\geq\xlmemmax{k}\zinthreshold{k}$, using a course of reasoning completely analogous to the one used above for the case $k=n$.
  Since this inequality must hold for every $\xstate{k}\neq\xhatstate{k}$, we infer from Equation~\eqref{eq:alphaopt} that we must have that $\fromxhatumem{k}\geq\umemopt{k}$.
  So we must check this condition for all the candidate sequences $\fromxhatstate{k}$ in $\frommog[\fromstates{k}]{z}{k}$, which proves Equation~\eqref{eq:criterion-at-k}.
\end{proof}

\begin{proof}[Proof of Theorem~\ref{theorem:construction}]
  This proof consists of two parts. 
  We will first prove that every sequence $\xhatstate{k:n}$ obtained by the optimal tree construction is an element of $\fromopt[\fromstates{k}]{z}{k}$. 
  Secondly we will prove that a sequence $\zstate{k:n}$ that is not part of the set of sequences obtained by the optimal tree construction cannot be an element of $\fromopt[\fromstates{k}]{z}{k}$.

  Let us start by proving that every sequence $\xhatstate{k:n}$ obtained by the optimal tree construction is an element of $\fromopt[\fromstates{k}]{z}{k}$. 
  It follows from the last step of the optimal tree construction that every $\xhatstate{k:n}$ of the constructed set is an element of $\frommogDoor[\fromstates{k}]{z}{k}{\xhatstate{k:n}}$, and therefore by Equation~\eqref{def:doorsnede} also of $\frommog[\fromstates{k}]{z}{k}$. 
  This last step also implies that $\xhatumemmax{n}\geq\specialumemopt{k}{\xhatstate{k:n}}$, which can be seen to be equivalent with $\fromxhatumem{k}\geq\xhatumemopt{k}$, by Equation~\eqref{eq:alphabetamaxn} and the repeated use of Equations~\eqref{eq:optalgemeen} and~\eqref{eq:alpharecurs}. 
  It then follows from Equation~\eqref{eq:criterion-at-k} that $\xhatstate{k:n}$ is an element of $\fromopt[\fromstates{k}]{z}{k}$.

  To conclude, we show that a sequence $\zstate{k:n}$ that is not part of the set of sequences obtained by the optimal tree construction cannot be an element of $\fromopt[\fromstates{k}]{z}{k}$.
  If a sequence $\zstate{k:n}$ is not part of the set of sequences obtained by the optimal tree construction, this either implies that it is not an element of $\frommog[\fromstates{k}]{z}{k}$, or that there is some $s\in\{k,\dots,n\}$ for which $\zumemmax{s}<\zumemoptfromto{k}{s}$.
  In the first case, it follows directly from Equation~\eqref{eq:criterion-at-k} that $\zstate{k:n}$ cannot be an element of $\fromopt[\fromstates{k}]{z}{k}$.
  In the second case, we see that $\zumemmax{s}<\zumemoptfromto{k}{s}$ implies that $\fromzumem{k}<\zumemopt{k}$, which can be seen to be equivalent with $\fromzumem{k}<\zumemopt{k}$ by the repeated use of Equations~\eqref{eq:optalgemeen} and~\eqref{eq:alpharecurs}. 
  It then follows from Equation~\eqref{eq:criterion-at-k} that $\zstate{k:n}$ cannot be an element of $\fromopt[\fromstates{k}]{z}{k}$.
\end{proof}

\begin{proof}[Proof of Theorem~\ref{theorem:essentialstep}]
  If $s=k$, this can be proved by contradiction.
  If for all $\xstate{s}\in\states{s}$ both conditions would not be fulfilled, the optimal tree construction would stop and the set $\fromopt[\fromstates{k}]{z}{k}$ would be empty. 
  This is a contradiction since every finite partially ordered set has at least one maximal element.

  Now consider any $s\in\{k+1,\dots,n\}$. 
  Equation~\eqref{eq:alphabetamax} implies that there is at least one sequence $\xstate{s:n}^*\in\states{s:n}$ for which $\specialumem{s-1}{\xhatstate{s-1}\oplus\xstate{s:n}^*}=\xhatumemmax{s-1}$. 
  We prove that the first state $\xstate{s}^*$ of this sequence meets both criteria of the theorem.

  We know that $\xhatstate{k:s-1}$ is found using the optimal tree construction, which implies that $\frommogDoor[\fromstates{k}]{z}{k}{\xhatstate{k:s-1}}$ is a non-empty set and $\xhatumemmax{s-1}\geq\specialumemopt{k}{\xhatstate{k:s-1}}$.
  It follows from this inequality that $\specialumem{s-1}{\xhatstate{s-1}\oplus\xstate{s:n}^*}\geq\specialumemopt{k}{\xhatstate{k:s-1}}$, which can be seen to be equivalent with $\fromxumemstar{s}\geq\specialumemopt{k}{\xhatstate{k:s-1}\oplus\xstate{s}^*}$ by Equations~\eqref{eq:alpharecurs} and~\eqref{eq:optalgemeen}. 
  Since we know that $\fromxumemstar{s}=\xumemmaxstar{s}$ by Equation~\eqref{eq:alphamaxrecurs}, we find that $\xumemmaxstar{s}\geq\specialumemopt{k}{\xhatstate{k:s-1}\oplus\xstate{s}^*}$, meaning that $\xstate{s}^*$ satisfies the first criterium.

  To prove that the state $\xstate{s}^*$ also satisfies the second criterium, which means that the set $\frommogDoor[\fromstates{k}]{z}{k}{\xhatstate{k:s-1}\oplus\xstate{s}^*}$ is non-empty, it suffices by Equation~\eqref{def:doorsnede} to prove that $\xhatstate{k:s-1}\oplus\xstate{s:n}^*$ is an element of $\frommog[\fromstates{k}]{z}{k}$.

  Since $\frommogDoor[\fromstates{k}]{z}{k}{\xhatstate{k:s-1}}$ is non-empty, there is at least one \mbox{$\zstate{s:n}\in\states{s:n}$} for which $\xhatstate{k:s-1}\oplus\zstate{s:n}$ is an element of $\frommog[\fromstates{k}]{z}{k}$. 
  Furthermore, we have chosen $\fromxstate{s}^*$ such that $\specialumem{s-1}{\xhatstate{s-1}\oplus\xstate{s:n}^*}=\xhatumemmax{s-1}$.
  Lemma~\ref{stel:complexiteitbegrensd} now implies that $\xhatstate{k:s-1}\oplus\fromxstate{s}^*$ is an element of $\frommog[\fromstates{k}]{z}{k}$.
\end{proof}

\begin{lemma}\label{stel:complexiteitbegrensd}
  Fix $k\in\{1,\dots,n-1\}$, $s\in\{k+1,\dots,n\}$, $\zstate{k-1}\in\states{k-1}$ and $\xhatstate{k:s-1}\in\states{k:s-1}$. 
  Choose an arbitrary $\fromxstate{s}^*\in\states{s:n}$ for which $\specialumem{s-1}{\xhatstate{s-1}\oplus\xstate{s:n}^*}=\xhatumemmax{s-1}$. 
  If there is some $\zstate{s:n}\in\states{s:n}$ for which $\xhatstate{k:s-1}\oplus\zstate{s:n}$ belongs to $\frommog[\fromstates{k}]{z}{k}$, then $\xhatstate{k:s-1}\oplus\fromxstate{s}^*$ belongs to $\frommog[\fromstates{k}]{z}{k}$.  
\end{lemma}

\begin{proof}
  To simplify the notations in this proof, it is convenient to use $\xhatstate{k-1}$ as an alternative notation for $\zstate{k-1}$.
  So from now on $\xhatstate{k-1}=\zstate{k-1}$.

  It follows by Lemma~\ref{stel:argmaxalpha} that $\fromxstate{s}^*\in\explfromopt[\fromstates{s}]{\hat{x}}{s-1}{s}$. 
  Together with Equation~\eqref{eq:pop2imprecies}, this implies that $\xhatstate{s-1}\oplus\xstate{s:n}^*\in\explfrommog[\fromstates{s-1}]{\hat{x}}{s-2}{s-1}$.
  If $s=k+1$, this concludes the proof.
  If $s\in\{k+2,\dots,n\}$, consider all $q\in\{k,\dots,s-1\}$ and check af there is some $q$ for which $\xhatstate{q}\notin\pos_q(\xhatstate{q-1})$ (see definition~\eqref{eq:pop3imprecies} ).
  If such a $q$ exists, denote the lowest $q\in\{k,\dots,s-2\}$ for which this is the case as~$q^*$. 
  By Equation~\eqref{eq:pop2imprecies} we know that $\xhatstate{q^*:s-1}\oplus\xstate{s:n}^*$ and $\xhatstate{q^*:s-1}\oplus\zstate{s:n}$ are both elements of $\explfrommog[\fromstates{q^*}]{\hat{x}}{q^*-1}{q^*}$, since all sequences in the set $\xhatstate{q^*}\oplus\states{q^*+1:n}$ belong to $\explfrommog[\fromstates{q^*}]{\hat{x}}{q^*-1}{q^*}$. 

  If no $q\in\{k,\dots,s-2\}$ exists for which $\xhatstate{q}\notin\pos_q(\xhatstate{q-1})$, we choose $q^*\coloneqq s-1$.
  It then follows by the repeated use of Equations~\eqref{eq:pop1imprecies} and~ \eqref{eq:pop2imprecies} that $\xhatstate{s-1}\oplus\zstate{s:n}$ belongs to $\explfrommog[\fromstates{s-1}]{\hat{x}}{s-2}{s-1}$ and we already know that $\xhatstate{s-1}\oplus\xstate{s:n}^*\in\explfrommog[\fromstates{s-1}]{\hat{x}}{s-2}{s-1}$.

  We now have a $q^*\in\{k,\dots,s-1\}$ for which both $\xhatstate{q^*:s-1}\oplus\xstate{s:n}^*$ and $\xhatstate{q^*:s-1}\oplus\zstate{s:n}$ belong to $\explfrommog[\fromstates{q^*}]{\hat{x}}{q^*-1}{q^*}$ and for which it holds that $\xhatstate{q}\in\pos_q(\xhatstate{q-1})$ for all $q\in\{k,\dots,q^*-1\}$.
  If $q^*=k$, this concludes the proof.

  If $q^*\in\{k+1,\dots,s-1\}$, notice that $\frommog[\fromstates{k}]{z}{k}$ is built up by repeatedly using Equations~\eqref{eq:criterion-at-k} and \eqref{eq:pop2imprecies}.
  We also know that $\xhatstate{q^*:s-1}\oplus\zstate{s:n}\in\explfrommog[\fromstates{q^*}]{\hat{x}}{q^*-1}{q^*}$ and it is given that $\xhatstate{k:s-1}\oplus\zstate{s:n}$ belongs to $\frommog[\fromstates{k}]{z}{k}$, which implies that
  \begin{equation*}
    \specialumem{q}{\xhatstate{q:s-1}\oplus\zstate{s:n}}\geq\xhatinxhatumemopt{q}
    \text{ for all $q\in\{k,\dots,q^*-1\}$.}
  \end{equation*}
  Furthermore, $\specialumem{s-1}{\xhatstate{s-1}\oplus\xstate{s:n}^*}=\xhatumemmax{s-1}$, so $\specialumem{s-1}{\xhatstate{s-1}\oplus\xstate{s:n}^*}\geq\specialumem{s-1}{\xhatstate{s-1}\oplus\zstate{s:n}}$ by Equation~\eqref{eq:alphabetamax}. Equation~\eqref{eq:alpharecurs} then tells us that
  \begin{equation*}
    \specialumem{t}{\xhatstate{t:s-1}\oplus\xstate{s:n}^*}\geq\specialumem{t}{\xhatstate{t:s-1}\oplus\zstate{s:n}}
    \text{  for all $t\in\{k,\dots,s-1\}$},
  \end{equation*}
  so we know that
  \begin{equation*}
    \specialumem{q}{\xhatstate{q:s-1}\oplus\xstate{s:n}^*}\geq\xhatinxhatumemopt{q}
    \text{ for all $q\in\{k,\dots,q^*-1\}$}.
  \end{equation*}
  This implies (since $\frommog[\fromstates{k}]{z}{k}$ is built up by repeatedly using Equations~\eqref{eq:criterion-at-k} and~\eqref{eq:pop2imprecies} and because $\xhatstate{q^*:s-1}\oplus\xstate{s:n}^*$ is an element of $\explfrommog[\fromstates{q^*}]{\hat{x}}{q^*-1}{q^*}$) that the sequence $\xhatstate{k:s-1}\oplus\fromxstate{s}^*$ belongs to $\frommog[\fromstates{k}]{z}{k}$, which concludes the proof.
\end{proof}

\begin{lemma}\label{stel:argmaxalpha}
  Consider any $s\in\{1,\dots,n\}$, $\xhatstate{s-1}\in\states{s-1}$ and $\fromxstate{s}^*\in\fromstates{s}$.
  If $\specialumem{s-1}{\xhatstate{s-1}\oplus\xstate{s:n}^*}=\xhatumemmax{s-1}$, then $\fromxstate{s}^*\in\explfromopt[\fromstates{s}]{\hat{x}}{s-1}{s}$.
\end{lemma}

\begin{proof}
  If $\specialumem{s-1}{\xhatstate{s-1}\oplus\xstate{s:n}^*}=\xhatumemmax{s-1}$, then we know by Equation~\eqref{eq:alphabetamax} that
  \begin{equation*}
    \specialumem{s-1}{\xhatstate{s-1}\oplus\xstate{s:n}^*}\geq\specialumem{s-1}{\xhatstate{s-1}\oplus\zstate{s:n}}
    \text{  for all $\fromzstate{s}\in\fromstates{s}$},
  \end{equation*}
  and therefore by Equations~\eqref{eq:alpha} and \eqref{eq:upper:state:out:mass} that
  \begin{equation*}
    \xhatinoutcupr[\singout{s-1}]{s-1}\xhatinjointcupr[\indsing{\fromxstate{s}^*}\indfromout{s}]{s}{s-1}
    \geq\xhatinoutcupr[\singout{s-1}]{s-1}\xhatinjointcupr[\indfromzstate{s}\indfromout{s}]{s}{s-1}.
  \end{equation*}
  Together with the positivity assumption~\eqref{eq:assumption}, this implies that
  \begin{equation}\label{eq:tussenresultaatje}
    \xhatinjointcupr[\indsing{\fromxstate{s}^*}\indfromout{s}]{s}{s-1}
    \geq\xhatinjointcupr[\indfromzstate{s}\indfromout{s}]{s}{s-1}
    \text{  for all $\fromzstate{s}\in\fromstates{s}$}.
  \end{equation}
  We now by \ref{C3} that 
  \begin{equation*}
    \xhatinjointclpr[-\indsing{\fromxstate{s}^*}\indfromout{s}]{s}{s-1}
    \geq
    \xhatinjointclpr[\indfromout{s}(\indfromzstate{s}-\indsing{\fromxstate{s}^*})]{s}{s-1}+
    \xhatinjointclpr[-\indfromzstate{s}\indfromout{s}]{s}{s-1}
  \end{equation*}
  which by conjugacy implies that
  \begin{equation*}
    \xhatinjointclpr[\indfromout{s}(\indfromzstate{s}-\indsing{\fromxstate{s}^*})]{s}{s-1}
    \leq
    \xhatinjointcupr[\indfromzstate{s}\indfromout{s}]{s}{s-1}-
    \xinjointcupr[\indsing{\fromxstate{s}^*}\indfromout{s}]{s}{s-1}.
  \end{equation*}
  Using Equation~\eqref{eq:tussenresultaatje}, we see that $\xhatinjointclpr[\indfromout{s}(\indfromzstate{s}-\indsing{\fromxstate{s}^*})]{s}{s}\leq0$ for all $\fromzstate{s}\in\fromstates{s}$, which concludes the proof, since $\fromxstate{s}^*\in\explfromopt[\fromstates{s}]{\hat{x}}{s-1}{s}$ by Equation~\eqref{eq:optimals}.
\end{proof}

\end{document}